%% file: Bilinear_Obs_SYSID.tex
\documentclass[twosided]{article}
\usepackage[T1]{fontenc}
\usepackage{lmodern}

\title{Learning Linear Dynamics from Bilinear Observations}
\author{Yahya Sattar$^\ast$\qquad Yassir Jedra$^\dagger$ \qquad Sarah Dean$^\ast$
}

\input{notation}
\usepackage{enumitem}

\usepackage{multirow}

\date{
    $^\ast$Department of Computer Science, Cornell University\\
    $^\dagger$Laboratory of Information and Decision Systems, MIT \\ 
    {\small\{ysattar, sdean\}@cornell.edu \, jedra@mit.edu} 
}
\begin{document}
\maketitle
\begin{abstract}
We consider the problem of learning a realization of a partially observed dynamical system with linear state transitions and bilinear observations. 
Under very mild assumptions on the process and measurement noises, we provide a finite time analysis for learning the unknown dynamics matrices (up to a similarity transform). 
Our analysis involves a regression problem with heavy-tailed and dependent data.
Moreover, each row of our design matrix contains a Kronecker product of current input with a history of inputs, making it difficult to guarantee persistence of excitation.
We overcome these challenges, first providing a data-dependent high probability error bound for arbitrary but fixed inputs.
Then, we derive a data-independent error bound for inputs chosen according to a simple random design.
Our main results provide an upper bound on the statistical error rates and
sample complexity of learning the unknown dynamics matrices from a single finite trajectory of bilinear observations.
\end{abstract}

\section{Introduction} \label{sec:Intro}
In many engineering and computing domains, measurements arise due to an interaction between a measurement probe and an unknown quantity of interest. 
This interaction is often modeled as \emph{bilinear} in the two variables.
Examples include the models used in compressed sensing~\cite{donoho2006compressed}, sensor design~\cite{tanaka2016semidefinite},
and matrix factorization~\cite{koren2021advances,jun2019bilinear}
for applications spanning medical imaging, control, and personalization.
In these classical models, the unknown quantity of interest is static or evolving according to known dynamics, independently of the measurement probe.
However, in modern applications, measurement probes may \emph{affect} the unknown quantity of interest.
Also known as the \emph{observer effect}~\cite{baclawski2018observer},
this phenomenon is present in settings ranging from electronic circuits,
where a measurement device may cause changes in resistance or impedance, to psychology and human behavior (e.g. the Hawthorne effect).
An individual receiving personalized recommendations may not only respond according to their current preference, but may undergo a preference shift in response~\cite{dean2022preference,khosravi2023bandits,leqi2021rebounding}.
Especially in non-physical domains, concrete models of these effects cannot be derived from first principles.
Instead, the quantity of interest evolves according to unknown dynamics, affected by the measurement probe.
In such settings, it is pertinent to identify the unknown dynamics.

In this paper, we study the learning problem for a dynamics model in which the input non-trivially affects both the state update and the observation of the state.
In particular, we consider partially observed stochastic dynamical systems with linear transitions and bilinear observations.
Our focus is on characterizing the number of measurements necessary for estimating the unknown dynamics model to a given level of precision.
This follows a recent surge of interest in understanding non-asymptotic properties and the sample complexity of learning dynamical systems. 
Most of the advancements in this direction are focused on linear time-invariant~(LTI) dynamical systems, including both complete state observation~\cite{faradonbeh2018finite,dean2018regret,simchowitz2018learning, dean2019sample,fattahi2019learning,sarkar2019finite,sarkar2019near,lale2020logarithmic,jedra2020finite,wagenmaker2020active} as well as partial state observation~\cite{tu2017non,hazan2017learning,hardt2018gradient,oymak2018non,tsiamis2019finite,simchowitz2019learning,sun2020finite,djehiche2022efficient, lee2022improved,mania2022time,bakshi2023new}. These results have been extended to certain classes of switched linear systems \cite{sarkar2019data,sattar2021identification,du2022data,sayedana2024strong}, both fully and partially observed. 
The problem of learning nonlinear dynamical systems has been mainly studied in the fully observed setting,
including state transition models with
nonlinear activation functions~\cite{oymak2019stochastic,bahmani2019convex,mhammedi2020learning,sattar2020non,jain2021near},
bilinearities~\cite{sattar2022finite},
nonlinear features~\cite{mania2022active}, or from a non-parametric perspective~\cite{taylor2021towards,ziemann2022single,kazemian2024random}.

Our work is closely related to the problem of learning LTI systems from \emph{partial} state observations. 
Most of the existing works in this setting either estimate the system's Markov parameters~\cite{oymak2021revisiting,simchowitz2019learning,lee2022improved,djehiche2022efficient,bakshi2023new} or the system's Hankel matrix~\cite{tsiamis2019finite,sarkar2019finite}, which are then used to estimate the state-space matrices~(up to a similarity transform) using classic Ho-Kalman Algorithm~\cite{ho1966effective}. 
We take a similar approach, however, we consider bilinear observations while the observation model is \emph{linear} in all of these works. 
Also related to our work, \cite{dean2020robust,dean2021certainty,mhammedi2020learning} propose and analyze settings in which a low dimensional LTI system is controlled using high-dimensional nonlinear observations.
These works make the assumption of \emph{perfect and noiseless decodability}. 
As a result,~\cite{mhammedi2020learning} shows how to learn a \emph{decoder} function directly from {multiple} trajectories. 
Once the decoder function is learned, the the state-space matrices can be directly identified---hence, this setting is not truly \emph{partially} observed. 
Our work is more closely related to~\cite{mania2022time}, who propose a time-varying linear regression problem with unknown parameters evolving according to an \emph{unactuated} linear dynamical system.
Unlike our work, \cite{mania2022time} directly learns the dynamics matrices by combining two ordinary least squares estimates. 
To the best of our knowledge, the problem of learning a partially observed LTI system from bilinear observations has not been studied before. 

{\bf Contributions:} We provide a non-asymptotic sample complexity analysis for  learning a partially observed LTI system from bilinear observations.
The problem is challenging due to the fact that observations give only partial information about the state, are nonlinear, and are correlated over time.
We address these challenges and make the following contributions:
\begin{itemize}[leftmargin=*,noitemsep,topsep=0pt]
\item  {\bf Learning from bilinear observations:}
We provide the first non-asymptotic error bounds and sample complexity analysis for learning a realization of a partially observed dynamical system with linear state transitions and bilinear observations, given data from a single finite trajectory. 
We provide two types of guarantees in Section~\ref{sec:main_results}: (a) Our data-dependent error bound (Theorem~\ref{thrm:learning_markov_par_main}) holds for any given sequence of inputs, and is useful for
downstream tasks 
because it precisely captures the shape of the uncertainty set. (b) Our data-independent error bound (Theorem~\ref{thrm:learning_markov_par_data_IND}) characterizes the sample complexity and captures the optimal $\tilde{\Ocal}(1/\sqrt{T})$ dependence on the trajectory length.

\item  {\bf Heavy-tailed noise and inputs:}
Our error bounds~(Theorems~\ref{thrm:learning_markov_par_main} and \ref{thrm:learning_markov_par_data_IND}) hold under very mild assumptions on the process and measurement noise (Assumption~\ref{assump:data/noise}), which are only assumed to be centered and have bounded covariance. Hence, our results hold for heavy-tailed noise processes. Moreover, because of bilinear observation, our least-squares estimation problem involves an input design matrix with heavy-tailed and highly dependent rows.

\item  {\bf Persistence of excitation:}
For inputs chosen according to a simple random design, we establish persistence of excitation (Theorem~\ref{thrm:persistence_of_excitation}) under two different types of assumptions on the distribution of inputs (Assumptions~\ref{assump:input}). This is particularly challenging because each row of the design matrix consists of a Kronecker product of an input vector with a history of input vectors. We overcome this challenge by using a blocking technique along-with heavy-tailed concentration and covering arguments to guarantee persistence of excitation.

\end{itemize}
Moreover, in Section~\ref{sec:statespace} we use a slightly modified version of results in \cite{oymak2021revisiting}
to give guarantees for learning state-space matrices~(up to a similarity transform) using the classic Ho-Kalman algorithm~\cite{ho1966effective}. 
In Section~\ref{sec:experiments}, we run experiments with synthetic data to show a trade-off between the estimation error and the system's memory~(captured by the number of Markov parameters and the spectral radius of the dynamics matrix). 
We then present the proofs of our main results in Section~\ref{sec:proof_main_results}, and conclude with a discussion of future directions in Section~\ref{sec:conclusion}.

\subsection{Notations}
We use boldface lowercase (uppercase) letters to denote vectors (matrices). $\rho(\vX),\|\vX\|$ and $\tf{\vX}$ denote the spectral radius, spectral norm and Frobenius norm of a matrix $\vX$, respectively. Similarly, we denote the Euclidean norm and Frobenius norm of a vector $\vv$ by $\tn{\vv}$ and $\tf{\vv}$, respectively. For a positive definite matrix $\vM \in \R^{d \times d}$, the matrix or ellipsoidal norm of a vector $\vv \in \R^d$ is defined by $\norm{\vv}_{\vM} = \sqrt{\vv^\T \vM \vv}$. We use $\vek(\cdot): \R^{m \times n} \mapsto \R^{mn}$ to denote the vectorization operator, which transforms a matrix into a vector, whereas $\mat(\cdot): \R^{mn} \mapsto \R^{m \times n}$ denotes its inverse operator, which transforms a vector into a matrix. $\Sc^{d-1}$ denotes the unit sphere in $\R^d$. 
For a random vector $\vv$, we denote its covariance matrix by $\bSi[\vv]$. We use $\gtrsim$ and $\lesssim$ for inequalities that hold up to an absolute constant factor. $\tilde{\Ocal}(\cdot)$ shows dependence on a specific quantity of interest~(up to constants and logarithmic factors). 
Finally, $\otimes$ denotes the Kronecker product.

\section{Problem Formulation}\label{sec:Prob_Formulatiom}

In this paper, we consider the problem of learning the parameters of a partially observed dynamical system from a single trajectory.
In particular, the system has linear transitions and bilinear observations, evolving according to
\begin{equation}
\begin{aligned}\label{eqn:bilinear sys}
	\xb_{t+1} &= \Ab \xb_t + \Bb\ub_t + \wb_{t},\\
	y_t &= \vu_t^\T\vC \xb_t + z_t,
\end{aligned}
\end{equation}
where $\xb_t \in \R^n$ is the state, $\ub_t \in \R^p$ is the input, $y_t \in \R$ is the output, $\wb_t \in \R^n$ is the process noise, and $z_t \in \R$ is the measurement noise at time $t$. 
In this model, the input affects both the state (through the linear dynamics update) and the observation of the state (through the bilinear measurement equation). 
Without loss of generality, we assume that the system starts at zero initial state, i.e., $\vx_0 = 0$. Our goal is to learn a balanced realization of the unknown system matrices $\Ab \in \R^{n \times n}$, $\Bb \in \R^{n \times p}$, and $\Cb \in \R^{p \times n}$ from finite input-output samples $\{(\ub_t,\vy_t)\}_{t=0}^T$ obtained from a single trajectory of \eqref{eqn:bilinear sys}. 
A key intermediate step of our identification procedure is based on estimating the first $L$ Markov parameters\footnote{Note that though the form of the Markov parameters is the same as for LTI systems, the relationship between inputs and outputs differs due to the bilinear observation.} of the system,
which are given by the matrices $\{\vC \vA^{i}\vB\}_{i=0}^{L-1}$.
Markov parameters describe the input-output behavior of the system.
They are of interest in their own right, and they can be used to predict future outputs.
Additionally, as we explain in Section~\ref{sec:statespace}, the Markov parameters are sufficient for identifying the state space parameters $\Ab,\Bb,\Cb$ up to a similarity transform.

\subsection{Least Squares Estimation}

Our identification procedure is based on the observation that we can estimate the system's Markov parameters  by regressing the outputs $y_t$ to an expression defined by the history of inputs $\{\vu_\tau\}_{\tau \leq t}$. 
For $t \geq L$, the output $y_t$ can be written in terms of the past inputs as follows:
\begin{equation}
\begin{aligned}\label{eqn:yt_expansion}
	y_t 
	&= \vu_t^\T\vC \Ab^L\xb_{t-L} + \sum_{i = 0}^{L-1} \vu_t^\T\vC \Ab^i\vB \vu_{t-i-1}  + \sum_{i = 0}^{L-1} \vu_t^\T\vC \Ab^i\vw_{t-i-1} + z_t.
\end{aligned} 
\end{equation}
This expression differs from the expansion for LTI systems due to the left multiplication by $\vu_t$. 
To ease notation, we organize the inputs $\{\vu_t\}_{t=0}^T$ and the noise $\{\vw_t\}_{t=0}^T$ into the following vectors,
\begin{align}
	\ubb_t := \begin{bmatrix} \vu_{t}^\T & \vu_{t-1}^\T  & \cdots & \vu_{t-L+1}^\T \end{bmatrix}^\T \in \R^{pL}, \quad 
	\wbb_t := \begin{bmatrix} \vw_{t}^\T & \vw_{t-1}^\T  & \cdots & \vw_{t-L+1}^\T \end{bmatrix}^\T \in \R^{nL}, \label{eqn:util_wtil_vec}
\end{align}
and the  first $L$ Markov parameters 
into a matrix,
\begin{align}
	\vG &:= \begin{bmatrix} \vC \vB & \vC \vA\vB& \cdots &\vC \vA^{L-1}\vB \end{bmatrix} \in \R^{p \times pL} ,  \quad\vF := \begin{bmatrix} \vC  & \vC \vA & \cdots &\vC \vA^{L-1} \end{bmatrix} \in \R^{p \times nL} ,  \label{eqn:G_mtx}
\end{align}
where the matrix $\vF$ is similar to $\vG$ in structure.
With these definitions, we can write the output $y_t$ in terms of Markov parameters, 
inputs, and noise as follows:
\begin{equation}
\begin{aligned}
	y_t  &= \vu_t^\T \vG \ubb_{t-1} + \vu_t^\T\vF \wbb_{t-1} + \vu_t^\T\ve_t + z_t, \\
	 	 &\equiv \ubb_{t-1}^\T \otimes \vu_t^\T \vek(\vG)  + \wbb_{t-1}^\T \otimes \vu_t^\T \vek(\vF)  + \vu_t^\T\ve_t + z_t,\label{eqn:yt_Gutil_Fwtil}
\end{aligned}
\end{equation}
where we define $\ve_t: = \vC \Ab^{L}\xb_{t-L}$, which corresponds to the error due to the effect of unknown state at time $t-L$. 
Given the covariate-output relation~\eqref{eqn:yt_Gutil_Fwtil}, 
we treat the terms depending on $\wbb_{t-1}$, $\ve_t$, and $z_t$ as additive noise and attempt to estimate $\vG$ from the covariates $\ubb_{t-1} \otimes \vu_t$. 
From a single input-output trajectory $\{(\vu_t, \vy_t)\}_{t=0}^T$, we generate $T-L$ sub-sequences of length $L \geq 1$ and define a regression problem with covariates and outputs $\{(\ubb_{t-1} \otimes \vu_t, y_t)\}_{t=L+1}^T$. 
In the following section, we formally define and analyze this estimation procedure.

\section{Main Results on Learning Markov Parameters}\label{sec:main_results}
We are interested in bounding the estimation error on the Markov parameters. First, consider the least squares problem formulated above:
\begin{equation}
\begin{aligned} \label{eqn:ERM_Ghat}
	\vGhat &= \arg \min_{\vG \in \R^{p \times pL}}\frac{1}{2(T-L)} \sum_{t=L+1}^{T} (y_t - \vu_t^\T\vG \ubb_{t-1})^2
    = \arg \min_{\vG \in \R^{p \times pL}}\frac{1}{2(T-L)}\tn{\vy - \vUtil \vek(\vG)}^2,
\end{aligned}
\end{equation}
where we arrange the covariates $\ubb_{t-1} \otimes \vu_t$ into the design matrix $\vUtil \in \R^{(T-L) \times p^2L}$, and the outputs $y_t$ into the output vector $\vy \in \R^{(T-L)}$ as follows,
\begin{equation}
\begin{aligned} \label{eqn:Util_and_y} 
	\vUtil := \begin{bmatrix} \ubb_{L} \otimes \vu_{L+1} & \ubb_{L+1} \otimes \vu_{L+2} & \cdots & \ubb_{T-1} \otimes \vu_{T} \end{bmatrix}^\T, \quad
	 \vy := \begin{bmatrix} y_{L+1} & \cdots & y_{T} \end{bmatrix}^\T .
\end{aligned}
\end{equation}
 When the problem is over-determined, i.e. $\vUtil$ is full rank, the solution to the least-squares problem~\eqref{eqn:ERM_Ghat} is given by $\vek(\vGhat) =  (\vUtil^\T \vUtil)^{-1}\vUtil^\T \vy$.
 Define $\vzeta\in \R^{(T-L)}$ to contain the stacked noise terms depending on $\wbb_{t-1}$, $\ve_t$, and $z_t$ in~\eqref{eqn:yt_Gutil_Fwtil}.
 Then, the estimation error is given by 
 $\vek(\hat{\vG}) - \vek(\vG) = (\vUtil^\T \vUtil)^{-1}\vUtil^\T \vzeta$. Note that, bounding this estimation error is challenging because: (a) The covariates are dependent and heavy-tailed. Even for the sub-Gaussian inputs $\vu_t$, the covaraites $\ubb_{t-1} \otimes \vu_t$ are heavy tailed. (b) The additive noise $\zeta_t:= \wbb_{t-1}^\T \otimes \vu_t^\T \vek(\vF)  + \vu_t^\T\ve_t + z_t$ is also heavy-tailed and correlated over time. Even if $\vw_t$ and $z_t$ are individually sub-Gaussian, the overall noise $\zeta_t$ is heavy-tailed.

\subsection{Data-dependent Error Bounds}
We begin by focusing on a data-dependent bound.
Setting $\vVtil := \vUtil^\T\vUtil$, we seek to bound the the data-dependent ellipsoidal norm of the estimation error, 
 \begin{equation}
 	\begin{aligned} \label{eqn:define_ellipsoidal}
 		\norm{\vek(\hat{\vG}) - \vek(\vG)}_{\vVtil} &= \sqrt{\vzeta^\top\vUtil(\vUtil^\T \vUtil)^{-1}\vUtil^\T \vzeta}. 
 	\end{aligned}
 \end{equation}
This norm can precisely capture the ellipsoidal shape of the uncertainty set, and is thus useful for downstream tasks such as output prediction and control. 
Furthermore, data-dependent bounds require only mild assumptions about how the data is generated.
The resulting uncertainty set is valid in hindsight for any given sequence of inputs
regardless of the data collection policy.
Later on, we will show that for stochastic inputs with either bounded second and fourth moments, or bounded Euclidean norm, we have $\vVtil \succeq \tilde{\Ocal}(T-L) \Iden_{p^2L}$ with high probability. Combined with our data-dependent result, this guarantees persistence of excitation and a $\tilde{\Ocal}(1/\sqrt{T-L})$ error rate, which is optimal in the trajectory length. 

To upper bound the ellipsoidal norm of the estimation error in \eqref{eqn:define_ellipsoidal}, we assume the bilinear-observation system~\eqref{eqn:bilinear sys} 
satisfies the following two properties.
\begin{assumption}(Stability)\label{assump:stability}
    The dynamical system in \eqref{eqn:bilinear sys} is strictly stable, that is, $\rho(\vA) < 1$. 
\end{assumption}
It is well-known that, when $\rho(\vA) < 1$, there exist $\rho \in (\rho(\vA), 1)$ and $\phi(\vA, \rho) \geq 1$ such that, $\|\vA^k\| \leq \phi(\vA,\rho) \rho^k$ for all $k \in \mathbb{Z}_+$. The quantity $\phi(\vA,\rho) := \sup_{k \in \mathbb{Z}_+}(\norm{\vA^k}/ \rho^k)$ is finite by Gelfand's formula, and it measures the transient response of the system, and can be upper bounded by its $\mathcal{H}_\infty$ norm \cite{tu2017non}.
This decay condition is important for showing that the bias of our estimate is small when $L$ is large enough, and is a relatively common assumption~\cite{oymak2021revisiting, lee2022improved}. 

\begin{assumption}(Noise properties)\label{assump:data/noise}  
The process and measurement noise are stochastic $\{\wb_t\}_{t=0}^T \distas \Dcal_w$, $\{z_t\}_{t=0}^T \distas \Dcal_z$, centered $\E[\vw_t] = 0$, $ \E[z_t] = 0$, and have bounded covariance $ \E[\vw_t\vw_t^\T] = \vSigma_w$, and $\E[z_t^2] = \sigma^2_z$, such that $\norm{\vSigma_w}, \sigma_z < \infty$. 
\end{assumption}
These assumptions on the process and measurement noise are mild.
They are less restrictive than sub-Gaussian assumptions and hold even for heavy-tailed noise processes. Our first result provides an upper bound on the ellipsoidal norm of the estimation error in \eqref{eqn:define_ellipsoidal}.
\begin{theorem}[Learning Markov Parameters~(Data-Dependent)] \label{thrm:learning_markov_par_main} Fix $\delta \in (0,1)$, and suppose we are given a single trajectory $\{(\ub_t,\vy_t)\}_{t=0}^T$ of the system~\eqref{eqn:bilinear sys}. Let $\vGamma_w^{\infty} := \sum_{i=0}^{\infty} \vA^i \vSigma_w (\vA^i)^\T$ denote the infinite time controllability Gramian associated with the process noise. Let $\vVtil :=  \vUtil^\T\vUtil$, where $\vUtil$ is the design matrix given by~\eqref{eqn:Util_and_y}, and suppose $\vVtil \succ 0$. Suppose Assumptions~\ref{assump:stability} and~\ref{assump:data/noise} hold, and  $\tn{\vu_t} \leq \beta$, for some $\beta>0$, and for all $t \in [T]$. Let $\sigma_\vw^2 := \|\vSigma_w\|\tf{\vF}^2 \big(1 + \frac{\phi(\vA,\rho)\rho^L}{1-\rho}\big)$, $\sigma_\ve^2 := \norm{\vGamma_w^\infty} \norm{\vC\vA^L}^2\frac{\phi(\vA,\rho)}{1-\rho}$, and $K = \max\{\norm{\vB}, \norm{\vC}\}$. Then, solving the least-squares problem~\eqref{eqn:ERM_Ghat}, with probability at least $1- \delta$, we have,
\begin{align}\label{eq:Markov_upper_bd}
    \norm{\vek(\hat{\vG}) - \vek(\vG)}_{\vVtil} &\lesssim \mysqrt{\frac{p^2 L}{\delta}\big(\sigma_z^2 + \sigma_\vw^2 \beta^2 L  + \sigma_\ve^2 \beta^2 \big)} + \beta^2 K^2 \frac{\phi(\vA,\rho)\rho^L}{1-\rho} \sqrt{T-L}.
\end{align}
\end{theorem}
Note that the first term in our error bound depends on the measurement/process noise variances $\sigma_z^2, \norm{\vSigma_w}$ indicating that the estimation error will be smaller for smaller values of $\sigma_z^2$ and $\norm{\vSigma_w}$. The second term decays exponentially with $L$, whereas the first term grows linearly with $L$. The optimal choice of $L$ can be obtained by minimizing our error bound with respect to $L$. Specifically, in Theorem~\ref{thrm:learning_markov_par_data_IND}, we choose $L \gtrsim \Ocal(\log(T-L)/\log(\rho^{-1}))$ to get a simplified error bound which decays as $\tilde{\Ocal}(1/\sqrt{T-L})$. Moreover, both terms depend on the stability $\rho$ and transient response $\phi(\vA,\rho)$ of the system, and they are smaller for stable systems with small transients.

The proof of Theorem~\ref{thrm:learning_markov_par_main} is presented in Section~\ref{sec:proof_learning_markov_param}. Note that our result holds for any given sequence of inputs, and heavy-tailed noise. Hence, we rely on heavy-tailed tools to upper bound the ellipsoidal norm of the estimation error. Specifically, we upper bound the covariance of the estimation error for any given sequence of inputs, which requires upper bounding the auto-covariance function of the heavy-tailed, non-centered and highly dependent noise process $\{\zeta_t := \wbb_{t-1}^\T \otimes \vu_t^\T \vek(\vF)  + \vu_t^\T\ve_t + z_t\}_{t=L+1}^T$. Finally, we use multidimensional Chebyshev's inequality to upper bound the ellipsoidal norm of the estimation error with high probability. 
Section~\ref{sec:proof_learning_markov_param} also presents an upper bound on the Frobenius norm of the estimation error. Specifically, under the same setting of Theorem~\ref{thrm:learning_markov_par_main}, with probability at least $1-\delta$, we get $\tf{\vGhat - \vG} \lesssim \kappa/\sqrt{\lambda_{\min}\big(\vUtil^\T \vUtil\big)}$, where $\kappa$ denotes the right hand side of~\eqref{eq:Markov_upper_bd}.

Lastly, in the data-dependent setting, we turn from estimation error to prediction error.
The estimated Markov parameters can predict the future outputs by $\yhat_{T+1} = \vu_{T+1}^\T \vGhat \ubb_{T}$. Given the data-dependent estimation error bounds in Theorem~\ref{thrm:learning_markov_par_main}, our next result upper bounds the output prediction error for arbitrary inputs.  
 \begin{lemma}[Output Prediction~(Data-Dependent)]\label{lemma:output_prediction}
     Let $\vGhat$ be the solution to the least-squares problem~\eqref{eqn:ERM_Ghat}, and let $\yhat_{T+1} = \vu_{T+1}^\T \vGhat \ubb_{T}$ be the output prediction at time $T+1$. Suppose Assumptions~\ref{assump:stability} and~\ref{assump:data/noise} hold. Let $\vGamma_w^{(T)} := \sum_{i=0}^{T} \vA^i \vSigma_w (\vA^i)^\T$, and $\vGamma_u^{(T)} := \sum_{i=0}^{T}\sum_{j=0}^T \vA^i \vB \vu_{T-i}\vu_{T-j}^\T \vB^\T (\vA^i)^\T$ denote the finite time controllability Gramians associated with the process noise and control inputs respectively. Then, we have 
\begin{equation}
     \begin{aligned}
         \E[(\yhat_{T+1} - y_{T+1})^2 ] &\leq 2\norm{\vek(\vGhat) - \vek(\vG)}_{\vVtil}^2 \norm{\ubb_T \otimes \vu_{T+1}}_{\vVtil^{-1}}^2 + 2\beta^2 \norm{\vC\vA^L}^2\norm{\vGamma_u^{(T)} + \vGamma_{w}^{(T)}} \\ 
         & \qquad\qquad\qquad\qquad\qquad\qquad\qquad\qquad\qquad\qquad\;\; + \beta^2\norm{\vSigma_w}\tf{\vF}^2 + \sigma_z^2.
     \end{aligned}
\end{equation}
 \end{lemma}
The proof of Theorem~\ref{thrm:learning_markov_par_data_IND} is presented in  Section~\ref{sec:proof_output_prediction} and is relatively straightforward. Note that the output prediction error is upper bounded by the ellipsoidal norm of the estimation error, scaled by a factor $\norm{\ubb_T \otimes \vu_{T+1}}_{\vVtil^{-1}}^2$ which captures the similarity measure between the training covariates and the test covariate $\ubb_T \otimes \vu_{T+1}$. Note that, we can also get a data-independent output prediction error bound by taking expectations with respect to $\ubb_T \otimes \vu_{T+1}$ as well. As a result, the output prediction error will be upper bounded in terms of data-independent error bound on $\tf{\vG - \vGhat}^2$~(Theorem~\ref{thrm:learning_markov_par_data_IND}) scaled by the covariance of $\ubb_T \otimes \vu_{T+1}$, which turns out to be identity matrix under Assumption~\ref{assump:input}. The remaining terms will be similar, except $\Gamma_u^{(T)}$, which is replaced by $\sum_{i=0}^{T} \vA^i \vB \vB^\T(\vA^i)^\T$.

\subsection{Sample Complexity and Persistence of Excitation}
If the control inputs are stochastic, then we can guarantee persistence of excitation, which is then combined with Theorem~\ref{thrm:learning_markov_par_main} to get an optimal error rate of $\tilde{\Ocal}(1/\sqrt{T-L})$. Specifically, we assume that the control inputs satisfy the following assumption.
\begin{assumption}(Input properties)\label{assump:input} $\{\vu_t\}_{t=0}^T \distas \Dcal_u$ are stochastic  with zero mean $\E[\vu_t]  = 0$, isotropic covariance $\E[\vu_t \vu_t^\T] = \Iden_p$, and satisfy at least one of the following two conditions: 
\begin{enumerate}[label=(\textbf{\alph*}),leftmargin=*,noitemsep,topsep=0pt]
    \item Bounded Euclidean norm, i.e., there exists a scalar $\beta>0$ such that $\tn{\vu_t} \leq \beta$ for all $t \in [T]$. 
    \item Bounded fourth moment covariates, i.e., there exists a scalar $m_4>0$ such that $ \sup_{\vv \in \Sc^{p^2L-1}} \E[(\vv^\T(\ubb_{t-1}  \otimes \vu_{t}))^4] \leq m_4$ for all $t \in [T]$, where $\ubb_t$ is as defined in \eqref{eqn:util_wtil_vec}.
\end{enumerate}
\end{assumption}
Note that the assumption of bounded inputs can be relaxed to include unbounded stochastic inputs as well. For example, in the case of sub-Gaussian inputs, it is easy to show that, with probability at least $1-\delta$, the inputs are bounded as $\tn{\vu_t} \lesssim \Ocal(\sqrt{p\log(T/\delta)})$ for all $t \in [T]$. Similarly, the assumption of isotropic inputs can also be relaxed to inputs with positive-definite covariance $\E[\vu_t \vu_t^\T] \succ 0$. Note that, Assumption~\ref{assump:input}($\vb$) is strictly less restrictive than Assumption~\ref{assump:input}($\va$), and includes heavy-tailed distributions as well. The specific value of $m_4$ depends on the distribution of inputs. For example, when $\{\vu_t\}_{t=0}^T \distas \Ncal(0, \Iden_p)$, we have $ \sup_{\vv \in \Sc^{p^2L-1}} \E[(\vv^\T(\ubb_{t-1}  \otimes \vu_{t}))^4] \leq 9$~(see Appendix~\ref{sec:moment_bound_verification}). Our next main result guarantees persistence of excitation under Assumption~\ref{assump:input}. 
\begin{theorem}[Persistence of Excitation]\label{thrm:persistence_of_excitation}
Consider a sequence of inputs $\{\vu_t\}_{t=0}^T \distas \Dcal_u$ with number of samples 
satisfying,
\begin{align}
    T -L &\gtrsim \gamma_1 (L+1)\left (\log\left(\frac{2(L+1)}{\delta}\right) + \gamma_2 p^2L \right). \label{eqn:trajectory_size_main}
\end{align}
Suppose either of the following two conditions hold: $(\mathbf{1})$ $\{\vu_t\}_{t=0}^T$ satisfies Assumption~\ref{assump:input}($\va$), $\gamma_1 = \beta^4L$, and $\gamma_2 = 1$ in \eqref{eqn:trajectory_size_main}; $(\mathbf{2})$ $\{\vu_t\}_{t=0}^T$ satisfies Assumption~\ref{assump:input}($\vb$), $\gamma_1 = m_4$, and $\gamma_2 = \log(1 + 16 p^2L/\delta)$ in \eqref{eqn:trajectory_size_main}. Then, with probability at least $1- \delta$, we have
\begin{align}
    \lambda_{\min}\left(\vUtil^\T \vUtil\right) \equiv \lambda_{\min}\left( \sum_{t=L+1}^{T} (\bar{\vu}_{t-1} \otimes \vu_{t}) (\bar{\vu}_{t-1} \otimes \vu_{t})^\top \right) \geq (T-L)/4.
\end{align}
\end{theorem} 
The proof of Theorem~\ref{thrm:persistence_of_excitation} is given in Section~\ref{sec:proof_persistence_of_excitation}. Our proof relies on a blocking technique to deal with dependent inputs. This is followed by an application of Hoeffding's inequality or one-sided Bernstein's inequality under Assumption~\ref{assump:input}($\va$) or Assumption~\ref{assump:input}($\vb$) respectively. We then merge the concentration bounds for each block efficiently to get an optimal dependence on the number of samples. Finally, we use covering arguments to get a lower bound on the minimum eigenvalue.

It is worth mentioning that under Assumption~\ref{assump:input}($\vb$), despite only requiring a bounded fourth moment condition, our sample complexity condition  \eqref{eqn:trajectory_size_main} exhibits the same dependence on $\delta$ (of order $\log(1/\delta)$), as one would obtain under the stronger requirement of having inputs with subgaussian tails as is the case in Assumption~\ref{assump:input}($\va$). This phenomenon was first observed concurrently by Koltchinskii and Mendelson  \cite{koltchinskii2015bounding} and Oliveira \cite{oliveira2016lower} using different proof techniques but only for random matrices with independent rows. Notably, \cite{koltchinskii2015bounding} introduced the so-called small ball condition for bounding the smallest singular value of the random matrix, and \cite{oliveira2016lower} showed that only a fourth moment condition is sufficient to obtain similar bounds using a PAC-Bayes approach. In contrast, the proof method we use follows an elementary approach based only on a one-sided Bernstein concentration bound and a net argument and can be adapted, as we showcase, to random matrices with dependent rows. As such, we believe that our result and proof technique may be of independent interest.       

With this, we are ready to state our next result which gives a sample complexity guarantee for the estimation of $\vG$.

\begin{theorem}[Learning Markov Parameters~(Sample Complexity)] \label{thrm:learning_markov_par_data_IND} 
    Consider the same setup of Theorem~\ref{thrm:learning_markov_par_main}.  
    Moreover, suppose Assumption~\ref{assump:input}($\va$) also holds, and the trajectory length satisfies \eqref{eqn:trajectory_size_main} with $\gamma_1 = \beta^4L$ and $\gamma_2 = 1$. Suppose,
 \begin{align}
      L & \gtrsim \frac{1}{\log(\rho^{-1})} \left(\log (T-L) + \log \left( \frac{\delta\beta^2 K^2\phi(\vA,\rho)}{p^2L(\sigma_z^2 + \sigma_\vw^2 \beta^2 L + \sigma_\ve^2 \beta^2)(1-\rho)}\right)  \right).
 \end{align}
 Then, with probability at least $1-\delta$, we have
 \begin{align}
     \tf{\vGhat - \vG} \lesssim \mysqrt{\frac{p^2 L (\sigma_z^2 + \sigma_\vw^2 \beta^2L + \sigma_\ve^2 \beta^2)}{\delta(T-L)}}
 \end{align}
\end{theorem}

The proof of Theorem~\ref{thrm:learning_markov_par_data_IND} is presented in Section~\ref{sec:proof_learning_markov_par_data_IND}. Theorem~\ref{thrm:learning_markov_par_data_IND} suggests that one should choose $L = \tilde{\Ocal}\left( {\log(T-L)}/{\log(\rho^{-1})}\right)$. This indicates that for more stable systems~(smaller $\rho(\vA)$), one should choose smaller $L$ and vice versa. This is further verified by our numerical experiments in Section~\ref{sec:experiments}. The choice results in an optimal dependence of $\tilde{\Ocal}(1/\sqrt{T-L})$ on the trajectory length. 
Moreover, our error bound depends as expected on various noise strengths $\sigma_z^2, \sigma_\vw^2$ and $\sigma_\ve^2$. 
Note that $\sigma_\vw$ decays with $\norm{\vSigma_w}$ and captures the effect of process noise in our estimation error, whereas $\sigma_\ve$ decays with $\norm{\vGamma_w^\infty}\norm{\vC \vA^L}^2 \leq \norm{\vGamma_w^\infty}\norm{\vC}^2 \phi(\vA,\rho)^2 \rho^{2L}$ 
and captures the effect of error due to unknown state at time $t-L$.
The dependence on the failure probability is $1/\delta$ instead of $\log(1/\delta)$ dependence due to heavy-tailed nature of our problem. 
The dependence on $p^2L$ is also expected since there are $p^2 L$ unknown parameters in $\vG$.
However, the error bound also has an extra$ \beta^2L = \tilde{\Ocal}(pL)$ factor.
This additional dependence is a result of using heavy-tailed tools, and can be seen in other heavy-tailed analysis as well~\cite{mania2022time}. 

Additional factors of dimension appear in the sample complexity~\eqref{eqn:trajectory_size_main} as well.
Ignoring logarithmic factors, our sample complexity bound grows as $T-L \gtrsim ((L+1)L \beta^2)p^2L$. Ideally, one hopes for a $T-L \gtrsim p^2L$ sample complexity bound.
The additional multiplication by $(L+1)$ comes from our blocking technique. In the case of standard LTI system, this can be avoided by exploiting either (partial) circulant~\cite{oymak2021revisiting,krahmer2014suprema} or Toeplitz~\cite{sarkar2021finite,djehiche2019finite} structure in the design matrix. 
However, in the bilinear observation case, the design matrix~(see $\vUtil$ in \eqref{eqn:Util_and_y}) does not follow such a structure\footnote{Though it can be constructed from a row-wise Kronecker product of two dependent matrices of inputs, one of which has a partial circulant structure.}; removing the $(L+1)$ factor in the sample complexity bound is an interesting open problem. 
Moreover, the additional $L\beta^4$ multiplication is coming from the fact that we are using Hoeffding's inequality along-with Assumption \ref{assump:input}($\va$)~(which implies $\tn{\ubb_{t-1} \otimes \vu_{t}}^2 \leq L \beta^4$) to guarantee persistence of excitation. 
Our result in Theorem~\ref{thrm:persistence_of_excitation} suggests that we can get rid of $L\beta^4$ multiplication in our sample complexity bound by replacing Assumption~\ref{assump:input}($\va$) with Assumption~\ref{assump:input}($\vb$) and using one-sided Bernstein's inequality for non-negative random variables.

\section{Learning State-Space Parameters}\label{sec:statespace}
In this section, we combine Theorems~\ref{thrm:learning_markov_par_main} and \ref{thrm:learning_markov_par_data_IND} with the stability results for the Ho-Kalman algorithm in \cite{oymak2021revisiting} to upper bound the learning error for system matrices. In particular, we will upper bound the estimation error of $\vA, \vB$, and $\vC$ in terms of the estimation error $\tf{\vGhat - \vG}$. For this purpose, we make the following standard assumption on the system matrices.
\begin{assumption}(Controllability/Observability)\label{assump:Ho_Kalman}
	The pair $(\vA, \vB)$ is controllable; the pair $(\vA, \vC)$ is observable.     
\end{assumption}
This assumption implies that the (extended) controllability and observability matrices, defined respectively as
\begin{align}
	\vQ := 
	\begin{bmatrix}
		\vB & \vA\vB &\hdots &\vA^{L/2}\vB    
	\end{bmatrix}, \quad
	\vO := 
	\begin{bmatrix}
		\vC^\T & (\vC \vA)^\T &\hdots & (\vC \vA^{L/2 - 1})^\T    
	\end{bmatrix}^\T
\end{align}
have rank-$n$, when $L \geq 2n$. With these definitions, we define a clipped Hankel matrix as follows,
\begin{align}
	\vH := \vO \vQ \in \R^{pL/2 \times p(L/2+1)}.  
	\label{eqn:clipped_hankel}
\end{align}
Note that $\vH$ can be constructed from the Markov parameters in $\vG$. Similarly, we can also construct $\hat \vH$ from the estimated Markov parameters in $\vGhat$. Then, running the Ho-Kalman Algorithm~\cite{oymak2021revisiting} with input $\hat\vH$ gives us the estimates of $\vA, \vB, \vC$ up to a similarity transformation. Specifically, we get the following estimation guarantees.
\begin{theorem}[Balanced Realization~\cite{oymak2021revisiting}]\label{thrm:Ho-Kalman}
	Consider the clipped Hankel matrix $\vH$, defined in~\eqref{eqn:clipped_hankel}, and let $\hat \vH$ be its noisy estimate, constructed from $\vGhat$. Let $\vAbar, \vBbar, \vCbar$ be the state-space realization obtained from running Ho-Kalman Algorithm with Hankel matrix $\vH$. Similarly, let $\vAhat, \vBhat, \vChat$ be the state-space realization obtained from running Ho-Kalman Algorithm with Hankel matrix $\hat \vH$. Let $\vH^-$ be the matrix obtained by deleting the last $p$ columns of $\vH$, and let $\vL$ be the rank-$n$ approximations of $\vH^-$.
	Suppose Assumption~\ref{assump:Ho_Kalman} holds, $\sigma_{\min}(\vL)>0$, and the estimation error of $\vG$ satisfies $\tf{\vG - \vGhat} \leq {\sigma_{\min}(\vL)}/(2\sqrt{2L})$. Then, there exists a unitary matrix $\vT \in \R^{n \times n}$ such that,
	\begin{align}
			\max \{\tf{\vBbar - \vT^*\vBhat},\tf{\vCbar - \vChat \vT}\} &\leq \sqrt{\frac{L}{\sigma_{\min}(\vL)}} \tf{\vG - \vGhat}. \label{eqn:BC_estimation} \\
			\tf{\vAbar - \vT^* \vAhat \vT} &\leq  \sqrt{L}\big(\frac{\norm{\vH}+\norm{\vHhat }}{\sigma_{\min}(\vL)^2} + \frac{1}{\sigma_{\min}(\vL)}\big) \tf{\vG - \vGhat}. \label{eqn:A_estimation}
	\end{align}
\end{theorem}
The proof of Theorem~\ref{thrm:Ho-Kalman} is presented in Section~\ref{sec:proof_Ho_Kalman}, and is similar to the proof of~\cite{oymak2021revisiting}, except that it is slightly modified to get the error bounds in terms of $\tf{\vG - \vGhat}$ instead of $\norm{\vG - \vGhat}$. Note that, for a stable system and for reasonably large $L$, $\sigma_{\min}(\vL) \approx \sigma_{\min}(\vH)$, and corresponds to the most ``hard to identify'' mode of the system~\cite{oymak2021revisiting}. Also note that $\norm{\vHhat} \leq \norm{\vH}$ whenever $\norm{\vH  -\vHhat} \leq 1$. From the proof of Theorem~\ref{thrm:Ho-Kalman}, we have $\norm{\vH  -\vHhat} \leq \sqrt{L/2}\tf{\vG  -\vGhat} \lesssim \mysqrt{\frac{p^2 L^2 (\sigma_z^2 + \sigma_\vw^2 \beta^2L + \sigma_\ve^2 \beta^2)}{\delta(T-L)}}$ under the conditions of Theorems~\ref{thrm:learning_markov_par_data_IND} and \ref{thrm:Ho-Kalman} with high probability. Hence, choosing $T-L \gtrsim T_0 := p^2 L^2 (\sigma_z^2 + \sigma_\vw^2 \beta^2L + \sigma_\ve^2 \beta^2)/\delta$, we can replace $\norm{\vHhat}$ by $\norm{\vH}$ in \eqref{eqn:A_estimation}. Similarly, we can also make sure that the condition $\tf{\vG - \vGhat} \leq {\sigma_{\min}(\vL)}/(2\sqrt{2L})$ is satisfied by choosing $T-L \gtrsim T_0/\sigma_{\min}(\vL)^2$. Hence, combining Theorems~\ref{thrm:learning_markov_par_data_IND} and \ref{thrm:Ho-Kalman} along-with the above trajectory length requirements, will yield end-to-end learning guarantees on the state-space matrices. Finally, note that unlike~\cite{oymak2021revisiting,tsiamis2019finite,lee2022improved}, our error bounds in Theorem~\ref{thrm:Ho-Kalman} does not depend on the dimension of the hidden state $n$. This is because we are using $\tf{\vG - \vGhat}$~(derived in Theorem~\ref{thrm:learning_markov_par_data_IND}) instead of $\norm{\vG - \vGhat}$ in the proof of Theorem~\ref{thrm:Ho-Kalman}. Hence, our error bounds are useful even in the case of high dimensional hidden states.

 \section{Numerical Experiments} \label{sec:experiments}
For our experiments, we consider a bilinear-observation system with $n = 5$ hidden states and input dimension $p = 3$. Similar to~\cite{oymak2021revisiting}, the state-space matrices are generated as follows: The dynamics matrix $\vA$ is constructed as a diagonal matrix with its $n$ eigenvalues chosen from a uniform distribution between $[0, \rho]$, where $\rho < 1$, and we experiment with different values of $\rho$ to understand its relationship with the estimation error as well as the number of Markov parameters estimated. $\vB, \vC$ are generated with i.i.d. $\Ncal(0, 1/n)$ and $\Ncal(0,1/p)$ entries, respectively. The noise processes $\{\vw_t\}_{t=0}^T$, and $\{z_t\}_{t=0}^T$ are chosen according an exponential distribution, that is, each entries of $\vw_t$ and $z_t$ are independently generated according to the probability distribution function $f(x;\lambda) = \lambda e^{-\lambda x}$ for $x \geq 0$, and $0$ elsewhere. Lastly, we chose the control inputs to be $\{\vu_t\}_{t=0}^T \distas \Ncal(0,\Iden_p)$.

In Figure~\ref{figure1}, we plot the estimation error $\tf{\vG - \vGhat}^2$ over different values of $\rho, L$ and $T$. Each experiment is repeated $20$ times and we plot the mean and one standard deviation. Figure \ref{fig1a} shows that, when the system has a shorter memory (i.e., $\rho(\vA)$ is close to $0$) choosing smaller $L$, as long as $L \geq 2n$, results in smaller estimation error. Contrary to this, if the system has a longer memory (i.e., $\rho(\vA)$ is close to $1$), then choosing larger values of $L$ gives better estimation~(see Figure~\ref{fig1b}). This implies a trade-off between the memory of the system captured by $\rho(\vA)$ and the number of Markov parameters estimated $L$. In Figure~\ref{fig1c}, we show this trade-off more clearly by plotting the estimation error over varying $\rho(\vA)$ and $L$ while fixing $T = 1600$.

We observe double descent curves~\cite{nakkiran2020optimal} in Figures~\ref{fig1a} and \ref{fig1b}. This is because our regression problem is unregularized and has $p^2 L = 9L$ unknown parameters, and the number of covariates is $T-L$. Hence, for $L=50$, we see the peak at $T=500$, and the error decays smoothly after this point. Note that the peak occurs at $T = L + p^2L = 500$~(where the number of unknown parameters become equal to the number of covariates). Similarly, for $L=30$, we see the peak at $T = L+ p^2L = 300$. For $L = \{6,12,18\}$, we do not see the double descent because in Figures~\ref{fig1a} and \ref{fig1b}, we start at $T=100$ and repeat our experiments after every $50$ samples.

\begin{figure}[t!]
    \begin{centering}
        \begin{subfigure}[t]{2.0in}
            \includegraphics[width=\linewidth]{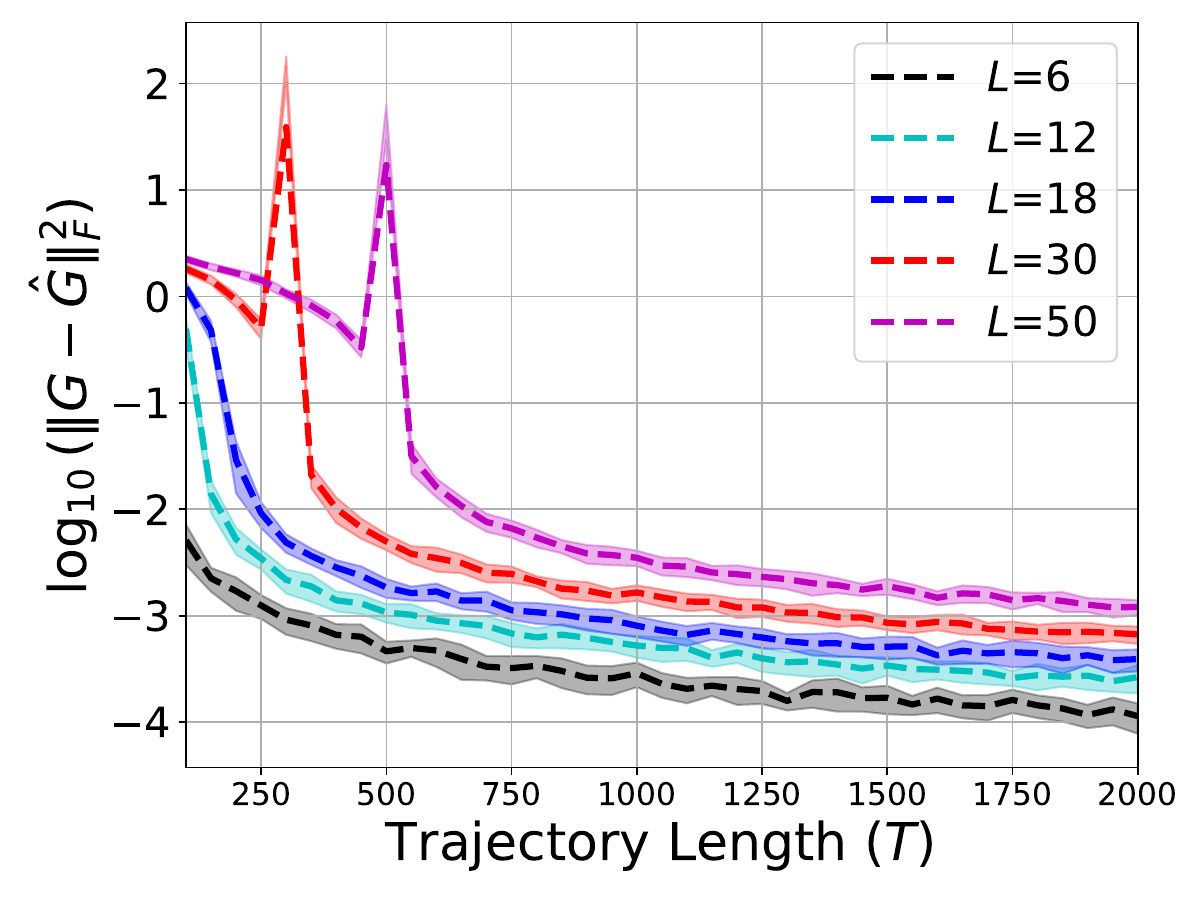}\vspace{-6pt}
            \caption{$\rho(\vA) \leq 0.5$}\label{fig1a}
        \end{subfigure}
    \end{centering}
    ~
    \begin{centering}
        \begin{subfigure}[t]{2.0in}
            \includegraphics[width=\linewidth]{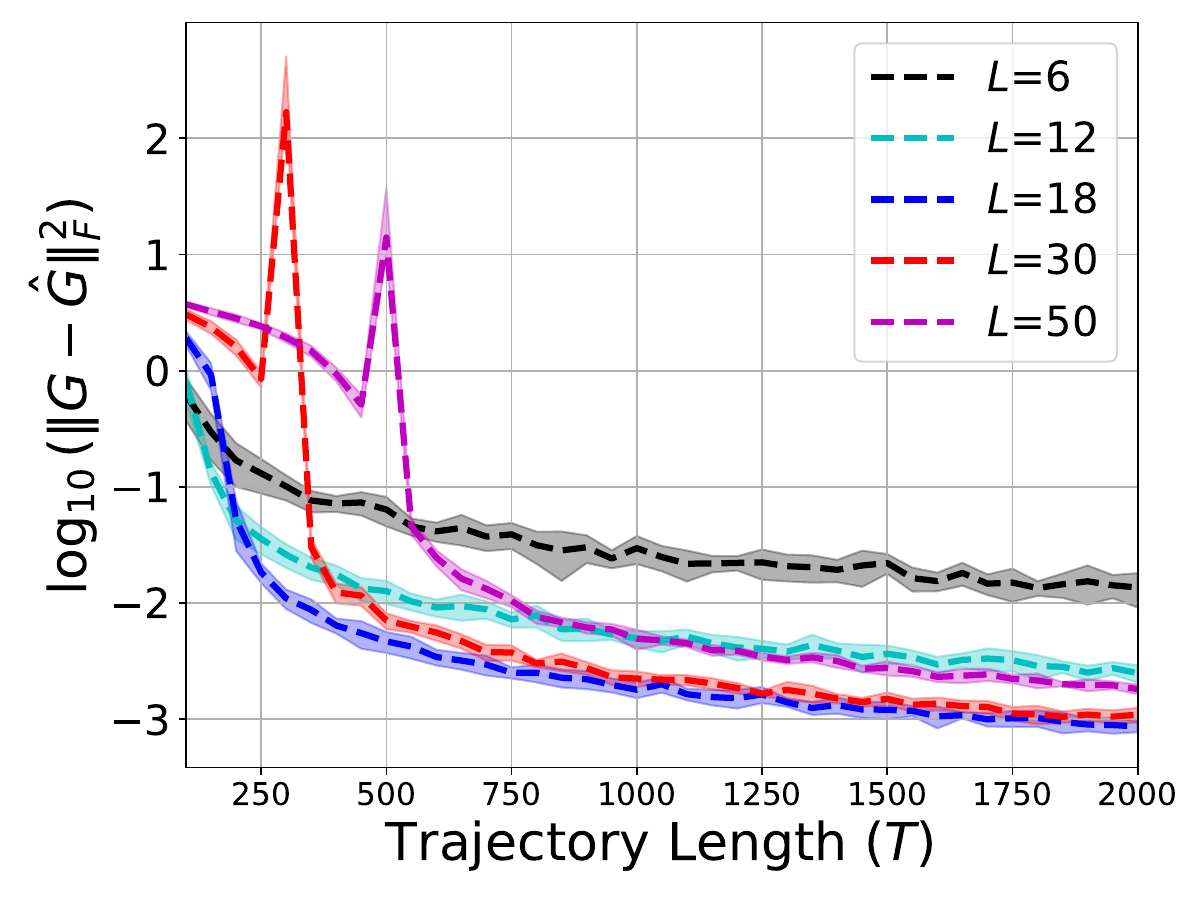}\vspace{-6pt}
            \caption{$\rho(\vA) \leq 0.99$}\label{fig1b}
        \end{subfigure}
    \end{centering}
    ~
    \begin{centering}
        \begin{subfigure}[t]{2.0in}
            \includegraphics[width=\linewidth]{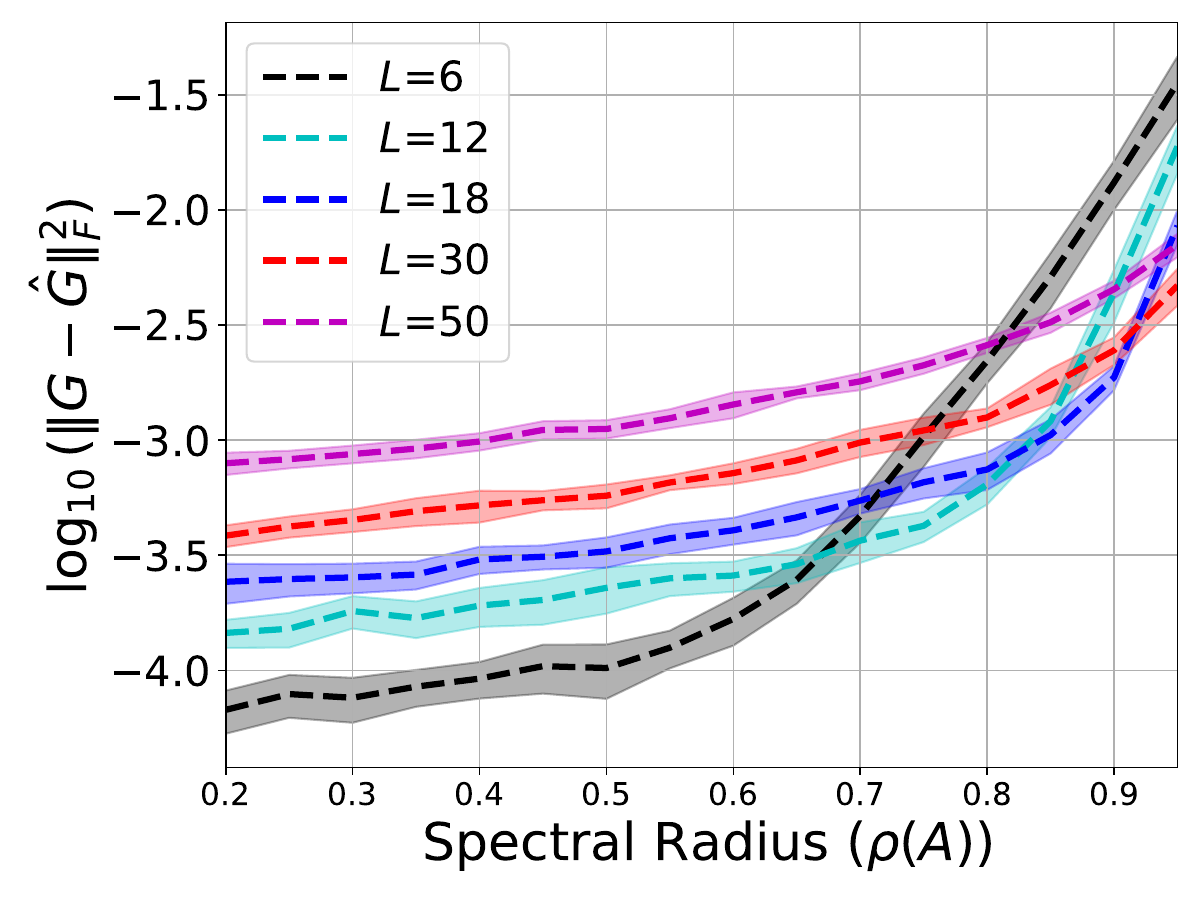}\vspace{-6pt}
            \caption{varying $\rho(\vA)$}\label{fig1c}
        \end{subfigure}
    \end{centering}
    \caption{ We plot the estimation error $\tf{\vG - \vGhat}^2$ over different values of $T, L, \rho(\vA)$ while fixing $n=5$ and $p=3$. Our plots show a trade-off between the memory of the system captured by $\rho(\vA)$ and the number of Markov parameters estimated $L$.
    }
    \label{figure1}
\end{figure}

\section{Proofs of Main Results}\label{sec:proof_main_results}
We now present the proofs of our main results on Markov parameters estimation \& prediction, persistence of excitation, and state space recovery.

\subsection{Learning Markov Parameters~(Data-Dependent Bounds)}
Before we begin the proof, we state a couple of supporting results which are used in the proof of Theorem~\ref{thrm:learning_markov_par_main}. 
Recall the input-output relation~\eqref{eqn:yt_Gutil_Fwtil} used to represent the output of the system~\eqref{eqn:bilinear sys} in terms of its Markov parameters. Letting $\zeta_t :=  \wbb_{t-1}^\T \otimes \vu_t^\T \vek(\vF)  + \vu_t^\T\ve_t + z_t$,  the input-output relation~\eqref{eqn:yt_Gutil_Fwtil} can be alternately expressed as follows: $y_t = \ubb_{t-1}^\T \otimes \vu_t^\T \vek(\vG)  + \zeta_t$. 
The noise process $\zeta_t$ determines the estimation error of least squares problem defined in~\eqref{eqn:ERM_Ghat}, which is given by 
 $\vek(\hat{\vG}) - \vek(\vG) = (\vUtil^\T \vUtil)^{-1}\vUtil^\T \vzeta$, where
 \begin{equation}
\begin{aligned}
 	\vzeta := \begin{bmatrix} \zeta_{L+1} & \zeta_{L+2} & \cdots & \zeta_{T} \end{bmatrix}^\T.
\end{aligned}
\end{equation}
We observe that $\zeta_t= \wbb_{t-1}^\T \otimes \vu_t^\T \vek(\vF)  + \vu_t^\T\ve_t + z_t$ is heavy-tailed and correlated over time. Even if $\vw_t$ and $z_t$ are individually sub-Gaussian, the overall noise process $\zeta_t$ is heavy-tailed. Hence, we use heavy-tailed tools~\cite{mania2022time} to upper bound the estimation error. On a high level, our proof is based on upper bounding the covariance of the estimation error for any given sequence of inputs. This requires upper bounding the auto-covariance function of the heavy-tailed, non-centered, and highly correlated noise process $\{\zeta_t\}$. Finally, we use multidimensional Chebyshev’s inequality to upper bound the ellipsoidal norm of the estimation error with high probability.

We begin by upper bounding the conditional auto-covariance of the effective noise process $\{\zeta_{\tau+1}\}_{\tau = L}^{T-1}$ in terms of the system parameters.

\begin{lemma}[Auto-covariance of effective noise]\label{lemma:auto_covariance}
    Consider the input-output relation~\eqref{eqn:yt_Gutil_Fwtil}. Suppose Assumption~\ref{assump:data/noise} holds, and let $\zeta_{\tau+1} :=  \wbb_{\tau}^\T \otimes \vu_{\tau+1}^\T \vek(\vF)  + \vu_{\tau+1}^\T\ve_{\tau+1} + z_{\tau+1}$ denote the effective noise. Let $\Rcal_\zeta[\tau,\tau'\bgl \vu_{1:T}] :=\E\big[\zeta_{\tau+1} \zeta_{\tau'+1} \bgl \vu_{1:T}\big] - \E\big[\zeta_{\tau+1}\bgl \vu_{1:T}\big]\E\big[\zeta_{\tau'+1}\bgl \vu_{1:T}\big]$ denote the conditional auto-covariance function of the effective noise process $\{\zeta_{\tau+1}\}_{\tau = L}^{T-1}$. Define
\begin{align}
	\bar\vdelta_w(\tau,\tau') := \begin{bmatrix}\delta(\tau-\tau') & \delta(\tau-\tau'+1) & \hdots & \delta(\tau-\tau'+L-1) \\ \delta(\tau-\tau'-1) & \delta(\tau-\tau') & \hdots & \delta(\tau-\tau'+L-2) \\ \vdots & \vdots & \ddots & \vdots \\ \delta(\tau-\tau'-L+1) & \delta(\tau-\tau'-L+2) & \hdots & \delta(\tau-\tau')\end{bmatrix} \in \R^{L \times L} \label{eqn:delta_bar},
\end{align}
to be a Toeplitz matrix of Kronecker delta functions $\delta(\cdot)$, and let
\begin{align}
    \vdelta_w(t,i) := \begin{bmatrix}\delta(t-i) & \delta (t-i-1) & \hdots & \delta(t-i-L+1) \end{bmatrix} \in \R^{1 \times L}. \label{eqn:delta_array}
\end{align}
Then, the conditional auto-covariance function $\Rcal_\zeta[\tau,\tau'\bgl \vu_{1:T}]$ is given as follows,
    \begin{align}
    \Rcal_\zeta[\tau,\tau'\bgl \vu_{1:T}] 
    &= \vek(\vF)^\T \big(\bar\vdelta_w(\tau,\tau') \otimes \vSigma_w \otimes \vu_{\tau+1}\vu_{\tau'+1}^\T\big)\vek(\vF) + \!\! \sum_{i=0}^{\min\{\tau,\tau'\}-L} \!\!\! \vu_{\tau+1}^\T\vC\vA^{\tau-i} \vSigma_w (\vA^{\tau'-i})^\T\vC^\T\vu_{\tau'+1}  \nn \\
	&+ \big[\big(\vu_{\tau'+1}^\T \vC \vA^L \sum_{i=0}^{\tau'-L}\vA^{\tau'-L-i} \big(\vdelta_w(\tau,i) \otimes \vSigma_w\big)\big) \otimes \vu_{\tau+1}^\T\big] \vek(\vF) + \sigma_z^2 \delta(\tau-\tau') \nn \\
	&+ \big[\big(\vu_{\tau+1}^\T \vC \vA^L \sum_{i=0}^{\tau-L}\vA^{\tau-L-i} \big(\vdelta_w(\tau',i) \otimes \vSigma_w\big)\big) \otimes \vu_{\tau'+1}^\T\big] \vek(\vF). \label{eqn:result_of_lemma_autocov}
    \end{align}
\end{lemma}
The proof of Lemma~\ref{lemma:auto_covariance} is deferred to Section~\ref{sec:proof_of_auto_covariance}. To better understand the result of Lemma~\ref{lemma:auto_covariance}, consider the case when $\tau = \tau'$. In this case, 
Lemma~\ref{lemma:auto_covariance} gives the following upper bound on the conditional auto-covariance $\Rcal_\zeta[\tau,\tau \bgl \vu_{1:T}] =\E\big[\zeta_{\tau+1}^2  \bgl \vu_{1:T}\big] - \E\big[\zeta_{\tau+1}\bgl \vu_{1:T}\big]^2$:
\begin{align}
	\Rcal_\zeta[\tau,\tau \bgl \vu_{1:T}] &= \vek(\vF)^\T \big(\bar\vdelta_w(\tau,\tau) \otimes \vSigma_w \otimes \vu_{\tau+1}\vu_{\tau+1}^\T\big)\vek(\vF) + \sum_{i=0}^{\tau-L}\vu_{\tau+1}^\T\vC\vA^{\tau-i} \vSigma_w (\vA^{\tau-i})^\T\vC^\T\vu_{\tau+1}  \nn \\
    & +  2\big[\big(\vu_{\tau+1}^\T \vC \vA^L \sum_{i=0}^{\tau-L}\vA^{\tau-L-i} \big(\vdelta_w(\tau,i) \otimes \vSigma_w\big)\big) \otimes \vu_{\tau+1}^\T\big] \vek(\vF)+ \sigma_z^2, \nn \\
	&\leqsym{i} \vek(\vF)^\T \big(\Iden_L \otimes \vSigma_w \otimes \vu_{\tau+1}\vu_{\tau+1}^\T\big)\vek(\vF) + \| \sum_{i=0}^{\tau-L}\vC\vA^{\tau-i} \vSigma_w (\vA^{\tau-i})^\T\vC^\T\| \tn{\vu_{\tau+1}}^2  + \sigma_z^2, \nn \\
	&\leqsym{ii}  \| \Iden_L \otimes \vSigma_w \otimes \vu_{\tau+1}\vu_{\tau+1}^\T \| \tf{\vF}^2 + \|\vC \vA^L\vGamma_w^{\infty} (\vA^L)^\T\vC^\T\| \tn{\vu_{\tau+1}}^2  + \sigma_z^2, \nn \\
	&\leqsym{iii}  \big(\|\vSigma_w\|  \tf{\vF}^2 + \|\vC\vA^L\|^2 \|\vGamma_w^{\infty}\|  \big) \beta^2  + \sigma_z^2, \label{eqn:lemma_autocov_simplified_tau_equal_tau_prime}
\end{align}
where we get (i) from the observation that $\bar\vdelta_w(\tau,\tau) =  \Iden_L$, and $ \vdelta_w(\tau,i) = 0$ for all $i \notin [\tau-L+1, \tau]$, (ii) from setting $\vGamma_w^{\infty} := \sum_{i=0}^{\infty} \vA^{i}\vSigma_w(\vA^{i})^\T$, and (iii) from the spectral properties of the Kronecker product. \eqref{eqn:lemma_autocov_simplified_tau_equal_tau_prime} illustrates that the conditional variance of the overall noise $\zeta_t$ is equal to the (scaled) summation of the individual variances.
Similarly, we can upper bound the right hand side~(RHS) of \eqref{eqn:result_of_lemma_autocov}, to obtain an upper bound on $\Rcal_\zeta[\tau,\tau' \bgl \vu_{1:T}]$ when $\tau \neq \tau'$. This is done in Section~\ref{sec:proof_of_cond_covariance} for different $(\tau,\tau')$ intervals.

Our next supporting result upper bounds the conditional covariance of the estimation error $\vek(\vGhat) - \vek(\vG) =: \Delta\vG$, and it is central to the proof of our main result in Theorem~\ref{thrm:learning_markov_par_main}.

\begin{theorem}[Conditional Covariance of $\Delta\vG$]\label{thrm:cond_covariance}
 Consider the problem~\eqref{eqn:ERM_Ghat} of estimating the Markov parameter matrix $\vG$ defined in~\eqref{eqn:G_mtx}. Let $\Delta\vG := \vek(\vGhat) - \vek(\vG)$ denote the estimation error, and let $\vSigma[\Delta\vG \bgl \vu_{1:T}] := \E[\Delta\vG \Delta\vG^\T\bgl \vu_{1:T}] - \E[\Delta\vG \bgl \vu_{1:T}] \E[\Delta\vG \bgl \vu_{1:T}]^\T$ denote its conditional covariance. Let $\vF$ be as in \eqref{eqn:G_mtx}, and let $\vGamma_w^{\infty} := \sum_{i=0}^{\infty} \vA^i \vSigma_w (\vA^i)^\T$ denote the infinite time controllability Gramian associated with the process noise. Then, under Assumptions~\ref{assump:stability} and~\ref{assump:data/noise}, we have
 \begin{align}
    \vSigma[\Delta\vG \bgl \vu_{1:T}] 
    & \preceq \bigg[\sigma_z^2 + \|\vSigma_w\|\tf{\vF}^2\beta^2 L \big(1 + \frac{\phi(\vA,\rho)\rho^L}{1-\rho}\ \big)  + \norm{\vGamma_w^\infty} \norm{\vC\vA^L}^2\beta^2 \frac{\phi(\vA,\rho)}{1-\rho}  \bigg] \big(\sum_{\tau = L}^{T-1} \ubb_{\tau}\ubb_{\tau}^\T \otimes \vu_{\tau+1}\vu_{\tau+1}^\T \big)^{-1}. \nn
 \end{align}
\end{theorem}
The proof of Theorem~\ref{thrm:cond_covariance} can be found in Section~\ref{sec:proof_of_cond_covariance}. Excluding the system dependent constants, 
the dependence on the individual noise variances scales as $ \tilde{\Ocal} (\sigma_z^2 + \norm{\vSigma_w}L + \norm{\vC\vA^L}^2\norm{\vGamma_w^\infty})$ for sufficiently large $L$ and stable systems.

Next, we present a lemma to upper bound the weighted norm of the conditional mean estimated error, which will be used to prove our main result in Theorem~\ref{thrm:learning_markov_par_main}.

\begin{lemma}[Conditional Mean of $\Delta\vG$]\label{lemma:cond_mean_norm}
    Consider the same set up of Theorem~\ref{thrm:cond_covariance}. Setting $\vVtil := \vUtil^\T \vUtil$, and assuming $\vVtil \succ 0$, we get the following upper bound on the (weighted) norm of conditionally expected estimation error,
    \begin{align}
        \norm{\E[\Delta\vG \bgl \vu_{1:T}]}_{\vVtil} &\leq \beta^2 \norm{\vB} \norm{\vC}  \frac{\phi(\vA,\rho)\rho^L}{1 - \rho} \sqrt{T-L}, \\
        \text{and} \quad \tn{\E[\Delta\vG \bgl \vu_{1:T}]} 
        &\leq \beta^2 \norm{\vB} \norm{\vC}  \frac{\phi(\vA,\rho)\rho^L}{1 - \rho} \mysqrt{(T-L)\big/\lambda_{\min}\big( \sum_{\tau = L}^{T-1} \ubb_{\tau}\ubb_{\tau}^\T \otimes \vu_{\tau+1}\vu_{\tau+1}^\T \big)}.
    \end{align}
\end{lemma}
The proof of Lemma~\ref{lemma:cond_mean_norm} is deferred to Section~\ref{sec:proof_of_cond_mean_norm}. With our supporting results in Lemma~\ref{lemma:auto_covariance}, Theorem~\ref{thrm:cond_covariance}, and Lemma~\ref{lemma:cond_mean_norm} we are now ready to present the proof of our main result data-depdendent error bound.

\subsubsection{Proof of Theorem~\ref{thrm:learning_markov_par_main}}\label{sec:proof_learning_markov_param}
\begin{proof}
Because of heavy-tailed as well as highly dependent covariates and the effective noise process, we use a similar approach to~\cite{mania2022time} to upper bound the estimation error. For the ease of notation, we define
\begin{align}
    \Xi := \sigma_z^2 + \|\vSigma_w\|\tf{\vF}^2\beta^2L \big(1 + \frac{\phi(\vA,\rho)\rho^L}{1-\rho}\ \big)  + \norm{\vGamma_w^\infty} \norm{\vC\vA^L}^2\beta^2 \frac{\phi(\vA,\rho)}{1-\rho}, \label{eqn:Xi_def_proof}  
\end{align}
such that, we have $\vSigma[\Delta\vG \bgl \vu_{1:T}] \preceq \Xi \big(\sum_{\tau = L}^{T-1} \ubb_{\tau}\ubb_{\tau}^\T \otimes \vu_{\tau+1}\vu_{\tau+1}^\T \big)^{-1}$ according to Theorem~\ref{thrm:cond_covariance}. Combining this with the Chebyshev's inequality,
\begin{align}
	\P\bigg(\sqrt{(\Delta\vG -\E[\Delta\vG \bgl \vu_{1:T}])^\T \vSigma[\Delta\vG \bgl \vu_{1:T}]^{-1}(\Delta\vG -\E[\Delta\vG \bgl \vu_{1:T}])} > \epsilon \bigg) \leq \frac{p^2L}{\epsilon^2}, \label{eqn:Chebyshev_ineq}
\end{align}
we obtain the following upper bound on the weighted Euclidean norm of centered estimation error: We have, 
\begin{align}
	&\P\bigg(\norm{\Delta\vG -\E[\Delta\vG \bgl \vu_{1:T}]}_{\vVtil} > \sqrt{\frac{p^2 L \Xi}{\delta}}~\bigg) \nn \\
	&\qquad\qquad= \P\bigg(\sqrt{(\Delta\vG -\E[\Delta\vG \bgl \vu_{1:T}])^\T \vVtil(\Delta\vG -\E[\Delta\vG \bgl \vu_{1:T}])} > \sqrt{\frac{p^2 L \Xi}{\delta}} \big)\nn \\
	& \qquad \qquad \leq \P\bigg(\sqrt{(\Delta\vG -\E[\Delta\vG \bgl \vu_{1:T}])^\T \vSigma[\Delta\vG \bgl \vu_{1:T}]^{-1}(\Delta\vG -\E[\Delta\vG \bgl \vu_{1:T}])} > \sqrt{\frac{p^2 L}{\delta}} \big) \leq \delta. \label{eqn:centered_est_error_weighted_bound}
\end{align}
Finally, combining~\eqref{eqn:centered_est_error_weighted_bound} with Lemma~\ref{lemma:cond_mean_norm}, we get the statement of Theorem~\ref{thrm:learning_markov_par_main} as follows,
\begin{align}
	\P \bigg(\norm{\vek(\vG)- \vek(\hat{\vG})}_{\vVtil} & \leq \sqrt{\frac{p^2 L \Xi}{\delta}} + \norm{\E[\Delta\vG \bgl \vu_{1:T}]}_{\vVtil}\bigg) \geq 1- \delta,  \nn \\
	 \implies \P \bigg(\norm{\vek(\vG)- \vek(\hat{\vG})}_{\vVtil} &\leq \sqrt{\frac{p^2 L \Xi}{\delta}} + \beta^2 \norm{\vB} \norm{\vC}  \frac{\phi(\vA,\rho)\rho^L}{1 - \rho} \sqrt{T-L} \bigg) \geq 1- \delta.
\end{align}
Following a similar arguments, we combine Theorem~\ref{thrm:cond_covariance} with the Chebyshev's inequality~\eqref{eqn:Chebyshev_ineq} to obtain the following Euclidean norm bound as well,
\begin{align}
&\P\bigg(\tn{\Delta\vG -\E[\Delta\vG \bgl \vu_{1:T}]} > \mysqrt{\frac{p^2 L \Xi \big(\sum_{\tau = L}^{T-1}\ubb_{\tau}\ubb_{\tau}^\T \otimes \vu_{\tau+1}\vu_{\tau+1}^\T\big)^{-1}}{\delta}}\bigg) \nn \\
&\qquad\qquad\leq \P\big(\sqrt{(\Delta\vG -\E[\Delta\vG \bgl \vu_{1:T}])^\T \vSigma[\Delta\vG \bgl \vu_{1:T}]^{-1}(\Delta\vG -\E[\Delta\vG \bgl \vu_{1:T}])} > \sqrt{\frac{p^2 L}{\delta}} \big) \leq \delta,	\label{eqn:centered_est_error_bound}
\end{align}
where $\Xi$ is as defined in \eqref{eqn:Xi_def_proof}. Finally, combining~\eqref{eqn:centered_est_error_bound} with Lemma~\ref{lemma:cond_mean_norm}, we get the following upper bound on the Frobenius norm of the estimation error,
\begin{align}
	\P \bigg(\tf{\vGhat - \vG}  &\leq \mysqrt{\frac{p^2 L \Xi}{\delta \lambda_{\min}\big(\sum_{\tau = L}^{T-1}\ubb_{\tau}\ubb_{\tau}^\T \otimes \vu_{\tau+1}\vu_{\tau+1}^\T\big)}} + \tn{\E[\Delta\vG \bgl \vu_{1:T}]}\bigg) \geq 1- \delta, \nn\\
   \implies\P \bigg(\tf{\vGhat - \vG} &\leq \mysqrt{\frac{p^2 L \Xi}{\delta \lambda_{\min}\big(\sum_{\tau = L}^{T-1}\ubb_{\tau}\ubb_{\tau}^\T \otimes \vu_{\tau+1}\vu_{\tau+1}^\T\big)}} + \frac{\beta^2 \norm{\vB} \norm{\vC}  \frac{\phi(\vA,\rho)\rho^L}{1 - \rho} \sqrt{T-L}}{\sqrt{\lambda_{\min}\big( \sum_{\tau = L}^{T-1} \ubb_{\tau}\ubb_{\tau}^\T \otimes \vu_{\tau+1}\vu_{\tau+1}^\T \big)}}\bigg) \geq 1- \delta. \label{eqn:est_error_ell2_v1}
\end{align}
This completes the proof.
\end{proof}

\subsubsection{Proof of Lemma~\ref{lemma:output_prediction}}\label{sec:proof_output_prediction}
\begin{proof}
Before we begin the proof, let $\vUtil  =  \vQtil \tilde{\vSigma} \vRtil^\T $ be the singular value decomposition, where $\vQtil \in \R^{T-L \times p^2L}$. Then, $\vVtil = \vUtil^\T \vUtil = \vRtil\tilde{\vSigma}^2 \vRtil^\T$. This further implies, $\vVtil^{1/2} = \vRtil\tilde{\vSigma} \vRtil^\T$, and $(\vVtil^{1/2})^\T = \vVtil^{1/2}$. Therefore, we have
\begin{align}
\vu_{T+1}^\T \vGhat \ubb_{T} - \vu_{T+1}^\T \vG \ubb_{T} &= (\vek(\vGhat) - \vek(\vG))^\T (\ubb_T \otimes \vu_{T+1}), \nn \\
&= (\vek(\vGhat) - \vek(\vG))^\T \vVtil^{1/2} \vVtil^{-1/2} (\ubb_T \otimes \vu_{T+1}), \nn \\
&\leq \tn{\vVtil^{1/2}(\vek(\vGhat) - \vek(\vG))} \tn{\vVtil^{-1/2} (\ubb_T \otimes \vu_{T+1})}, \nn \\
\implies (\vu_{T+1}^\T \vGhat \ubb_{T} - \vu_{T+1}^\T \vG \ubb_{T})^2 &\leq \norm{\vek(\vGhat) - \vek(\vG)}_{\vVtil}^2 \norm{\ubb_T \otimes \vu_{T+1}}_{\vVtil^{-1}}^2. \label{eqn:Vtil_Vtil_inverse}
\end{align}
With this, we upper bound the expected output prediction error as follows,
     \begin{align}
         \E[(\yhat_{T+1} - y_{T+1})^2 \bgl \vu_{1:{T+1}}] &= \E[\big(\vu_{T+1}^\T \vGhat \ubb_{T} - \vu_{T+1}^\T \vG \ubb_{T}  - \vu_{T+1}^\T \vF \wbb_{T}  - \vu_{T+1}^\T\ve_{T} - z_{T+1}\big)^2 \bgl \vu_{1:{T+1}}], \nn \\
         & \leqsym{a}  2\E[(\vu_{T+1}^\T \vGhat \ubb_{T} - \vu_{T+1}^\T \vG \ubb_{T})^2\bgl \vu_{1:{T+1}}] + 2\E[(\vu_{T+1}^\T\ve_{T})^2\bgl \vu_{1:{T+1}}] \nn \\
         &+ \E [(\vu_{T+1}^\T \vF \wbb_{T})^2 \bgl \vu_{1:{T+1}}] +\E[ (z_{T+1})^2 \bgl \vu_{1:{T+1}}], \nn \\
         & \leqsym{b}  2\norm{\vek(\vGhat) - \vek(\vG)}_{\vVtil}^2 \norm{\ubb_T \otimes \vu_{T+1}}_{\vVtil^{-1}}^2  + 2\vu_{T+1}^\T\E[\ve_{T}\ve_{T}^\T \bgl \vu_{1:{T}}] \vu_{T+1}   \nn \\
         & + \norm{\vSigma_w} \tn{\vF^\T \vu_{T+1}}^2 + \sigma_z^2, \nn \\
         & \leqsym{c}  2\norm{\vek(\vGhat) - \vek(\vG)}_{\vVtil}^2 \norm{\ubb_T \otimes \vu_{T+1}}_{\vVtil^{-1}}^2  + 2\beta^2 \norm{\vC\vA^L}^2 \norm{\vGamma_u^{(T)} + \vGamma_{w}^{(T)}}   \nn \\
         & + \beta^2\norm{\vSigma_w}\tf{\vF}^2 + \sigma_z^2,
     \end{align}
where we obtain (a) from the pair-wise conditional independence of $\wbb_T, z_{T+1}$ and the remaining terms, and using the identity $-2ab \leq a^2 + b^2$ for any scalars $a,b \in \R$; (b) from \eqref{eqn:Vtil_Vtil_inverse}; and (c) from upper bounding $\E[\ve_{T}\ve_{T}^\T \bgl \vu_{1:{T}}]$ as follows,
\begin{align}
    \E[\ve_{T}\ve_{T}^\T \bgl \vu_{1:{T}}] = \vC \Ab^{L} \E[\vx_{T-L} \vx_{T-L}^\T\bgl \vu_{1:T}](\vA^L)^\T\vC^\T  & =  \vC \Ab^{L}\big( \vGamma_u^{(T)} + \vGamma_w^{(T)} \big)(\vA^L)^\T\vC^\T, \nn \\
    &\preceq \norm{\vC \Ab^{L}}^2\norm{\vGamma_u^{(T)} + \vGamma_w^{(T)}} \Iden_p,
\end{align}
where, we upper bound $\E[\vx_{T-L} \vx_{T-L}^\T\bgl \vu_{1:T}]$ as follows, 
\begin{align}
    \E[\vx_{T-L} \vx_{T-L}^\T\bgl \vu_{1:T}] &= \E\big[\big(\sum_{i=0}^{T-L-1}\vA^{i}\vB \vu_{T-L-i-1} + \vA^{i}\vw_{T-L-i-1}\big)\big(\sum_{i=0}^{T-L-1}\vA^{i}\vB \vu_{T-L-i-1}+ \vA^{i}\vw_{T-L-i-1}\big)^\T\bgl \vu_{1:T}\big], \nn \\
    &= \sum_{i=0}^{T-L-1}\sum_{j=0}^{T-L-1}\vA^i \vB \vu_{T-L-i-1} \vu_{T-L-j-1}^\T \vB^\T (\vA^j)^\T +  \sum_{i=0}^{T-L-1} \vA^i \vSigma_w (\vA^i)^\T, \nn \\
    &\preceq \underbrace{\sum_{i=0}^{T}\sum_{j=0}^{T}\vA^i \vB \vu_{T-i} \vu_{T-j}^\T \vB^\T (\vA^j)^\T}_{\vGamma_u^{(T)}} +  \underbrace{\sum_{i=0}^{T} \vA^i \vSigma_w (\vA^i)^\T}_{\vGamma_w^{(T)}}.
\end{align}
Alternately, we can also upper bound the output prediction error in terms of the Frobenius norm of the estimation error. For this, we first note that,
\begin{align}
(\vu_{T+1}^\T \vGhat \ubb_{T} - \vu_{T+1}^\T \vG \ubb_{T})^2 \leq \tn{\vek(\vGhat) - \vek(\vG)}^2 \tn{\ubb_T \otimes \vu_{T+1}}^2 \leq \beta^4 L \tf{\vGhat - \vG},
\end{align}
where we use $\tn{\ubb_T \otimes \vu_{T+1}}^2 = (\ubb_T^\T \otimes \vu_{T+1}^\T)(\ubb_T \otimes \vu_{T+1}) = \ubb_T^\T \ubb_T \otimes \vu_{T+1}^\T\vu_{T+1} \leq \beta^4L$. Hence, we have

\begin{align}
         \E[(\yhat_{T+1} - y_{T+1})^2 \bgl \vu_{1:{T+1}}] 
         & \leq  2\E[(\vu_{T+1}^\T \vGhat \ubb_{T} - \vu_{T+1}^\T \vG \ubb_{T})^2\bgl \vu_{1:{T+1}}] + 2\E[(\vu_{T+1}^\T\ve_{T})^2\bgl \vu_{1:{T+1}}] \nn \\
         &+ \E [(\vu_{T+1}^\T \vF \wbb_{T})^2 \bgl \vu_{1:{T+1}}] +\E[ (z_{T+1})^2 \bgl \vu_{1:{T+1}}], \nn \\
         & \leq  2\beta^4 L \tf{\vGhat - \vG}^2  + 2\beta^2  \norm{\vC\vA^L}^2 \norm{\vGamma_u^{(T)} + \vGamma_{w}^{(T)}} + \beta^2\norm{\vSigma_w}\tf{\vF}^2 + \sigma_z^2. 
     \end{align}
This completes the proof.
\end{proof}

\subsection{Persistence of Excitation~(Data-Independent Bounds)} \label{sec:proof_persistence_of_excitation}
In this section, we provide a lower bound on the smallest singular value of the design matrix. When combined with Theorem~\ref{thrm:learning_markov_par_main}, this gives consistent estimation of Markov parameter matrix $\vG$. We will guarantee persistence of excitation under two different assumptions. Our first result provides a two-sided concentration bound under the assumption of bounded inputs with zero-mean and isotropic covariance, whereas, our second result provides a one-sided concentration bound under very mild assumption on the inputs, that is, bounded second and fourth moments.

\subsubsection{Two-sided Concentration Bound -- entire spectrum}
Before we state our result, recall Assumption~\ref{assump:input}($\va$) from Section~\ref{sec:main_results}. For the sake of convenience, we are repeating the statement of Assumption~\ref{assump:input}($\va$) here. 
\begin{assumptionp}{\ref{assump:input}($\va$)}(Bounded inputs)\label{assump:input_proof} $\{\vu_t\}_{t=0}^T \distas \Dcal_u$ are stochastic  with zero mean $\E[\vu_t]  = 0$, isotropic covariance $\E[\vu_t \vu_t^\T] = \Iden_p$ and have bounded Euclidean norm, i.e., there exists a scalar $\beta>0$ such that $\tn{\vu_t} \leq \beta$ for all $t \in [T]$.
\end{assumptionp}
We remark that the assumption of bounded inputs can be relaxed to include unbounded stochastic inputs as well. For example, in the case of sub-Gaussian inputs, it is easy to show that, with probability at least $1-\delta$, the inputs are bounded as $\tn{\vu_t} \lesssim \Ocal(\sqrt{p\log(T/\delta)})$ for all $t \in [T]$. Similarly, the assumption of isotropic inputs can also be relaxed to input with positive-definite covariance $\E[\vu_t \vu_t^\T] \succ 0$. Using Assumption~\ref{assump:input_proof}, we get the following persistence of excitation result. 

\begin{proposition}[Persistence of Excitation]\label{prop:two-sided-min-eig}
Consider a sequence of inputs $\{\vu_t\}_{t=0}^T \distas \Dcal_u$ satisfying Assumptions~\ref{assump:input_proof}, and let $\ubb_t := \begin{bmatrix} \vu_{t}^\T & \vu_{t-1}^\T  & \cdots & \vu_{t-L+1}^\T \end{bmatrix}^\T$. Then, for all $\varepsilon \in (0, 1)$, we have 
\begin{align*}
     \PP\left ( \lambda_{\min}\left( \sum_{t=L}^{T-1} (\bar{\vu}_{t} \otimes \vu_{t+1}) (\bar{\vu}_{t} \otimes \vu_{t+1})^\top \right) \ge  (1-\varepsilon)^2 (T-L)  \right ) \ge 1 - 9^{p^2L} \cdot 2(L+1)\exp\left( - \frac{(T-L)\varepsilon^2}{2 (L+1)L \beta^4} \right).
\end{align*}  
\end{proposition}

\begin{remark}
    Note that we can also re-write the statement of Proposition \ref{prop:two-sided-min-eig} as follows: for all $\delta \in(0,1)$, for all $\varepsilon > 0$, the following event 
    \begin{align}
    \lambda_{\min}\left( \sum_{t=L}^{T-1} (\bar{\vu}_{t} \otimes \vu_{t+1}) (\bar{\vu}_{t} \otimes \vu_{t+1})^\top \right) \ge  (1-\varepsilon)^2 (T-L)
    \end{align}
    holds with probability at least $1- \delta$, provided that 
    \begin{align}
    T - L \ge \frac{2(L+1)L \beta^4}{\varepsilon^2}\left( \log\left(\frac{2(L+1)}{\delta} \right) + p^2L \log(9) \right). 
    \end{align}
    Recalling \eqref{eqn:Util_and_y}, this is equivalent to $\lambda_{\min}(\vUtil^\T\vUtil) \geq (1-\varepsilon)^2 (T-L)$, which shows persistence of excitation.
\end{remark}

 \begin{proof} For ease of notation, we introduce $\vUtil = \begin{bmatrix}
\bar{\vu}_{L}  \otimes \vu_{L+1} & \hdots  & \bar{\vu}_{T-1} \otimes \vu_{T} 
\end{bmatrix}^\T$ and note that, 
\begin{align}
    \vUtil^\top \vUtil = \sum_{t=L}^{T-1} (\bar{\vu}_{t} \otimes \vu_{t+1}) (\bar{\vu}_{t} \otimes \vu_{t+1})^\top.
\end{align}
By further noting that $(\bar{\vu}_{t} \otimes \vu_{t+1}) (\bar{\vu}_{t} \otimes \vu_{t+1})^\top = (\bar{\vu}_{t} \bar{\vu}_{t}^\top ) \otimes ( \vu_{t+1} \vu_{t+1}^\top)$, and that $\{\vu_t\}_{t = 0}^T$ are isotropic due to Assumption~\ref{assump:input_proof}, we can also write 
\begin{align}
    \EE[\vUtil^\top \vUtil] = \sum_{t=L}^{T-1} \EE[ (\bar{\vu}_{t} \otimes \vu_{t+1}) (\bar{\vu}_{t} \otimes \vu_{t+1})^\top ] = (T-L) \Iden_{p^2 L}.
\end{align}
Our proof proceeds by establishing a concentration bound on $\Vert \vUtil^\top \vUtil - \EE[\vUtil^\top \vUtil ]\Vert$. We note, by the variational form of the operator norm, that 
\begin{align}
        \left\Vert \vUtil^\top \vUtil - \EE[\vUtil^\top \vUtil]  \right\Vert & = \left\Vert \vUtil^\top \vUtil -  (T-L) \Iden_{p^2 L}  \right\Vert =  \sup_{\vv \in \Scal^{p^2 L -1}}  \left\vert \Vert \vUtil \vv \Vert^2_2 - (T-L)  \right\vert^2.   
\end{align}
The rest of the proof proceeds as follows:

\noindent $\bullet$ {\bf Step 1) Blocking:} Letting $\vv \in \Scal^{p^2 L-1}$, we have 
\begin{align}
    \left\vert \Vert \vUtil \vv \Vert^2_2 - (T-L)  \right\vert  & = \left \vert \sum_{t=L}^{T-1}  \vert (\bar{\vu}_{t}\otimes \vu_{t+1} )^\top \vv \vert^2 - (T- L)    \right \vert, \nn \\
 & = \left \vert \sum_{i=0}^{L} \left(\sum_{k = 1}^{ (T- L)/(L+1)  }    \vert (\bar{\vu}_{k(L+1) + i-1}\otimes \vu_{k(L+1) + i} )^\top \vv \vert^2 - \frac{T- L}{L+1} \right)   \right \vert, \nn \\
 & \le  \sum_{i=0}^{L} \left \vert \sum_{k = 1}^{ (T- L)/(L+1)  }   \left(\vert (\bar{\vu}_{k(L+1) + i - 1}\otimes \vu_{k(L+1) + i} )^\top \vv \vert^2 - 1\right)   \right \vert, \nn  \\
 & = \sum_{i = 0}^L S_i. 
\end{align}
\noindent where we define $ S_i =  \left \vert \sum_{k = 1}^{ (T- L)/(L+1)  }   \left(\vert (\bar{\vu}_{k(L+1) + i - 1}\otimes \vu_{k(L+1) + i} )^\top \vv \vert^2 - 1\right)  \right \vert$ for all $i \in \lbrace 0, \dots, L\rbrace$, and make the simplifying assumption that $T-L$ can be divided by $L+1$. (\emph{Note that, this goes without loss of generality, and we assume it for the sake of clarity. It can be easily avoided by noting that 
$$
\sum_{t=L}^{T-1} (\bar{\vu}_{t} \otimes \vu_{t+1}) (\bar{\vu}_{t} \otimes \vu_{t+1})^\top \succeq \sum_{t=L}^{ (L+1)\lfloor \frac{T-L}{L+1} \rfloor + L} (\bar{\vu}_{t} \otimes \vu_{t+1}) (\bar{\vu}_{t} \otimes \vu_{t+1})^\top, 
$$
where $\lfloor\cdot \rfloor$ denotes the floor operator. We can then analyze everything with $T_0 = (L+1)\lfloor \frac{T-L}{L+1} \rfloor + L$, and note that $ T-L \le  T_0 \le T-1$.}) 

\noindent $\bullet$ {\bf Step 2) Hoeffding's inequality:} Next, let $i \in \lbrace 0, \dots, L\rbrace$, and let us provide a concentration bound on $S_i$. Recall that for all $t \in [T]$, $\tn{\vu_{t}} \leq \beta$ due to Assumption~\ref{assump:input_proof}, and $\tn{\vv} = 1$, thus, we can immediately verify that for all $k \in [(T-L)/(L+1)]$, $0 \le \vert (\bar{\vu}_{k(L+1) + i - 1}\otimes \vu_{k(L+1) + i} )^\top \vv
\vert^2 \leq L \beta^4$. Moreover, we note that $( (\bar{\vu}_{k(L+1) + i - 1}\otimes \vu_{k(L+1) + i} )^\top \vv )_{k \ge 1}$ are independent, thus Hoeffding's inequality applies and we obtain: for all $\rho >0$,
\begin{align}
    \PP\left( S_i \ge \rho  \right) \le 2\exp\left( - \frac{2(L+1)\rho^2}{(T-L)L \beta^4} \right).
\end{align}
Observe that the event $\sum_{i = 0}^{L} S_i > \rho $ is included in the event that there exists $i \in \lbrace 0, \dots, L\rbrace$, $S_i > \rho/(L+1)$. Therefore, we obtain by union bound that for all $\rho \ge 0$,
\begin{align}
     \PP\left(\sum_{i=0}^LS_i > \rho \right) \le  \sum_{i=0}^{L} \PP\left( S_i > \frac{\rho}{L+1}  \right) \le 2(L+1)\exp\left( - \frac{2\rho^2}{(T-L)(L+1)L \beta^4} \right).
\end{align}
Hence, we have just established that for all $\vv \in \Scal^{p^2L-1}$, for all $\rho > 0$, we have 
\begin{align*}
     \PP\left (  \left\vert \Vert \vUtil \vv \Vert^2_2 - (T-L)  \right\vert > \rho  \right )\le 2(L+1)\exp\left( - \frac{2\rho^2}{(T-L)(L+1)L \beta^4} \right).
\end{align*}

\noindent $\bullet$ {\bf Step 3) Covering with $1/4$-net:} Next, we may use a $\epsilon$-net argument to conclude a bound on $\Vert \vUtil^\top  \vUtil - \EE[\vUtil^\top \vUtil] \Vert$. Indeed, by using Lemma 2.5 in \cite{ziemann2023tutorial} with $\epsilon = 1/4$, we obtain that for all $\rho >0$, we have 
\begin{align}
    \PP\left ( \Vert \vUtil^\top  \vUtil - \EE[\vUtil^\top \vUtil] \Vert > \rho  \right ) & \le 9^{p^2 L}\max_{\vv \in \Ncal_{1/4}} \PP\left (  \left\vert \Vert \vUtil \vv \Vert^2_2 - (T-L)  \right\vert > \rho/2  \right ), \nn \\
    & \le 9^{p^2 L} \cdot 2(L+1)\exp\left( - \frac{\rho^2}{2 (T-L)(L+1)L \beta^4} \right),
\end{align}
where $\Ncal_{1/4}$ is a minimal $1/4$-net of the unit sphere $\Scal^{p^2 L-1}$. Reparameterizing $\rho = (T- L)\varepsilon$ and   using Lemma 3 in \cite{jedra2022finite}, we obtain that for all $\varepsilon > 0$, 
\begin{align*}
     \PP\left ( \sigma_{\min}(\vUtil) \ge  (1-\varepsilon)\sqrt{T-L}  \right ) \ge 1 - 9^{p^2L} \cdot 2(L+1)\exp\left( - \frac{(T-L)\varepsilon^2}{2 (L+1)L \beta^4} \right).
\end{align*}    
Recalling that $\lambda_{\min}(\vUtil^\top \vUtil) = \sigma_{\min}^2(\vUtil)$ concludes the proof. 
\end{proof}

\subsubsection{One-sided Concentration Bound -- lower tail of the spectrum}
In this section, we present an alternate persistence of excitation guarantee under very mild assumptions using the proof techniques adapted from~\cite{mania2022time}. Specifically, we work with stochastic inputs $\{\vu_t\}_{t=0}^T$ that satisfy following properties.

\begin{assumptionp}{\ref{assump:input}($\vb$)}(Heavy-tailed inputs)\label{assump:input_proof_onesided} $\{\vu_t\}_{t=0}^T \distas \Dcal_u$ are stochastic  with zero mean $\E[\vu_t]  = 0$, isotropic covariance $\E[\vu_t \vu_t^\T] = \Iden_p$ and there exists a scalar $m_4>0$ such that $ \sup_{\vv \in \Sc^{p^2L-1}} \E[(\vv^\T(\ubb_t  \otimes \vu_{t+1}))^4] \leq m_4$ for all $t \in [T]$, where $\ubb_t := \begin{bmatrix} \vu_{t}^\T & \vu_{t-1}^\T  & \cdots & \vu_{t-L+1}^\T \end{bmatrix}^\T$.
\end{assumptionp}
Assumption~\ref{assump:input_proof_onesided} can be satisfied by many distributions $\Dcal_u$. In Section~\ref{sec:moment_bound_verification}, we show that in the case of Gaussian inputs, Assumption~\ref{assump:input_proof_onesided} holds with $m_4 = 9$. We remark that the assumption of isotropic inputs is for the clarity of exposition, and can be relaxed to input with positive-definite covariance $\E[\vu_t \vu_t^\T] \succ 0$. Using Assumption~\ref{assump:input_proof_onesided}, we get the following lower bound on the smallest eigenvalue of the (scaled) empirical covariance matrix.

\begin{proposition}[Persistence of Excitation]\label{prop:one-sided-min-eig}
Consider a sequence of inputs $\{\vu_t\}_{t=0}^T \distas \Dcal_u$ satisfying Assumptions~\ref{assump:input_proof_onesided}, and let $\ubb_t := \begin{bmatrix} \vu_{t}^\T & \vu_{t-1}^\T  & \cdots & \vu_{t-L+1}^\T \end{bmatrix}^\T$. Suppose the sequence length satisfies the following lower bound,
\begin{align}
    T -L &\geq 32(L+1)m_4\left (\log\left(\frac{2(L+1)}{\delta}\right) +p^2L \log\left(1 + \frac{16 p^2L}{\delta}\right)\right)
\end{align}
Then, for all $\varepsilon \in (0, 1)$, we have 
\begin{align*}
     \PP\left ( \lambda_{\min}\left( \sum_{t=L}^{T-1} (\bar{\vu}_{t} \otimes \vu_{t+1}) (\bar{\vu}_{t} \otimes \vu_{t+1})^\top \right) \ge   (T-L)/4  \right ) \ge 1 - \delta.
\end{align*}  
\end{proposition}

\begin{proof}
To begin, recall that
\begin{align}
	\vUtil^\T \vUtil = \sum_{t = L}^{T-1} (\ubb_t  \otimes \vu_{t+1}) (\ubb_t^\T \otimes \vu_{t+1}^\T) = \sum_{t = L}^{T-1} \ubb_t \ubb_t^\T \otimes \vu_{t+1}\vu_{t+1}^\T =: \sum_{t = L}^{T-1} \vutil_t \vutil_t^\T.
\end{align}
Then, using Assumption~\ref{assump:input_proof_onesided}, we have
\begin{align}
	\E[\vUtil^\T \vUtil] &= \sum_{t = L}^{T-1} \E[\ubb_t \ubb_t^\T] \otimes \E[\vu_{t+1}\vu_{t+1}^\T] = \sum_{t = L}^{T-1} \Iden_{pL} \otimes \Iden_{p} = (T-L) \Iden_{p^2L}.
\end{align}
Next, letting $\vv \in \Scal^{p^2L-1}$, consider the quantity,
\begin{align}
	\vv^\T \vUtil^\T \vUtil \vv = \sum_{t = L}^{T-1} (\vv^\T\vutil_t)^2,
\end{align}
and note that the above quantity can be viewed as a summation of the random process $\{(\vv^\T\vutil_t)^2\}_{t = L}^{T-1}$. In the following we will derive a one-sided concentration bound for this random process. The rest of the proof proceeds as follows:

\noindent $\bullet$ {\bf Step 1) Blocking:} We begin by using blocking technique, similar to the proof of Proposition~\ref{prop:two-sided-min-eig}, to get independent samples, as follows,
\begin{align}
	\sum_{t = L}^{T-1} (\vv^\T\vutil_t)^2 = \sum_{k = 0}^{L} \sum_{\tau = 1}^{(T-L)/(L+1)} (\vv^\T\vutil_{\tau (L+1) + k-1})^2,
\end{align}
where we make the simplifying assumption that $T-L$ can be divided by $L+1$. Recall that, this can be relaxed as discussed in the proof of Proposition~\ref{prop:two-sided-min-eig}. 

\noindent $\bullet$ {\bf Step 2) Bernstein's inequality for non-negative random variables:} Before we apply Bernstein's inequality, we upper bound the mean and second moment of the random variable $(\vv^\T\vutil_t)^2$ using Assumption~\ref{assump:input_proof_onesided} as follows, 
\begin{align}
	\E[(\vv^\T\vutil_t)^2] &= \E[\vv^\T\vutil_t \vutil_t^\T \vv] = \vv^\T\E[\vutil_t \vutil_t^\T]\vv = \vv^\T \Iden_{p^2L} \vv = 1. \\
 \text{and} \quad \E[(\vv^\T\vutil_t)^4] &= \E[(\vv^\T(\ubb_t \otimes \vu_{t+1}))^4] \leq m_4. 
\end{align}
Then, using one-sided Bernstein's inequality for non-negative random variables, we have
\begin{align}
	\P \bigg(\sum_{\tau = 1}^{(T-L)/(L+1)} \big((\vv^\T\vutil_{\tau (L+1) + k-1})^2 - \E[(\vv^\T\vutil_{\tau (L+1) + k-1})^2]\big)  \leq - \frac{T-L}{(L+1)}z\bigg) &\leq \exp\bigg(-\frac{(T-L)z^2}{2(L+1)m_4}\bigg), \\
	\implies \quad \P \bigg(\sum_{\tau = 1}^{(T-L)/(L+1)} (\vv^\T\vutil_{\tau (L+1) + k -1})^2  \leq \frac{T-L}{L} (1-z)\bigg) &\leq \exp\bigg(-\frac{(T-L)z^2}{2(L+1)m_4}\bigg).
\end{align} 
Union bounding $L+1$ such events, we obtain, 
\begin{align}
	\P \bigg(\sum_{k = 0}^{L}\sum_{\tau = 1}^{(T-L)/(L+1)} (\vv^\T\vutil_{\tau (L+1) + k -1})^2  \leq (T-L) (1-z)\bigg) &\leq (L+1)\exp\bigg(-\frac{(T-L)z^2}{2(L+1)m_4}\bigg). \label{eqn:one_sided_bernstein}
\end{align}
This can be alternately represented by letting
\begin{align}
	(L+1)\exp\bigg(-\frac{(T-L)z^2}{2(L+1)m_4}\bigg) &= \delta, \nn \\
	\iff \frac{(T-L)z^2}{2(L+1)m_4} &= \log((L+1)/\delta), \nn \\
	\impliedby z &= \sqrt{\frac{L+1}{T-L}2 m_4\log((L+1)/\delta)}.
\end{align}
Hence, we have
\begin{align}
	\P \bigg(\sum_{t = L}^{T-1} (\vv^\T\vutil_t)^2  \leq (T-L) -\sqrt{2m_4(L+1)(T-L)\log((L+1)/\delta)}\bigg) &\leq \delta.
\end{align}

\noindent $\bullet$ {\bf Step 3) Covering with ${\delta}/{(8p^2L)}$-net:} Next, we will use a covering argument to lower bound the minimum eigenvalue of the (scaled) empirical covariance matrix as follows: Let $\Ncal_\epsilon := \{\vv_1, \vv_2, \dots, \vv_{|\Ncal_\epsilon|} \} \subset \Scal^{p^2L-1}$ be the $\epsilon$-net of $\Scal^{p^2L-1}$ such that for any $\vv \in \Scal^{p^2L-1}$, there exists $\vv_i \in \Ncal_\epsilon$ such that $\tn{\vv - \vv_i} \leq \epsilon$. From Lemma 5.2 of \cite{vershynin2010introduction}, we have $|\Ncal_\epsilon| \leq (1 + 2/ \epsilon)^{p^2L}$. 

Let us choose $\vv \in \Scal^{p^2L-1}$ for which $\lambda_{\min} (\vUtil^\T \vUtil) = \vv^\T \vUtil^\T \vUtil \vv$, and choose $\vv_i \in \Ncal_\epsilon$ which approximates $\vv$ as $\tn{\vv - \vv_i} \leq \epsilon$. By triangle inequality, we have
\begin{align}
	|\vv^\T \vUtil^\T \vUtil \vv  - \vv_i^\T \vUtil^\T \vUtil \vv_i| &= |\vv^\T \vUtil^\T \vUtil ( \vv - \vv_i) + (\vv - \vv_i)^\T \vUtil^\T \vUtil \vv_i|, \nn \\
	& \leq \norm{\vUtil^\T \vUtil} \tn{\vv} \tn{\vv - \vv_i} + \norm{\vUtil^\T \vUtil} \tn{\vv_i} \tn{\vv - \vv_i}, \nn \\
	& \leq 2 \epsilon \norm{\vUtil^\T \vUtil}. 
\end{align}
This further implies,
 \begin{align}
     \vv^\T \vUtil^\T \vUtil \vv \geq  \vv_i^\T \vUtil^\T \vUtil \vv_i - 2 \epsilon \norm{\vUtil^\T \vUtil}, \nn \\
     \implies \lambda_{\min}(\vUtil^\T \vUtil ) \geq \inf_{\vv_i \in \Ncal_\epsilon} \vv_i^\T \vUtil^\T \vUtil \vv_i - 2 \epsilon \norm{\vUtil^\T \vUtil}.
 \end{align}
Hence, in order to lower bound $\lambda_{\min}(\vUtil^\T \vUtil )$, we also need an upper bound on $\norm{\vUtil^\T \vUtil}$. This can be done as follows: First, we have 
\begin{align}
	\E[\norm{\vUtil^\T \vUtil}] = \E[\lambda_{\max}(\vUtil^\T \vUtil)] \leq \E[\tr(\vUtil^\T\vUtil)] = \tr\big(\E[\vUtil^\T \vUtil]\big) = (T-L)\tr\big(\Iden_{p^2L}\big) = p^2L(T-L).
\end{align}
Then, using Markov's inequality, we have
\begin{align}
	\P\bigg( \norm{\vUtil^\T \vUtil} > \frac{p^2L(T-L)}{\delta}\bigg) \leq  \frac{\E[ \norm{\vUtil^\T \vUtil} ]}{p^2L(T-L)} \delta \leq \delta.
\end{align}
Let  $\Ecal := \{\norm{\vUtil^\T \vUtil}  \leq \frac{2p^2L(T-L)}{\delta}\}$ denote the event that $\norm{\vUtil^\T \vUtil}$ is bounded by the specified threshold. Then, it is straightforward to see that $\P(\Ecal) \geq 1 - \delta/2$. This also implies,
\begin{align}
	\P\bigg(\lambda_{\min}(\vUtil^\T\vUtil) < (T-L)(1/2-z)\bigg)  &\leq \P\bigg(\big\{\lambda_{\min}(\vUtil^\T\vUtil) < (T-L)(1/2-z)\big\} \bigcap \Ecal \bigg) + \P (\Ecal^c), \nn \\
	&\leq \P\bigg(\big\{\inf_{\vv_i \in \Ncal_\epsilon} \vv_i^\T \vUtil^\T \vUtil \vv_i - 2 \epsilon \norm{\vUtil^\T \vUtil} < (T-L)(1/2-z)\big\} \bigcap \Ecal \bigg) + \delta/2, \nn \\
	&\leq \P\bigg(\inf_{\vv_i \in \Ncal_\epsilon} \vv_i^\T \vUtil^\T \vUtil \vv_i  < (T-L)(1/2-z) +  4 \epsilon\frac{p^2L(T-L)}{\delta}\bigg) + \delta/2, \nn \\
	&= \P\bigg(\inf_{\vv_i \in \Ncal_\epsilon} \vv_i^\T \vUtil^\T \vUtil \vv_i  < (T-L)(1/2-z +  \frac{4 \epsilon p^2L}{\delta})\bigg) + \delta/2, \nn \\
	&= \P\bigg(\inf_{\vv_i \in \Ncal_\epsilon} \vv_i^\T \vUtil^\T \vUtil \vv_i  < (T-L)(1-z )\bigg) + \delta/2, \label{eqn:before_union_bound}
\end{align}
where we obtained the last inequality by choosing $\epsilon = \frac{\delta}{8p^2L}$. Using~\eqref{eqn:one_sided_bernstein} with union bounding over all the elements in $\Ncal_\epsilon$, we obtain,
\begin{align}
	 \P\bigg(\inf_{\vv_i \in \Ncal_\epsilon} \vv_i^\T \vUtil^\T \vUtil \vv_i  < (T-L)(1-z )\bigg) &\leq |\Ncal_\epsilon| (L+1)\exp\bigg(-\frac{(T-L)z^2}{2(L+1)m_4}\bigg).
\end{align}
 From Lemma 5.2 of \cite{vershynin2010introduction}, we have $|\Ncal_\epsilon| \leq (1 + 2/ \epsilon)^{p^2L}$. Hence, we have
\begin{align}
    \P\bigg(\inf_{\vv_i \in \Ncal_\epsilon} \vv_i^\T \vUtil^\T \vUtil \vv_i  < (T-L)(1-z )\bigg) &\leq (1 + \frac{16 p^2L}{\delta})^{p^2L} \cdot (L+1)\exp\bigg(-\frac{(T-L)z^2}{2(L+1)m_4}\bigg).
\end{align}
This can be alternately represented by letting,
\begin{align}
	(1 + \frac{16 p^2L}{\delta})^{p^2L} \cdot (L+1)\exp\bigg(-\frac{(T-L)z^2}{2(L+1)m_4}\bigg) &= \delta/2, \nn \\
	 \iff \exp\bigg(-\frac{(T-L)z^2}{2(L+1)m_4}\bigg) &= \delta/(2(L+1)) \cdot (1 + \frac{16 p^2L}{\delta})^{-p^2L}, \nn \\
	 \iff \frac{(T-L)z^2}{2(L+1)m_4} &= \log(2(L+1)/\delta) +p^2L \log\big(1 + \frac{16 p^2L}{\delta}\big), \nn \\
	\impliedby z &= \mysqrt{\frac{2(L+1)m_4}{(T-L)}\big(\log\big(\frac{2(L+1)}{\delta}\big) +p^2L \log\big(1 + \frac{16 p^2L}{\delta}\big)\big)}.
\end{align}
Plugging this back into \eqref{eqn:before_union_bound}, we have
\begin{align}
		\P\bigg(\lambda_{\min}(\vUtil^\T\vUtil) < (T-L)/2 -   \sqrt{2(L+1)(T-L)m_4\big(\log\big(\frac{2(L+1)}{\delta}\big) +p^2L \log\big(1 + \frac{16 p^2L}{\delta}\big)\big)}\bigg) \leq \delta  
\end{align}
Finally, choosing the trajectory length via
\begin{align}
	(T-L)/4 &\geq \sqrt{2(L+1)(T-L)m_4\big(\log\big(\frac{2(L+1)}{\delta}\big) +p^2L \log\big(1 + \frac{16 p^2L}{\delta}\big)\big)}, \nn \\
	 \iff (T-L)/16 &\geq 2(L+1)m_4\big(\log\big(\frac{2(L+1)}{\delta}\big) +p^2L \log\big(1 + \frac{16 p^2L}{\delta}\big)\big), \nn \\
	 \impliedby T -L &\geq 32(L+1)m_4\big(\log\big(\frac{2(L+1)}{\delta}\big) +p^2L \log\big(1 + \frac{16 p^2L}{\delta}\big)\big),
\end{align}
we have  
\begin{align}
	\P\big(\lambda_{\min}(\vUtil^\T\vUtil) \geq (T-L)/4 \big) \geq 1 - \delta.  
\end{align}
This completes the proof.
\end{proof}

\subsubsection{Proof of Theorem~\ref{thrm:learning_markov_par_data_IND}} \label{sec:proof_learning_markov_par_data_IND}
\begin{proof}
    The proof of Theorem~\ref{thrm:learning_markov_par_data_IND} relies on combining The proofs of Theorems~\ref{thrm:learning_markov_par_main} and \ref{thrm:persistence_of_excitation}. We begin by defining the event $\Ecal := \{\lambda_{\min}\big(\sum_{\tau = L}^{T-1}\ubb_{\tau}\ubb_{\tau}^\T \otimes \vu_{\tau+1}\vu_{\tau+1}^\T\big) \geq (T-L)/4\}$. Then, under the conditions of Theorem~\ref{thrm:persistence_of_excitation}, we have $\P(\Ecal) \geq 1 -\delta$. Next, recalling the definition of $\Xi$ from \eqref{eqn:Xi_def_proof}, we have
    \begin{align}
        &\P \left( \tf{\vGhat - \vG} > \mysqrt{\frac{4 p^2 L \Xi}{\delta(T-L)}} + 2\beta^2 \norm{\vB} \norm{\vC}  \frac{\phi(\vA,\rho)\rho^L}{1 - \rho}\right) \nn \\
        &\qquad\leq \P \left( \left\{\tf{\vGhat - \vG} > \mysqrt{\frac{4 p^2 L \Xi}{\delta(T-L)}} + 2\beta^2 \norm{\vB} \norm{\vC}  \frac{\phi(\vA,\rho)\rho^L}{1 - \rho}\right\} \bigcap \Ecal\right) + \P (\Ecal^c), \nn \\
        &\qquad \leq \P \bigg(\tf{\vGhat - \vG} \leq \mysqrt{\frac{p^2 L \Xi}{\delta \lambda_{\min}\big(\sum_{\tau = L}^{T-1}\ubb_{\tau}\ubb_{\tau}^\T \otimes \vu_{\tau+1}\vu_{\tau+1}^\T\big)}} + \frac{\beta^2 \norm{\vB} \norm{\vC}  \frac{\phi(\vA,\rho)\rho^L}{1 - \rho} \sqrt{T-L}}{\sqrt{\lambda_{\min}\big( \sum_{\tau = L}^{T-1} \ubb_{\tau}\ubb_{\tau}^\T \otimes \vu_{\tau+1}\vu_{\tau+1}^\T \big)}}\bigg) + \delta, \nn \\
        & \qquad \leqsym{a} 2 \delta,
    \end{align}
    where we obtain (a) from \eqref{eqn:est_error_ell2_v1} in Section~\ref{sec:proof_learning_markov_param}. Hence, under the conditions of Theorems~\ref{thrm:learning_markov_par_main} and \ref{thrm:persistence_of_excitation}, with probability at least $1-2\delta$, we have
    $
        \tf{\vGhat - \vG} \leq  \sqrt{9 p^2 L \Xi}/\sqrt{\delta(T-L)},
    $
    provided that we choose $L$ via,
    \begin{align}
        2\beta^2 \norm{\vB} \norm{\vC}  \frac{\phi(\vA,\rho)\rho^L}{1 - \rho} &\leq \mysqrt{\frac{p^2 L \Xi}{\delta(T-L)}}, \nn\\
        \iff \rho^L &\leq \frac{(1-\rho)}{ 2\beta^2 \norm{\vB} \norm{\vC} \phi(\vA,\rho)}  \mysqrt{\frac{p^2 L \Xi}{\delta(T-L)}}, \nn \\
        \impliedby L &\geq \frac{1}{2 \log (\rho^{-1})}\log \left( \frac{(T-L) \delta}{p^2L \Xi}\right) + \frac{1}{\log(\rho^{-1})} \log \left( \frac{2\beta^2 \norm{\vB} \norm{\vC} \phi(\vA,\rho)}{1- \rho}\right).
    \end{align}
finally replacing $\delta$ with $\delta/2$ gives the statement of the theorem.
\end{proof}

\subsection{State-space Parameter Estimation}\label{sec:proof_Ho_Kalman}
 \subsubsection{Proof of Theorem~\ref{thrm:Ho-Kalman}}
 \begin{proof}
The proof is similar to that in \cite{oymak2021revisiting}, except that it is slightly modified to get the upper bounds in terms of $\tf{\vG - \vGhat}$ instead of $\norm{\vG -\vGhat}$. Recalling Algorithm 1 in \cite{oymak2018non}, let
\begin{itemize}
	\item $\vH, \vH^-, \vH^+, \vL, \vO, \vQ$ be the matrices corresponding to $\vG$.
	\item $\vHhat, \vHhat^-, \vHhat^+, \vLhat, \vOhat, \vQhat$ be the matrices corresponding to $\vGhat$.
	\item $\vAbar, \vBbar, \vCbar$ be the output of Ho-Kalman Algorithm with input $\vH$.
	\item $\vAhat, \vBhat, \vChat$ be the output of Ho-Kalman Algorithm with input $\vHhat$.
\end{itemize}

\noindent $\bullet$ {\bf Step (1)} \emph{Perturbation bounds for $\vH, \vHhat$ and $\vL, \vLhat$:}
Let $\vH[i] - \vHhat[i] \in \R^{p \times p(L/2+1)}$ denote the $i$-th block row of $\vH - \vHhat$. Since each $\vH[i] - \vHhat[i]$ is a sub-matrix of the Markov parameters matrix difference $\vG - \vGhat$, we have $\tf{\vH[i] - \vHhat[i]} \leq \tf{\vG - \vGhat}$. This further implies,
\begin{align}
	\tf{\vH -  \vHhat}^2 \leq \sum_{i = 1}^{L/2} \tf{\vH[i] - \vHhat[i]}^2 \leq L/2 \tf{\vG - \vGhat}^2. \label{eqn:H_square}
\end{align}
Since $\vH^- - \vHhat^-$ and $\vH^+ - \vHhat^+$ are the sub-matrices of $\vH - \vHhat$, therefore, we have
\begin{align}
	\tf{\vH^- - \vHhat^-} &\leq  \tf{\vH -  \vHhat} \leq \sqrt{L/2} \tf{\vG - \vGhat}. \label{eqn:H_minus_v1} \\
	\tf{\vH^+ - \vHhat^+} &\leq  \tf{\vH -  \vHhat} \leq \sqrt{L/2} \tf{\vG - \vGhat}. \label{eqn:H_pluss_v1}
\end{align}
To proceed, recall that $\vL = \vH^-$ by definition and $\vLhat$ is the best rank-$n$ approximations of $\vHhat^-$. This implies that $\tf{\vHhat^- - \vLhat} \leq \tf{\vHhat^- - \vM}$, for any rank-$n$ matrix $\vM$. Choose $\vM = \vH^-$, we have
\begin{align}
	\tf{\vL - \vLhat} \leq \tf{\vL - \vHhat^-} + \tf{\vHhat^- - \vLhat} \leq \tf{\vH^- - \vHhat^-} + \tf{\vHhat^- - \vH^-}    \leq \sqrt{2L} \tf{\vG - \vGhat}. \label{eqn:L_perturbation_final}
\end{align}

\noindent $\bullet$ {\bf Step (2)} \emph{Perturbation bounds for $\vObar, \vOhat$ and $\vQbar, \vQhat$:} Let $\vL, \vLhat$ have the singular value decomposition $\vL = \vU \vSigma \vV^*$, and $\vLhat = \vUhat \hat\vSigma \vVhat^*$. Then. from Algorithm 1 in \cite{oymak2021revisiting}, we have $\vObar = \vU \vSigma^{1/2}$, $\vQbar = \vSigma^{1/2}\vV^*$, $\vOhat = \vUhat \hat\vSigma^{1/2}$, and $\vQhat = \hat\vSigma^{1/2}\vVhat^*$. Let $\vT$ be an $n \times n$ unitary matrix. Then, applying Lemma 5.14 of \cite{tu2015low} with the robustness condition $\sigma_{\min}(\vL) \geq 2\norm{\vL - \vLhat}$, we get
\begin{align}
	\tf{\vObar - \vOhat\vT}^2 + \tf{\vQbar - \vT^{-1}\vQhat}^2 \leq \frac{2}{\sqrt{2} - 1} \frac{\tf{\vL - \vLhat}^2}{\sigma_{\min}(\vL)}.  \label{eqn:direct_tu_application} 
\end{align} 
Combining this with \eqref{eqn:L_perturbation_final}, we have
\begin{align}
	\tf{\vObar - \vOhat\vT}^2 + \tf{\vQbar - \vT^*\vQhat}^2 \leq 10L\frac{\tf{\vG - \vGhat}^2}{\sigma_{\min}(\vL)}.  \label{eqn:OQ_perturbation_final} 
\end{align} 

\noindent $\bullet$ {\bf Step (3)} \emph{Perturbation bounds for $\vBbar, \vBhat$, and $\vCbar, \vChat$:}
Since $\vCbar - \vChat\vT$ is a sub-matrix of $\vObar - \vOhat\vT$, and $\vB - \vT^*\vB$ is a sub-matrix of $\vQbar - \vT^*\vQhat$, it immediately follows that,
\begin{align}
	\tf{\vCbar - \vChat \vT} &\leq \tf{\vObar - \vOhat\vT} \leq \sqrt{\frac{10L}{\sigma_{\min}(\vL)}} \tf{\vG - \vGhat}, \\
	\tf{\vBbar - \vT^*\vBhat} &\leq \tf{\vQbar - \vT^*\vQhat} \leq \sqrt{\frac{10L}{\sigma_{\min}(\vL)}} \tf{\vG - \vGhat}.
\end{align}

\noindent $\bullet$ {\bf Step (4)} \emph{Perturbation bounds for $\vAbar, \vAhat$:} Let $\vX = \vOhat\vT$, and $\vY = \vT^* \vQhat$. Then, from equation (30-31) in \cite{oymak2021revisiting}, we have

\begin{align}
	\tf{\vAbar - \vT^* \vAhat \vT} \leq \tf{(\vObar^\dagger - \vX^\dagger)\vH^+ \vQbar^\dagger} + \tf{\vX^\dagger(\vH^+ - \vHhat^+)\vQbar^\dagger} + \tf{\vX^\dagger \vHhat^+(\vQbar^\dagger - \vY^\dagger)}. \label{eqn:A_perturbation_split}
\end{align}
From Appendix F in~\cite{oymak2021revisiting}, we have
\begin{align}
	\tf{(\vObar^\dagger - \vX^\dagger)\vH^+ \vQbar^\dagger} &\leq \tf{\vObar^\dagger - \vX^\dagger} \norm{\vH^+}\norm{ \vQbar^\dagger} \leq \tf{\vObar - \vX} \max\{\norm{\vX^\dagger}^2, \norm{\vObar^\dagger}^2\} \norm{\vH^+}\norm{ \vQbar^\dagger}, \nn \\
	&\leq \sqrt{\frac{10L}{\sigma_{\min}(\vL)}} \tf{\vG - \vGhat}\big(\frac{2}{\sigma_{\min}(\vL)}\big)\big(\sqrt{\frac{2}{\sigma_{\min}(\vL)}}\big)\norm{\vH^+}, \nn \\
	&\leq \frac{9\sqrt{L}\norm{\vH}}{\sigma_{\min}(\vL)^2} \tf{\vG - \vGhat}. \label{eqn:A_split_first_final}
\end{align}
Next, from Appendix F in~\cite{oymak2021revisiting}, we also have
\begin{align}
	\tf{\vX^\dagger(\vH^+ - \vHhat^+)\vQbar^\dagger} &\leq \tf{\vH^+ - \vHhat^+}\norm{\vX^\dagger}\norm{\vQbar^\dagger} \leq \frac{\sqrt{2L}}{\sigma_{\min}(\vL)} \tf{\vG - \vGhat}. \label{eqn:A_split_second_final} \\
\text{and,} \;\; \tf{\vX^\dagger \vHhat^+(\vQbar^\dagger - \vY^\dagger)} &\leq  \tf{\vQbar^\dagger - \vY^\dagger} \norm{\vHhat^+}\norm{ \vX^\dagger} \leq \frac{9\sqrt{L}\norm{\vHhat}}{\sigma_{\min}(\vL)^2} \tf{\vG - \vGhat}. \label{eqn:A_split_third_final} 
\end{align}
Combining \eqref{eqn:A_split_first_final}, \eqref{eqn:A_split_second_final}, and \eqref{eqn:A_split_third_final} into \eqref{eqn:A_perturbation_split}, and using Appendix G in~\cite{oymak2021revisiting} we have
\begin{align}
		\tf{\vAbar - \vT^* \vAhat \vT} \leq  \big(\frac{9\sqrt{L}(\norm{\vH}+\norm{\vHhat })}{\sigma_{\min}(\vL)^2} + \frac{\sqrt{2L}}{\sigma_{\min}(\vL)}\big) \tf{\vG - \vGhat}
\end{align}
Observe that the robustness condition is satisfied by assuming $\tf{\vG - \vGhat} \leq \frac{\sigma_{\min}(\vL)}{2\sqrt{2L}}$
\end{proof}

\section{Conclusion and Discussion}\label{sec:conclusion}
We provide the first non-asymptotic error bounds and sample complexity analysis for learning a realization of a partially observed dynamical system with linear state transitions and bilinear observations, given a single finite trajectory of input-output samples. Under stability and very mild assumptions on the process and measurement noises, we provide both data-dependent error bound, precisely capturing the shape of uncertainty, and data independent error bound, optimally scaling as $\tilde{\Ocal}(1/\sqrt{T-L})$. Our results show that despite getting heavy-tailed covariates and overall noise process due to bilinear observation, the non-asymptotic error bounds for bilinear-observation system are similar to that of the standard LTI system with partial state observation, except some additional dependencies arise because of using heavy-tailed tools.  Of independent interest, we also establish persistence of excitation results for the heavy-tailed and correlated covariates used to learn the system's Markov parameters. Finally, numerical experiments corroborate our theory and highlight a trade-off between the estimation error and the system's memory. 

There are several open questions leading to interesting directions for future work. 
First, can our method be extended to linear/bilinear systems which are marginally stable, i.e. the spectral radius $\rho(\vA)=1$? 
Our current estimator~\eqref{eqn:ERM_Ghat} is biased by the state $L$ steps in the past, which is not guaranteed to shrink if the dynamics are only marginally stable.
One possible approach for overcoming this challenge is to build an estimator that regresses the output against a finite history of inputs \emph{and} outputs.
Such an estimator arises naturally from the predictor form of the Kalman filter, which has been used for identification in the LTI setting~\cite{tsiamis2019finite,lale2020logarithmic}.
In our linear/bilinear setting, the predictor form involves \emph{bilinear} transitions,
so it is not straightforward to immediately extend this technique.
Alternatively, estimators can be constructed from instrumental variables in a two stage regression process. This perspective has been successfully used for LTI identification~\cite{simchowitz2019learning,bakshi2023new}, but its extension to the linear/bilinear setting is similarly nontrivial.
Lastly, spectral filtering is an alternative method for handling marginal stability in LTI systems, though it has been used mainly for prediction rather than identification~\cite{hazan2017learning,hazan2018spectral}.

Second, what is the best way to make \emph{predictions} in this linear/bilinear setting?
When the parameters $(A,B,C)$ are known and the inputs are observed, filtering techniques for linear time-varying (LTV) systems are directly applicable.
It is of interest to characterize the sensitivity of these methods to estimation errors, or to develop robust methods, in order to analyze the sample complexity of state estimation and output prediction
when the dynamics are unknown. 
Another interesting direction is to investigate filtering approaches that can use estimated Markov parameter directly, sidestepping the state-space recovery step~\cite{vincent2023prediction}.
Lastly, since the observations are bilinearly influenced by inputs, it is interesting to consider how the choice of inputs can enhance (or diminish) achievable prediction accuracy.
This line of inquiry is related to classic problems in sensor design and placement~\cite{tanaka2016semidefinite,tzoumas2016sensor}.

Finally, what is the best achievable \emph{control performance} given limited data from a linear/bilinear system?
This is a rich but challenging question:
even if the dynamics are known, the optimal control problem does not have a straightforward tractable solution, even for simple linear or quadratic costs.
Due to the impact of inputs on observations, 
the separation principle does not hold and there is a direct trade-off between inputs which are useful for estimating versus controlling the state.

\section*{Acknowledgement}
We thank Samet Oymak, Max Simchowitz,
Mahdi Soltanolkotabi, Stephen Tu, Ingvar Zieman, and Leonardo F. Toso for helpful discussions.
This work was partly funded by NSF CCF 2312774, NSF OAC-2311521,
a LinkedIn Research Award, and a gift from Wayfair. Yassir Jedra is generously supported by the Knut and Alice Wallenberg Foundation Postdoctoral Scholarship Program at MIT- KAW 2022.0366.

\newpage
{
    \bibliographystyle{IEEEtran}
	\bibliography{Bibfiles}
}

\appendix

\section{Proofs of Supporting Results}
\subsection{Proof of Lemma~\ref{lemma:auto_covariance}} \label{sec:proof_of_auto_covariance}
\begin{proof}
    We begin our proof by first estimating the conditional mean of the effective noise process $\{\zeta_{\tau+1}\}_{\tau = L}^{T-1}$, as follows,
    \begin{align}
	\E \big[\zeta_{\tau+1} \bgl \vu_{1:T} \big] &= \E \big[\wbb_{\tau}^\T \otimes \vu_{\tau+1}^\T \vek(\vF)  + \vu_{\tau+1}^\T\ve_{\tau+1} + z_{\tau+1} \bgl \vu_{1:T} \big], \nn \\
	&= \vu_{\tau+1}^\T \E \big[ \ve_{\tau+1} \bgl \vu_{1:T}\big] =  \vu_{\tau+1}^\T\vC \Ab^{L} \E \big[ \xb_{\tau-L+1} \bgl \vu_{1:T}\big], \nn \\
	&=  \vu_{\tau+1}^\T\vC \Ab^{L}  \sum_{i=0}^{\tau-L}\vA^{\tau-L-i}\vB \vu_i = \vu_{\tau+1}^\T\vC \Ab^{L}  \sum_{i=0}^{\tau-L}\vA^{\tau-L-i}\vB \vu_i, \nn \\
	&= \vu_{\tau+1}^\T\sum_{i=0}^{\tau-L}\vC\vA^{\tau-i}\vB \vu_i. \label{eqn:E_zeta_tau_plus}
\end{align}
where we used the fact that $\E [\wb_\tau \bgl \vu_{1:T}] = 0$, $\E [\wbb_\tau \bgl \vu_{1:T}] = 0$, $\E [z_{\tau+1} \bgl \vu_{1:T}] = 0$, and $\vx_0 = 0$ due to Assumption~\ref{assump:data/noise}. Next, we estimate the conditional auto-correlation function of the effective noise process $\{\zeta_{\tau+1}\}_{\tau = L}^{T-1}$, as follows,
\begin{align}
	&\E\big[\zeta_{\tau+1} \zeta_{\tau'+1} \bgl \vu_{1:T}\big] \nn \\ 
	&= \E\big[(\wbb_{\tau}^\T \otimes \vu_{\tau+1}^\T \vek(\vF)  + \vu_{\tau+1}^\T\ve_{\tau+1} + z_{\tau+1})(\wbb_{\tau'}^\T \otimes \vu_{\tau'+1}^\T \vek(\vF)  + \vu_{\tau'+1}^\T\ve_{\tau'+1} + z_{\tau'+1})\bgl \vu_{1:T}\big], \nn \\
	& = \vek(\vF)^\T \E\big[(\wbb_{\tau} \otimes \vu_{\tau+1})(\wbb_{\tau'}^\T \otimes \vu_{\tau'+1}^\T) \bgl \vu_{1:T}\big]\vek(\vF) + \E\big[\vu_{\tau+1}^\T\ve_{\tau+1}\ve_{\tau'+1}^\T\vu_{\tau'+1} \bgl \vu_{1:T}\big] \nn \\
	&\qquad\qquad + \E\big[z_{\tau+1}z_{\tau'+1} \bgl \vu_{1:T}\big] + \E\big[\vu_{\tau+1}^\T\ve_{\tau+1}(\wbb_{\tau'}^\T \otimes \vu_{\tau'+1}^\T) \bgl \vu_{1:T}\big]\vek(\vF) \nn \\
	&\qquad\qquad+ \E\big[\vu_{\tau'+1}^\T\ve_{\tau'+1}(\wbb_{\tau}^\T \otimes \vu_{\tau+1}^\T) \bgl \vu_{1:T}\big]\vek(\vF) +\E\big[z_{\tau+1}(\wbb_{\tau'}^\T \otimes \vu_{\tau'+1}^\T) \bgl \vu_{1:T}\big]\vek(\vF)\nn \\
	&\qquad\qquad +\E\big[z_{\tau'+1}(\wbb_{\tau}^\T \otimes \vu_{\tau+1}^\T) \bgl \vu_{1:T}\big]\vek(\vF) + \E\big[z_{\tau+1}\vu_{\tau'+1}^\T\ve_{\tau'+1} \bgl \vu_{1:T}\big]+\E\big[z_{\tau'+1}\vu_{\tau+1}^\T\ve_{\tau+1} \bgl \vu_{1:T}\big]. \label{eqn:E_zeta_zeta_prime}
\end{align}
In the following, we will calculate each term in \eqref{eqn:E_zeta_zeta_prime} separately to obtain an upper bound on the conditional auto-covariance $\Rcal_\zeta[\tau,\tau'\bgl \vu_{1:T}]$.

\noindent $\bullet$ {\bf First term in \eqref{eqn:E_zeta_zeta_prime}:} This term captures the conditional auto-correlation of the noise process $\{\wbb_\tau\}_{\tau = L}^{T-1}$. Using Assumption~\ref{assump:data/noise}, this term takes the following form,
\begin{align}
	\E\big[(\wbb_{\tau} \otimes \vu_{\tau+1})(\wbb_{\tau'}^\T \otimes \vu_{\tau'+1}^\T) \bgl \vu_{1:T}\big]  &= \E\big[\wbb_{\tau}\wbb_{\tau'}^\T \otimes \vu_{\tau+1}\vu_{\tau'+1}^\T \bgl \vu_{1:T}\big], \nn \\
	&\eqsym{a} \E\big[\wbb_{\tau}\wbb_{\tau'}^\T  \bgl \vu_{1:T}\big] \otimes \E\big[ \vu_{\tau+1}\vu_{\tau'+1}^\T \bgl \vu_{1:T}\big], \nn \\
	& \; = \;\E\big[\wbb_{\tau}\wbb_{\tau'}^\T  \bgl \vu_{1:T}\big] \otimes \vu_{\tau+1}\vu_{\tau'+1}^\T, \nn \\
	&\eqsym{b} \bar\vdelta_w(\tau,\tau') \otimes \vSigma_w \otimes \vu_{\tau+1}\vu_{\tau'+1}^\T, \label{eqn:E_wbar_wbar_prime}
\end{align}
where we get (a) from the independence of $\wbb_{\tau}\wbb_{\tau'}^\T$ and $\vu_{\tau+1}\vu_{\tau'+1}^\T$, (b) from Assumption~\ref{assump:data/noise} and defining a Toeplitz matrix $\bar\vdelta_w(\tau,\tau')$ of Kronecker delta functions $\delta(\cdot)$ as follows, 
\begin{align}
	\bar\vdelta_w(\tau,\tau') := \begin{bmatrix}\delta(\tau-\tau') & \delta(\tau-\tau'+1) & \hdots & \delta(\tau-\tau'+L-1) \\ \delta(\tau-\tau'-1) & \delta(\tau-\tau') & \hdots & \delta(\tau-\tau'+L-2) \\ \vdots & \vdots & \ddots & \vdots \\ \delta(\tau-\tau'-L+1) & \delta(\tau-\tau'-L+2) & \hdots & \delta(\tau-\tau')\end{bmatrix} \in \R^{L \times L} \label{eqn:delta_bar_proof}.
\end{align} 
Using the above definition, we observe that $\E\big[\wbb_{\tau}\wbb_{\tau'}^\T  \bgl \vu_{1:T}\big]$ is an $nL \times nL$ block Toeplitz matrix with the $ij$-th $n \times n$ block given by $\delta(\tau-\tau'-i+j)\vSigma_w$. This further implies that $\E\big[(\wbb_{\tau} \otimes \vu_{\tau+1})(\wbb_{\tau'}^\T \otimes \vu_{\tau'+1}^\T) \bgl \vu_{1:T}\big]$ is an $npL \times npL$ block Toeplitz matrix with the $ij$-th $np \times np$ block given by $\delta(\tau-\tau'-i+j)\vSigma_w\otimes \vu_{\tau+1}\vu_{\tau'+1}^\T$.

\noindent $\bullet$ {\bf Second term in \eqref{eqn:E_zeta_zeta_prime}:} This term captures the conditional auto-correlation of the error process $\{\ve_{\tau+1}\}_{\tau = L}^{T-1}$ which quantifies the effect of unknown state at time $\tau-L+1$. This term can be calculated as follows,
\begin{align}
	\E\big[\vu_{\tau+1}^\T\ve_{\tau+1}\ve_{\tau'+1}^\T\vu_{\tau'+1} \bgl \vu_{1:T}\big] &= \vu_{\tau+1}^\T\E\big[\ve_{\tau+1}\ve_{\tau'+1}^\T \bgl \vu_{1:T}\big]\vu_{\tau'+1}, \nn \\
	&= \vu_{\tau+1}^\T\vC \vA^L\E\big[\vx_{\tau-L+1}\vx_{\tau'-L+1}^\T \bgl \vu_{1:T}\big](\vA^L)^\T\vC^\T\vu_{\tau'+1}. \label{eqn:E_ut_et_et_prime_ut_prime}
\end{align}
To proceed, from \eqref{eqn:bilinear sys} we have
\begin{align}
	\E[\vx_t\vx_{t'}^\T \bgl \vu_{1:T}] &= \E\bigg[\big(\sum_{i=0}^{t-1}\vA^{t-i-1}\vB \vu_i+\sum_{i=0}^{t-1}\vA^{t-i-1}\vw_i\big)\big(\sum_{j=0}^{t'-1}\vA^{t'-j-1}\vB \vu_j+\sum_{j=0}^{t'-1}\vA^{t'-j-1}\vw_j\big)^\T \bgl \vu_{1:T}\bigg], \nn \\
	& = \sum_{i=0}^{t-1}\sum_{j=0}^{t'-1}\vA^{t-i-1}\vB \vu_i \vu_j^\T \vB^\T(\vA^{t'-j-1})^\T + \sum_{i=0}^{t-1}\sum_{j=0}^{t'-1}\vA^{t-i-1} \E [\vw_i \vw_j^\T\bgl \vu_{1:T}] (\vA^{t'-j-1})^\T, \nn \\
	& \eqsym{a} \sum_{i=0}^{t-1}\sum_{j=0}^{t'-1}\vA^{t-i-1}\vB \vu_i \vu_j^\T \vB^\T(\vA^{t'-j-1})^\T + \sum_{i=0}^{\min\{t,t'\}-1}\vA^{t-i-1} \vSigma_w (\vA^{t'-i-1})^\T,
\end{align}
where we obtain (a) by using Assumption~\ref{assump:data/noise}. Hence, we have
\begin{align}
	\E\big[\vx_{\tau-L+1}\vx_{\tau'-L+1}^\T \bgl \vu_{1:T}\big] = \sum_{i=0}^{\tau-L}\sum_{j=0}^{\tau'-L}\vA^{\tau-L-i}\vB \vu_i \vu_j^\T \vB^\T(\vA^{\tau'-L-j})^\T + \sum_{i=0}^{\min\{\tau,\tau'\}-L}\vA^{\tau-L-i} \vSigma_w (\vA^{\tau'-L-i})^\T,
\end{align}
Plugging this back into \eqref{eqn:E_ut_et_et_prime_ut_prime}, we obtain the following,
\begin{align}
	&\E\big[\vu_{\tau+1}^\T\ve_{\tau+1}\ve_{\tau'+1}^\T\vu_{\tau'+1} \bgl \vu_{1:T}\big] \nn \\
	& \quad\quad = \vu_{\tau+1}^\T\vC \vA^L \bigg[\sum_{i=0}^{\tau-L}\sum_{j=0}^{\tau'-L}\vA^{\tau-L-i}\vB \vu_i \vu_j^\T \vB^\T(\vA^{\tau'-L-j})^\T + \sum_{i=0}^{\min\{\tau,\tau'\}-L}\vA^{\tau-L-i} \vSigma_w (\vA^{\tau'-L-i})^\T\bigg](\vA^L)^\T\vC^\T\vu_{\tau'+1}, \nn \\
	& \quad\quad = \vu_{\tau+1}^\T \bigg[\sum_{i=0}^{\tau-L}\sum_{j=0}^{\tau'-L}\vC\vA^{\tau-i}\vB \vu_i \vu_j^\T \vB^\T(\vA^{\tau'-j})^\T\vC^\T + \sum_{i=0}^{\min\{\tau,\tau'\}-L}\vC\vA^{\tau-i} \vSigma_w (\vA^{\tau'-i})^\T\vC^\T\bigg]\vu_{\tau'+1}. \label{eqn:E_etau_etau_prime}
\end{align}

\noindent $\bullet$ {\bf Third term in \eqref{eqn:E_zeta_zeta_prime}:} This term captures the conditional auto-correlation of the measurement noise process $\{z_{\tau+1}\}_{\tau = L}^{T-1}$. Using Assumption~\ref{assump:data/noise}, it is very straightforward to see that
\begin{align}
	\E\big[z_{\tau+1}z_{\tau'+1} \bgl \vu_{1:T}\big] = \sigma_z^2 \delta(\tau-\tau'). \label{eqn:E_ztau_ztau_prime}
\end{align}

The remaining terms in \eqref{eqn:E_zeta_zeta_prime} captures the cross-correlation between various noise terms, and are calculated in the following

\noindent $\bullet$ {\bf Fourth \& Fifth terms in \eqref{eqn:E_zeta_zeta_prime}:} These two terms capture the conditional cross-correlation between the error process $\{\ve_{\tau+1}\}_{\tau = L}^{T-1}$ and the noise process $\{\wbb_\tau\}_{\tau = L}^{T-1}$. To begin, considering the fourth term in~\eqref{eqn:E_zeta_zeta_prime}, we have
\begin{align}
	\E\big[\vu_{\tau'+1}^\T\ve_{\tau'+1}(\wbb_{\tau}^\T \otimes \vu_{\tau+1}^\T) \bgl \vu_{1:T}\big]
	&= \E\big[\vu_{\tau'+1}^\T\ve_{\tau'+1}\wbb_{\tau}^\T \otimes \vu_{\tau+1}^\T \bgl \vu_{1:T}\big], \nn \\
	&= \vu_{\tau'+1}^\T \vC \vA^L \E\big[\vx_{\tau'-L+1}\wbb_{\tau}^\T \bgl \vu_{1:T}\big]\otimes \vu_{\tau+1}^\T. \label{eqn:E_wbar_kron_ut_ut_prime_et_prime}
\end{align}
To proceed, from~\eqref{eqn:bilinear sys} we have
\begin{align}
	\E\big[\vx_{t'}\wbb_{t}^\T \bgl \vu_{1:T}\big] &\,= \E\big[\big(\sum_{i=0}^{t'-1}\vA^{t'-i-1}\vB \vu_i+\sum_{i=0}^{t'-1}\vA^{t'-i-1}\vw_i\big)\wbb_{t}^\T \bgl \vu_{1:T}\big], \nn \\
	&\,= \sum_{i=0}^{t'-1}\vA^{t'-i-1}\E\big[\vw_i\wbb_{t}^\T \bgl \vu_{1:T}\big], \nn \\
	&\,= \sum_{i=0}^{t'-1}\vA^{t'-i-1} \big(\vdelta_w(t,i) \otimes \vSigma_w\big),
\end{align}
where we define $\vdelta_w(t,i) := \begin{bmatrix}\delta(t-i) & \delta (t-i-1) & \hdots & \delta(t-i-L+1) \end{bmatrix} \in \R^{1 \times L}$. Note that $\vdelta_w(t,i) = [1~0~\cdots~0]$ when $i = t$, and $\vdelta_w(t,i) = [0~\cdots~0~1]$ when $i = t-L+1$. Specifically, $\vdelta_w(t,i)$ has all zero entries except the $(t-i)$-th entry when $i \in [t-L+1,t]$, and it has all zero entries when $i \notin [t-L+1,t]$. Plugging this back into \eqref{eqn:E_wbar_kron_ut_ut_prime_et_prime}, we have
\begin{align}
	\E\big[\vu_{\tau'+1}^\T\ve_{\tau'+1}(\wbb_{\tau}^\T \otimes \vu_{\tau+1}^\T) \bgl \vu_{1:T}\big] 
	&=  \big(\vu_{\tau'+1}^\T \vC \vA^L \sum_{i=0}^{\tau'-L}\vA^{\tau'-L-i} \big(\vdelta_w(\tau,i) \otimes \vSigma_w\big)\big) \otimes \vu_{\tau+1}^\T. \label{eqn:E_etau_wbar_prime_fourth}
\end{align}
Using a similar line of reasoning, the fifth term in \eqref{eqn:E_zeta_zeta_prime} can also be calculated as,
\begin{align}
	\E\big[\vu_{\tau+1}^\T\ve_{\tau+1}(\wbb_{\tau'}^\T \otimes \vu_{\tau'+1}^\T) \bgl \vu_{1:T}\big] 
	&=  \big(\vu_{\tau+1}^\T \vC \vA^L \sum_{i=0}^{\tau-L}\vA^{\tau-L-i} \big(\vdelta_w(\tau',i) \otimes \vSigma_w\big)\big) \otimes \vu_{\tau'+1}^\T. \label{eqn:E_etau_wbar_prime_fifth}
\end{align}
\noindent $\bullet$ {\bf Remaining terms in \eqref{eqn:E_zeta_zeta_prime}:}
The last four terms in \eqref{eqn:E_ztau_ztau_prime} captures the conditional cross-correlation between the measurement noise process $\{z_{\tau+1}\}_{\tau = L}^{T-1}$ and the rest of noise/error processes $\{\wbb_{\tau}\}_{\tau = L}^{T-1}$,  $\{\ve_{\tau+1}\}_{\tau = L}^{T-1}$. Because the measurement noise process is independent of the rest of noise/error processes due to Assumption~\ref{assump:data/noise}, all of these terms are zero, 
\begin{equation}
\begin{aligned}\label{eqn:zero_terms}
	&\E\big[z_{\tau+1}(\wbb_{\tau'}^\T \otimes \vu_{\tau'+1}^\T) \bgl \vu_{1:T}\big] = \E\big[z_{\tau'+1}(\wbb_{\tau}^\T \otimes \vu_{\tau+1}^\T) \bgl \vu_{1:T}\big] =  0, \\
	& \E\big[z_{\tau+1}\vu_{\tau'+1}^\T\ve_{\tau'+1} \bgl \vu_{1:T}\big] = \E\big[z_{\tau'+1}\vu_{\tau+1}^\T\ve_{\tau+1} \bgl \vu_{1:T}\big]=  0.
\end{aligned}
\end{equation}
\noindent $\blacksquare$ {\bf Finalizing the proof:} Finally plugging~\eqref{eqn:E_wbar_wbar_prime},~\eqref{eqn:E_etau_etau_prime},~\eqref{eqn:E_ztau_ztau_prime},~\eqref{eqn:E_etau_wbar_prime_fourth},~\eqref{eqn:E_etau_wbar_prime_fifth}, and \eqref{eqn:zero_terms} into \eqref{eqn:E_zeta_zeta_prime}, and combining the result with~\eqref{eqn:E_zeta_tau_plus}, we obtain the following expression for the conditional auto-covariance of the of the effective noise process $\{\zeta_{\tau+1}\}_{\tau = L}^{T-1}$,
\begin{align}
	\Rcal_\zeta[\tau,\tau' \bgl \vu_{1:T}] &:=\E\big[\zeta_{\tau+1} \zeta_{\tau'+1} \bgl \vu_{1:T}\big] - \E\big[\zeta_{\tau+1}\bgl \vu_{1:T}\big]\E\big[\zeta_{\tau'+1}\bgl \vu_{1:T}\big] \nn \\ 
	& = \vek(\vF)^\T \big(\bar\vdelta_w(\tau,\tau') \otimes \vSigma_w \otimes \vu_{\tau+1}\vu_{\tau'+1}^\T\big)\vek(\vF) \nn \\
	& + \vu_{\tau+1}^\T\sum_{i=0}^{\tau-L}\sum_{j=0}^{\tau'-L}\vC\vA^{\tau-i}\vB \vu_i \vu_j^\T \vB^\T(\vA^{\tau'-j})^\T\vC^\T\vu_{\tau'+1} \nn \\
	&+ \vu_{\tau+1}^\T\sum_{i=0}^{\min\{\tau,\tau'\}-L}\vC\vA^{\tau-i} \vSigma_w (\vA^{\tau'-i})^\T\vC^\T\vu_{\tau'+1} + \sigma_z^2 \delta(\tau-\tau') \nn \\
	&+ \big[\big(\vu_{\tau'+1}^\T \vC \vA^L \sum_{i=0}^{\tau'-L}\vA^{\tau'-L-i} \big(\vdelta_w(\tau,i) \otimes \vSigma_w\big)\big) \otimes \vu_{\tau+1}^\T\big] \vek(\vF), \nn \\
	&+ \big[\big(\vu_{\tau+1}^\T \vC \vA^L \sum_{i=0}^{\tau-L}\vA^{\tau-L-i} \big(\vdelta_w(\tau',i) \otimes \vSigma_w\big)\big) \otimes \vu_{\tau'+1}^\T\big] \vek(\vF), \nn \\
	&- \vu_{\tau+1}^\T\sum_{i=0}^{\tau-L}\sum_{j=0}^{\tau'-L}\vC\vA^{\tau-i}\vB \vu_i \vu_j^\T \vB^\T(\vA^{\tau'-j})^\T\vC^\T\vu_{\tau'+1}, \nn \\
	& = \vek(\vF)^\T \big(\bar\vdelta_w(\tau,\tau') \otimes \vSigma_w \otimes \vu_{\tau+1}\vu_{\tau'+1}^\T\big)\vek(\vF) \nn \\
	&+ \vu_{\tau+1}^\T\sum_{i=0}^{\min\{\tau,\tau'\}-L}\vC\vA^{\tau-i} \vSigma_w (\vA^{\tau'-i})^\T\vC^\T\vu_{\tau'+1} + \sigma_z^2 \delta(\tau-\tau') \nn \\
	&+ \big[\big(\vu_{\tau'+1}^\T \vC \vA^L \sum_{i=0}^{\tau'-L}\vA^{\tau'-L-i} \big(\vdelta_w(\tau,i) \otimes \vSigma_w\big)\big) \otimes \vu_{\tau+1}^\T\big] \vek(\vF), \nn \\
	&+ \big[\big(\vu_{\tau+1}^\T \vC \vA^L \sum_{i=0}^{\tau-L}\vA^{\tau-L-i} \big(\vdelta_w(\tau',i) \otimes \vSigma_w\big)\big) \otimes \vu_{\tau'+1}^\T\big] \vek(\vF). \label{eqn:E_zeta_zeta_prime_final} 
\end{align}
This completes the proof.
\end{proof}

\subsection{Proof of Theorem~\ref{thrm:cond_covariance}}\label{sec:proof_of_cond_covariance}
\begin{proof}
    To begin, recalling the input-output relation $y_t = \ubb_{t-1}^\T \otimes \vu_t^\T \vek(\vG)  + \zeta_t$, where $\zeta_t =  \wbb_{t-1}^\T \otimes \vu_t^\T \vek(\vF)  + \vu_t^\T\ve_t + z_t$, the estimation error is alternately given by the following relation, 
\begin{align}
	\Delta\vG := \vek(\hat{\vG}) - \vek(\vG)
    &=  \big(\sum_{\tau = L}^{T-1} \ubb_{\tau}\ubb_{\tau}^\T \otimes \vu_{\tau+1}\vu_{\tau+1}^\T\big)^{-1}\sum_{\tau = L}^{T-1}(\ubb_{\tau} \otimes \vu_{\tau+1}) \zeta_{\tau+1}.  \label{eqn:Delta_G_expansion}
\end{align}
Hence, the conditional covariance of the estimation error is given by,
\begin{align}
	\vSigma[\Delta\vG \bgl \vu_{1:T}] &= \E[\Delta\vG \Delta\vG^\T\bgl \vu_{1:T}] - \E[\Delta\vG \bgl \vu_{1:T}] \E[\Delta\vG \bgl \vu_{1:T}]^\T, \nn \\
    &=\big(\sum_{\tau = L}^{T-1} \ubb_{\tau}\ubb_{\tau}^\T \otimes \vu_{\tau+1}\vu_{\tau+1}^\T\big)^{-1} \bigg(\sum_{\tau = L}^{T-1}\sum_{\tau' = L}^{T-1}\E\big[\zeta_{\tau+1} \zeta_{\tau'+1} \bgl \vu_{1:T}\big]\ubb_{\tau}\ubb_{\tau'}^\T \otimes \vu_{\tau+1}\vu_{\tau'+1}^\T\bigg)\nn \\
    & \qquad\qquad \qquad \qquad\qquad\qquad \qquad\qquad \qquad \qquad\qquad \quad
    \big(\sum_{\tau = L}^{T-1} \ubb_{\tau}\ubb_{\tau}^\T \otimes \vu_{\tau+1}\vu_{\tau+1}^\T\big)^{-1} \nn \\
    &- \big(\sum_{\tau = L}^{T-1} \ubb_{\tau}\ubb_{\tau}^\T \otimes \vu_{\tau+1}\vu_{\tau+1}^\T\big)^{-1} \bigg(\sum_{\tau = L}^{T-1}\sum_{\tau' = L}^{T-1}\E\big[\zeta_{\tau+1}\bgl \vu_{1:T}\big]\E\big[\zeta_{\tau'+1}\bgl \vu_{1:T}\big]\ubb_{\tau}\ubb_{\tau'}^\T \otimes \vu_{\tau+1}\vu_{\tau'+1}^\T\bigg) \nn \\
	& \qquad\qquad \qquad \qquad\qquad\qquad \qquad\qquad \qquad \qquad\qquad \quad\big(\sum_{\tau = L}^{T-1} \ubb_{\tau}\ubb_{\tau}^\T \otimes \vu_{\tau+1}\vu_{\tau+1}^\T\big)^{-1}, \nn \\
	&= \big(\sum_{\tau = L}^{T-1} \ubb_{\tau}\ubb_{\tau}^\T \otimes \vu_{\tau+1}\vu_{\tau+1}^\T\big)^{-1} \bigg(\sum_{\tau = L}^{T-1}\sum_{\tau' = L}^{T-1}\Rcal_\zeta[\tau,\tau'\bgl \vu_{1:T}]\ubb_{\tau}\ubb_{\tau'}^\T \otimes \vu_{\tau+1}\vu_{\tau'+1}^\T\bigg) \nn \\
    &\qquad\qquad \qquad \qquad\qquad\qquad \qquad\qquad \qquad \qquad\qquad \quad\big(\sum_{\tau = L}^{T-1} \ubb_{\tau}\ubb_{\tau}^\T \otimes \vu_{\tau+1}\vu_{\tau+1}^\T\big)^{-1}, \label{eqn:cond_covariance_v2}
\end{align}
where $\Rcal_\zeta[\tau,\tau'\bgl \vu_{1:T}]$ denotes the conditional auto-covariance function of the effective noise process $\{\zeta_{\tau+1}\}_{\tau = L}^{T-1}$, and it is as given in Lemma~\ref{lemma:auto_covariance}. Therefore, in order to upper bound the conditional covariance $\vSigma[\Delta\vG \bgl \vu_{1:T}]$, it suffices to upper bound the summation $\sum_{\tau = L}^{T-1}\sum_{\tau' = L}^{T-1}\Rcal_\zeta[\tau,\tau'\bgl \vu_{1:T}]\ubb_{\tau}\ubb_{\tau'}^\T \otimes \vu_{\tau+1}\vu_{\tau'+1}^\T$. To do that, we first use the fact that $\Rcal_\zeta[\tau,\tau'\bgl \vu_{1:T}] = \Rcal_\zeta[\tau',\tau\bgl \vu_{1:T}]$ due to symmetry, and the observation that for any two vectors $\vv_1$ and $\vv_2$ of appropriate dimensions we have $\vv_1\vv_2^\T + \vv_2 \vv_1^\T \preceq \vv_1\vv_1^\T + \vv_2 \vv_2^\T$, to obtain the following upper bound,
\begin{align}
\sum_{\tau = L}^{T-1}\sum_{\tau' = L}^{T-1}\Rcal_\zeta[\tau,\tau' \bgl \vu_{1:T}]\ubb_{\tau}\ubb_{\tau'}^\T \otimes \vu_{\tau+1}\vu_{\tau'+1}^\T 
&\preceq \sum_{\tau = L}^{T-1} \Rcal_\zeta[\tau,\tau \bgl \vu_{1:T}]\ubb_{\tau}\ubb_{\tau}^\T \otimes \vu_{\tau+1}\vu_{\tau+1}^\T  \nn \\
&+ \sum_{\tau = L}^{T-1}\sum_{\substack{\tau' = L \\ \tau' \neq \tau}}^{T-1}\big|\Rcal_\zeta[\tau,\tau' \bgl \vu_{1:T}]\big|\ubb_{\tau}\ubb_{\tau}^\T \otimes \vu_{\tau+1}\vu_{\tau+1}^\T. \label{eqn:auto_corr_summation}
\end{align}
The first term on the RHS of \eqref{eqn:auto_corr_summation}, can be upper bounded using Lemma~\eqref{lemma:auto_covariance} and \eqref{eqn:lemma_autocov_simplified_tau_equal_tau_prime} as follows,
\begin{align}
	\sum_{\tau = L}^{T-1} \Rcal_\zeta[\tau,\tau \bgl \vu_{1:T}]\ubb_{\tau}\ubb_{\tau}^\T \otimes \vu_{\tau+1}\vu_{\tau+1}^\T \preceq  \big(\big(\|\vSigma_w\|  \tf{\vF}^2 + \|\vC\vA^L\|^2 \|\vGamma_w^{\infty}\|  \big) \beta^2  + \sigma_z^2\big) \sum_{\tau = L}^{T-1} \ubb_{\tau}\ubb_{\tau}^\T \otimes \vu_{\tau+1}\vu_{\tau+1}^\T. \label{eqn:tau_equal_taup_final}
\end{align}
Similarly, we use Lemma~\ref{lemma:auto_covariance} to upper bound the second term on the RHS of \eqref{eqn:auto_corr_summation}. We have,
\begin{align}
	\sum_{\tau = L}^{T-1}\sum_{\substack{\tau' = L \\ \tau' \neq \tau}}^{T-1}&\big|\Rcal_\zeta[\tau,\tau' \bgl \vu_{1:T}]\big|\ubb_{\tau}\ubb_{\tau}^\T \otimes \vu_{\tau+1}\vu_{\tau+1}^\T \nn \\
	& \preceq \sum_{\tau = L}^{T-1}\sum_{\substack{\tau' = L \\ \tau' \neq \tau}}^{T-1} \bigg[ \big|\vek(\vF)^\T \big(\bar\vdelta_w(\tau,\tau') \otimes \vSigma_w \otimes \vu_{\tau+1}\vu_{\tau'+1}^\T\big)\vek(\vF)\big| \nn \\
    & + \big|\sum_{i=0}^{\min\{\tau,\tau'\}-L}  \vu_{\tau+1}^\T\vC\vA^{\tau-i} \vSigma_w (\vA^{\tau'-i})^\T\vC^\T\vu_{\tau'+1}\big| \nn \\
	& + \big|\big[\big(\vu_{\tau'+1}^\T \vC \vA^L \sum_{i=0}^{\tau'-L}\vA^{\tau'-L-i} \big(\vdelta_w(\tau,i) \otimes \vSigma_w\big)\big) \otimes \vu_{\tau+1}^\T\big] \vek(\vF) \big|\nn \\
	& + \big|\big[\big(\vu_{\tau+1}^\T \vC \vA^L \sum_{i=0}^{\tau-L}\vA^{\tau-L-i} \big(\vdelta_w(\tau',i) \otimes \vSigma_w\big)\big) \otimes \vu_{\tau'+1}^\T\big] \vek(\vF)\big|\bigg] \ubb_{\tau}\ubb_{\tau}^\T \otimes \vu_{\tau+1}\vu_{\tau+1}^\T. \label{eqn:cross_terms}
\end{align}
In the following, we will upper bound each summation in \eqref{eqn:cross_terms} separately to obtain an upper bound on the conditional  covariance $\vSigma[\Delta\vG \bgl \vu_{1:T}]$.

\noindent $\bullet$ {\bf First summation in \eqref{eqn:cross_terms}:} This summation can be upper bounded as follows,
\begin{align}
\sum_{\tau = L}^{T-1}\sum_{\substack{\tau' = L \\ \tau' \neq \tau}}^{T-1} & \big|\vek(\vF)^\T \big(\bar\vdelta_w(\tau,\tau') \otimes \vSigma_w \otimes \vu_{\tau+1}\vu_{\tau'+1}^\T\big)\vek(\vF)\big| \ubb_{\tau}\ubb_{\tau}^\T \otimes \vu_{\tau+1}\vu_{\tau+1}^\T \nn \\
& \preceqsym{i} \sum_{\tau = L}^{T-1}\sum_{\substack{\tau' = \max\{L, \tau - L+1\} \\ \tau' \neq \tau}}^{\min\{T-1,\tau+L-1\}}  \big(\|\bar\vdelta_w(\tau,\tau')\| \|\vSigma_w\| \tf{\vF}^2 \tn{\vu_{\tau+1}} \tn{\vu_{\tau'+1}} \ubb_{\tau} \big)\ubb_{\tau}^\T \otimes \vu_{\tau+1}\vu_{\tau+1}^\T, \nn \\
& \preceq \|\vSigma_w\| \tf{\vF}^2 \beta^2  \sum_{\tau = L}^{T-1}\bigg[\sum_{\substack{\tau' = \max\{L, \tau - L+1\} \\ \tau' \neq \tau}}^{\min\{T-1,\tau+L-1\}}  \|\bar\vdelta_w(\tau,\tau')\|\bigg] \ubb_{\tau}\ubb_{\tau}^\T \otimes \vu_{\tau+1}\vu_{\tau+1}^\T,  \nn \\
&\preceq 2\|\vSigma_w\|\tf{\vF}^2 \beta^2 L \sum_{\tau = L}^{T-1} \ubb_{\tau}\ubb_{\tau}^\T \otimes \vu_{\tau+1}\vu_{\tau+1}^\T, \label{eqn:first_summation_complete}
\end{align}
where we obtain (i) from the following observation: recalling the Toeplitz matrix $\bar\vdelta_w(\tau,\tau')$ from \eqref{eqn:delta_bar}, we can upper bound its spectral norm as follows:
\begin{enumerate}
	\item[($\va$)] When $\tau' = \tau$, the matrix $\bar\vdelta_w(\tau,\tau')$ becomes an identity matrix $\Iden_L$ with $\|\bar\vdelta_w(\tau,\tau')\|  = 1$
	\item[($\vb$)] When $\tau-L<\tau' < \tau$, the matrix $\bar\vdelta_w(\tau,\tau')$ becomes a lower shift matrix with $i,j$-th element given by $[\bar\vdelta_w(\tau,\tau')]_{i,j} = \delta(i-(j+\tau-\tau'))$. This further implies $\|\bar\vdelta_w(\tau,\tau')\|  \leq 1$.
	\item[($\vc$)] When $\tau < \tau'< \tau+L$, the matrix $\bar\vdelta_w(\tau,\tau')$ becomes an upper shift matrix with $i,j$-th element given by $[\bar\vdelta_w(\tau,\tau')]_{i,j} = \delta((i+\tau'-\tau)-j)$. This further implies $\|\bar\vdelta_w(\tau,\tau')\|  \leq 1$.
	\item[($\vd$)] When $|\tau -\tau'| \geq L$, we have $\bar\vdelta_w(\tau,\tau') = \bv{0}$.
\end{enumerate}

\noindent $\bullet$ {\bf Second summation in \eqref{eqn:cross_terms}:} Using Assumption~\ref{assump:stability}, this term can be upper bounded as follows,
\begin{align}
    \sum_{\tau = L}^{T-1}\sum_{\substack{\tau' = L \\ \tau' \neq \tau}}^{T-1} \big|\sum_{i=0}^{\min\{\tau,\tau'\}-L}  &\vu_{\tau+1}^\T\vC\vA^{\tau-i} \vSigma_w (\vA^{\tau'-i})^\T\vC^\T\vu_{\tau'+1}\big| \ubb_{\tau}\ubb_{\tau}^\T \otimes \vu_{\tau+1}\vu_{\tau+1}^\T \nn \\
	&\preceqsym{a} \sum_{\tau = L}^{T-1}\sum_{\substack{\tau' = L \\ \tau' \neq \tau}}^{T-1} \big(\beta^2 \norm{\vC\vA^L}^2 \norm{\vGamma_w^\infty}\norm{\vA^{|\tau-\tau'|}}\big)  \ubb_{\tau}\ubb_{\tau}^\T \otimes \vu_{\tau+1}\vu_{\tau+1}^\T \nn \\
	&\preceq \beta^2 \norm{\vC\vA^L}^2 \norm{\vGamma_w^\infty} \phi(\vA,\rho) \sum_{\tau = L}^{T-1}\sum_{\substack{\tau' = L \\ \tau' \neq \tau}}^{T-1} \rho^{|\tau-\tau'|}  \ubb_{\tau}\ubb_{\tau}^\T \otimes \vu_{\tau+1}\vu_{\tau+1}^\T \nn \\
	& \preceqsym{b} 2\beta^2 \norm{\vC\vA^L}^2 \norm{\vGamma_w^\infty} \phi(\vA,\rho) \sum_{\tau = L}^{T-1} \ubb_{\tau}\ubb_{\tau}^\T \otimes \vu_{\tau+1}\vu_{\tau+1}^\T \sum_{i = 1}^{\max\{\substack{T-\tau-1, \\ \tau-L}\}} \rho ^i, \nn \\
	& \preceq 2\beta^2 \norm{\vC\vA^L}^2 \norm{\vGamma_w^\infty} \frac{\rho \phi(\vA,\rho)}{1-\rho} \sum_{\tau = L}^{T-1} \ubb_{\tau}\ubb_{\tau}^\T \otimes \vu_{\tau+1}\vu_{\tau+1}^\T, \label{eqn:second_summation_complete}
\end{align}
where we get (a) from the following observation: when $\tau > \tau'$, we have
\begin{align*}
    \big|\sum_{i=0}^{\min\{\tau,\tau'\}-L}  \vu_{\tau+1}^\T\vC\vA^{\tau-i} \vSigma_w (\vA^{\tau'-i})^\T\vC^\T\vu_{\tau'+1}\big| &= \big|\vu_{\tau+1}^\T\vC \vA^{\tau-\tau'}\big(\sum_{i=0}^{\tau'-L}\vA^{\tau'-i} \vSigma_w (\vA^{\tau'-i})^\T\big)\vC^\T\vu_{\tau'+1}\big|, \\
    & \leq \big|\vu_{\tau+1}^\T\vC\vA^L \vA^{\tau-\tau'}\vGamma_w^{\infty}(\vA^L)^\T\vC^\T\vu_{\tau'+1} \big|, \\
    & \leq \beta^2 \norm{\vC\vA^L}^2 \norm{\vGamma_w^\infty}\norm{\vA^{\tau-\tau'}}.
\end{align*}
Similarly, when $\tau < \tau'$, we have $\big|\sum_{i=0}^{\min\{\tau,\tau'\}-L}  \vu_{\tau+1}^\T\vC\vA^{\tau-i} \vSigma_w (\vA^{\tau'-i})^\T\vC^\T\vu_{\tau'+1}\big| \leq \beta^2 \norm{\vC\vA^L}^2 \norm{\vGamma_w^\infty}\norm{\vA^{\tau'-\tau}}$, using the same line of reasoning. Combining the two cases gives (a). Moreover, (b) is obtained by writing up the summation $\sum_{\tau = L}^{T-1}\sum_{\substack{\tau' = L \\ \tau' \neq \tau}}^{T-1} \rho^{|\tau-\tau'|}$ for each value of $\tau \in [L, T-1]$.

\noindent $\bullet$ {\bf Third summation in \eqref{eqn:cross_terms}:} Using Assumptions~\ref{assump:stability} and \ref{assump:data/noise}, this term can be upper bounded as follows,
\begin{align}
    \sum_{\tau = L}^{T-1}\sum_{\substack{\tau' = L \\ \tau' \neq \tau}}^{T-1} \big|\big[\big(\vu_{\tau'+1}^\T \vC \vA^L &\sum_{i=0}^{\tau'-L}\vA^{\tau'-L-i} \big(\vdelta_w(\tau,i) \otimes \vSigma_w\big)\big) \otimes \vu_{\tau+1}^\T\big] \vek(\vF) \big| \ubb_{\tau}\ubb_{\tau}^\T \otimes \vu_{\tau+1}\vu_{\tau+1}^\T \nn \\
    & = \sum_{\tau = L}^{T-1}\sum_{\substack{\tau' = L \\ \tau' \neq \tau}}^{T-1}  \big|\big(\vu_{\tau'+1}^\T \vC \vA^L \E\big[\vx_{\tau'-L+1}\wbb_{\tau}^\T \bgl \vu_{1:T}\big] \otimes \vu_{\tau+1}^\T\big) \vek(\vF)\big| \ubb_{\tau}\ubb_{\tau}^\T \otimes \vu_{\tau+1}\vu_{\tau+1}^\T, \nn \\
	& \preceq \beta^2 \tf{\vF} \sum_{\tau = L}^{T-1} \bigg[\sum_{\substack{\tau' = L \\ \tau' \neq \tau}}^{T-1} \|\vC \vA^L \E\big[\vx_{\tau'-L+1}\wbb_{\tau}^\T \bgl \vu_{1:T}\big]\|\bigg]  \ubb_{\tau}\ubb_{\tau}^\T \otimes \vu_{\tau+1}\vu_{\tau+1}^\T, \nn \\
	&\preceqsym{i} \beta^2 \tf{\vF}^2 \|\vSigma_w\| (L + \frac{\rho}{1-\rho}) \phi(\vA,\rho)\rho^L  \sum_{\tau = L}^{T-1} \ubb_{\tau}\ubb_{\tau}^\T \otimes \vu_{\tau+1}\vu_{\tau+1}^\T, \label{eqn:third_summation_complete}
\end{align}
where we obtain (i) by upper bounding $\|\vC \vA^L \E\big[\vx_{\tau'-L+1}\wbb_{\tau}^\T \bgl \vu_{1:T}\big]\|$ as follows:
\begin{enumerate}
\item [$(\va)$] When $\tau' \leq \tau$, we have $\E\big[\vx_{\tau'-L+1}\wbb_{\tau}^\T \bgl \vu_{1:T}\big]  =  \bv{0}$. Hence, $\|\vC \vA^L \E\big[\vx_{\tau'-L+1}\wbb_{\tau}^\T \bgl \vu_{1:T}\big]\| = 0$
\item[$(\vb)$] When $\tau< \tau' < \tau+L$, we have
\begin{align}
    &\E\big[\vx_{\tau'-L+1}\wbb_{\tau}^\T \bgl \vu_{1:T}\big] =  \sum_{i=\tau-L+1}^{\tau'-L}\vA^{\tau'-L-i} \big(\vdelta_w(\tau,i) \otimes \vSigma_w\big) =  \begin{bmatrix} \bv{0} & \hdots & \vSigma_w  & \vA \vSigma_w & \hdots & \vA^{\tau'-\tau-1} \vSigma_w\end{bmatrix}. \nn\\
    & \|\vC \vA^L \E\big[\vx_{\tau'-L+1}\wbb_{\tau}^\T \bgl \vu_{1:T}\big]\| = \|\begin{bmatrix} \bv{0} & \hdots & \vC (\vA^L\vSigma_w) & \vC\vA (\vA^L\vSigma_w) & \hdots & \vC\vA^{\tau'-\tau-1} (\vA^L\vSigma_w)\end{bmatrix}\|, \nn\\
    & \qquad\qquad\qquad\qquad\qquad\qquad\;\, \leq \tf{\begin{bmatrix} \bv{0} & \hdots & \vC & \vC\vA & \hdots & \vC\vA^{\tau'-\tau-1} \end{bmatrix}} \|\vA^L \vSigma_w\|,\nn \\
     & \qquad\qquad\qquad\qquad\qquad\qquad\;\, \leq \phi(\vA,\rho) \rho^L \norm{\vSigma_w} \tf{\vF}.
\end{align}
\item [$(\vc)$] When $\tau' = \tau+L$, we have 
\begin{align}
 &\E\big[\vx_{\tau'-L+1}\wbb_{\tau}^\T \bgl \vu_{1:T}\big] = \sum_{i=\tau-L+1}^{\tau}\vA^{\tau-i} \big(\vdelta_w(\tau,i) \otimes \vSigma_w\big) =  \begin{bmatrix} \vSigma_w & \vA \vSigma_w & \hdots & \vA^{L-1} \vSigma_w\end{bmatrix}. \nn \\
 & \|\vC \vA^L \E\big[\vx_{\tau'-L+1}\wbb_{\tau}^\T \bgl \vu_{1:T}\big]\| = \|\begin{bmatrix} \vC(\vA^L\vSigma_w) & \vC\vA (\vA^L\vSigma_w) & \hdots & \vC\vA^{L-1} (\vA^L\vSigma_w)\end{bmatrix}\|, \nn \\
    & \qquad\qquad\qquad\qquad\qquad\qquad\;\, \leq \tf{\begin{bmatrix} \vC & \vC\vA & \hdots & \vC\vA^{L-1} \end{bmatrix}} \|\vA^L \vSigma_w\|, \nn \\
     & \qquad\qquad\qquad\qquad\qquad\qquad\;\, \leq \phi(\vA,\rho) \rho^L \norm{\vSigma_w} \tf{\vF}.
\end{align}
\item[$(\vd)$] When $\tau' > \tau+ L$, we have 
\begin{align}
    &\E\big[\vx_{\tau'-L+1}\wbb_{\tau}^\T \bgl \vu_{1:T}\big] =  \sum_{i=\tau-L+1}^{\tau}\vA^{\tau'-L-i} \big(\vdelta_w(\tau,i) \otimes \vSigma_w\big) =  \vA^{\tau'-\tau-L}\begin{bmatrix} \vSigma_w & \vA \vSigma_w & \hdots & \vA^{L-1} \vSigma_w\end{bmatrix}. \nn \\
     & \|\vC \vA^L \E\big[\vx_{\tau'-L+1}\wbb_{\tau}^\T \bgl \vu_{1:T}\big]\| = \|\begin{bmatrix} \vC(\vA^{\tau'-\tau}\vSigma_w) & \vC\vA (\vA^{\tau'-\tau}\vSigma_w) & \hdots & \vC\vA^{L-1} (\vA^{\tau'-\tau}\vSigma_w)\end{bmatrix}\|, \nn \\
    & \qquad\qquad\qquad\qquad\qquad\qquad\;\, \leq \tf{\begin{bmatrix} \vC & \vC\vA & \hdots & \vC\vA^{L-1} \end{bmatrix}} \|\vA^{\tau'-\tau} \vSigma_w\|, \nn \\
     & \qquad\qquad\qquad\qquad\qquad\qquad\;\, \leq \phi(\vA,\rho) \rho^L \norm{\vSigma_w} \tf{\vF} (\rho^{\tau'-\tau-L}).
\end{align}
\end{enumerate}	
Hence, $\sum_{\substack{\tau' = L \\ \tau' \neq \tau}}^{T-1} \|\vC \vA^L \E\big[\vx_{\tau'-L+1}\wbb_{\tau}^\T \bgl \vu_{1:T}\big]\| \leq \big(L + \rho + \rho^2 + \dots\big)\phi(\vA,\rho) \rho^L \norm{\vSigma_w} \tf{\vF}$, which gives (i) above. 

\noindent $\bullet$ {\bf Fourth summation in \eqref{eqn:cross_terms}:}
Similar to the previous summation,  using Assumptions~\ref{assump:stability} and \ref{assump:data/noise}, this term can also be upper bounded as,
\begin{align}
    \sum_{\tau = L}^{T-1}\sum_{\substack{\tau' = L \\ \tau' \neq \tau}}^{T-1} \big|\big[\big(\vu_{\tau+1}^\T \vC \vA^L &\sum_{i=0}^{\tau-L}\vA^{\tau-L-i} \big(\vdelta_w(\tau',i) \otimes \vSigma_w\big)\big) \otimes \vu_{\tau'+1}^\T\big] \vek(\vF) \big| \ubb_{\tau}\ubb_{\tau}^\T \otimes \vu_{\tau+1}\vu_{\tau+1}^\T \nn \\
	&\preceq \beta^2 \tf{\vF}^2 \|\vSigma_w\| (L + \frac{\rho}{1-\rho}) \phi(\vA,\rho)\rho^L  \sum_{\tau = L}^{T-1} \ubb_{\tau}\ubb_{\tau}^\T \otimes \vu_{\tau+1}\vu_{\tau+1}^\T. \label{eqn:fourth_summation_complete}
\end{align}

\noindent $\blacksquare$ {\bf Finalizing the proof:} Finally plugging~\eqref{eqn:first_summation_complete},~\eqref{eqn:second_summation_complete},~\eqref{eqn:third_summation_complete}, and \eqref{eqn:fourth_summation_complete} into \eqref{eqn:cross_terms}, and combining the result with~\eqref{eqn:auto_corr_summation} and \eqref{eqn:tau_equal_taup_final}, we obtain the following upper bound, 
\begin{align}
    \sum_{\tau = L}^{T-1}\sum_{\tau' = L}^{T-1} &\Rcal_\zeta[\tau,\tau' \bgl \vu_{1:T}]\ubb_{\tau}\ubb_{\tau'}^\T \otimes \vu_{\tau+1}\vu_{\tau'+1}^\T \nn \\ 
    &\preceq \bigg[\sigma_z^2 + \beta^2\big(\|\vSigma_w\| \tf{\vF}^2 + \|\vC\vA^L\|^2 \|\vGamma_w^{\infty}\|\big)  + 2\|\vSigma_w\|\tf{\vF}^2 \beta^2 L \nn \\
    &+  2\beta^2 \norm{\vC\vA^L}^2 \norm{\vGamma_w^\infty} \frac{\rho \phi(\vA,\rho)}{1-\rho} + 2 \beta^2 \tf{\vF}^2 \|\vSigma_w\| (L + \frac{\rho}{1-\rho}) \phi(\vA,\rho)\rho^L \bigg]  \sum_{\tau = L}^{T-1} \ubb_{\tau}\ubb_{\tau}^\T \otimes \vu_{\tau+1}\vu_{\tau+1}^\T,  \nn \\
    & =  \bigg[\sigma_z^2 + \|\vSigma_w\|\tf{\vF}^2\beta^2\big(1 + 2L +2 (L + \frac{\rho}{1-\rho}) \phi(\vA,\rho)\rho^L \big)  \nn \\
    & + \norm{\vGamma_w^\infty} \norm{\vC\vA^L}^2\beta^2\big(1 + 2 \frac{\rho \phi(\vA,\rho)}{1-\rho}\big)  \bigg]  \sum_{\tau = L}^{T-1} \ubb_{\tau}\ubb_{\tau}^\T \otimes \vu_{\tau+1}\vu_{\tau+1}^\T, \nn \\
    & \preceq \bigg[\sigma_z^2 + 3\|\vSigma_w\|\tf{\vF}^2\beta^2 L \big(1 + \frac{\phi(\vA,\rho)\rho^L}{1-\rho}\ \big)  + 2 \norm{\vGamma_w^\infty} \norm{\vC\vA^L}^2\beta^2 \frac{\phi(\vA,\rho)}{1-\rho}  \bigg]  \sum_{\tau = L}^{T-1} \ubb_{\tau}\ubb_{\tau}^\T \otimes \vu_{\tau+1}\vu_{\tau+1}^\T,
\end{align}
which is then combined with~\eqref{eqn:cond_covariance_v2} to get the following upper bound on the conditional covariance of the estimation error,
\begin{align}
    &\vSigma[\Delta\vG \bgl \vu_{1:T}]\nn \\
    & \qquad = \big(\sum_{\tau = L}^{T-1} \ubb_{\tau}\ubb_{\tau}^\T \otimes \vu_{\tau+1}\vu_{\tau+1}^\T\big)^{-1} \bigg(\sum_{\tau = L}^{T-1}\sum_{\tau' = L}^{T-1}\Rcal_\zeta[\tau,\tau'\bgl \vu_{1:T}]\ubb_{\tau}\ubb_{\tau'}^\T \otimes \vu_{\tau+1}\vu_{\tau'+1}^\T\bigg)\big(\sum_{\tau = L}^{T-1} \ubb_{\tau}\ubb_{\tau}^\T \otimes \vu_{\tau+1}\vu_{\tau+1}^\T\big)^{-1}, \nn \\
    & \qquad\preceq \bigg[\sigma_z^2 + 3\|\vSigma_w\|\tf{\vF}^2\beta^2 L \big(1 + \frac{\phi(\vA,\rho)\rho^L}{1-\rho}\ \big)  + 2 \norm{\vGamma_w^\infty} \norm{\vC\vA^L}^2\beta^2 \frac{\phi(\vA,\rho)}{1-\rho}  \bigg] \big( \sum_{\tau = L}^{T-1} \ubb_{\tau}\ubb_{\tau}^\T \otimes \vu_{\tau+1}\vu_{\tau+1}^\T \big)^{-1}.
\end{align}
This completes the proof.
\end{proof}

\subsection{Proof of Lemma~\ref{lemma:cond_mean_norm}}\label{sec:proof_of_cond_mean_norm}
\begin{proof}
    To begin, recall that the estimation error is given by $\Delta\vG = (\vUtil^\T \vUtil)^{-1}\vUtil^\T (\vWtil\vek(\vF) + \vetil + \vz)$, where we define $\vWtil \in \R^{(T-L) \times npL}, \vetil \in \R^{(T-L)}$ and $\vz \in \R^{(T-L)}$ as follows,
\begin{equation}
\begin{aligned}\label{eqn:W_E_H_construction}
 	\vWtil &:= \begin{bmatrix} \wbb_{L} \otimes \vu_{L+1} & \wbb_{L+1} \otimes \vu_{L+2} & \cdots & \wbb_{T-1} \otimes \vu_{T} \end{bmatrix}^\T, \\
 	\vetil &:= \begin{bmatrix} \vu_{L+1}^\T\ve_{L+1} & \vu_{L+2}^\T\ve_{L+2} & \cdots &  \vu_{T}^\T\ve_{T} \end{bmatrix}^\T, \quad
 	\vz := \begin{bmatrix} z_{L+1} & z_{L+2} & \cdots & z_{T} \end{bmatrix}^\T.
\end{aligned}
\end{equation}
Hence, using Assumption~\ref{assump:data/noise}, the conditionally expected estimation error is given by 
    \begin{align}
        \E[\Delta\vG \bgl \vu_{1:T}] \eqsym{a} (\vUtil^\T \vUtil)^{-1} \vUtil^\T \E[\vetil \bgl \vu_{1:T}], \label{eqn:cond_mean_v1}
    \end{align}
    where we get (a) from the observation that $\E[\vWtil \bgl \vu_{1:T}] = 0$, and $\E[\vz \bgl \vu_{1:T}] = 0$ due to Assumption~\ref{assump:data/noise}. To proceed, let $\vUtil  =  \vQtil \tilde{\vSigma} \vRtil^\T $ be the singular value decomposition, where $\vQtil \in \R^{T-L \times p^2L}$. Then, from~\eqref{eqn:cond_mean_v1}, we have
    \begin{align}
        \tn{\E[\Delta\vG \bgl \vu_{1:T}]} &= \sqrt{\E[\Delta\vG \bgl \vu_{1:T}]^\T \E[\Delta\vG \bgl \vu_{1:T}]}, \nn \\
        &= \bigg[ \E[\vetil \bgl \vu_{1:T}]^\T \vUtil(\vUtil^\T \vUtil)^{-2} \vUtil^\T \E[\vetil \bgl \vu_{1:T}]\bigg]^{1/2}, \nn \\
        &= \bigg[ \E[\vetil \bgl \vu_{1:T}]^\T \vQtil \tilde{\vSigma} \vRtil^\T (\vRtil \tilde{\vSigma} \vQtil^\T \vQtil \tilde{\vSigma} \vRtil^\T )^{-2} \vRtil \tilde{\vSigma} \vQtil^\T \E[\vetil \bgl \vu_{1:T}]\bigg]^{1/2}, \nn \\
        &= \bigg[ \E[\vetil \bgl \vu_{1:T}]^\T \vQtil \tilde{\vSigma}^{-2} \vQtil^\T \E[\vetil \bgl \vu_{1:T}]\bigg]^{1/2}, \nn \\
        & \leqsym{a} \bigg[ \E[\vetil \bgl \vu_{1:T}]^\T \vQtil\vQtil^\T \E[\vetil \bgl \vu_{1:T}]\big/ \lambda_{\min}(\vUtil^\T \vUtil)\bigg]^{1/2}, \nn \\
        & \leqsym{b} \frac{\tn{\E[\vetil \bgl \vu_{1:T}]}}{\sqrt{\lambda_{\min}(\vUtil^\T \vUtil)}}, \label{eqn:cond_mean_v2}
    \end{align}
    where we obtain (a) from the fact that $\norm{\tilde{\vSigma}^{-2}}  = 1/\lambda_{\min} (\tilde{\vSigma}^{2}) = 1/ \lambda_{\min}(\vUtil^\T \vUtil)$, and (b) from the observation that $\vQtil\vQtil^\T = \vUtil(\vUtil^\T \vUtil)^{-1} \vUtil^\T$ is an orthogonal projection operator onto the column-space of $\vQtil$, hence, $\norm{\vQtil\vQtil^\T} \leq 1$. To proceed, consider the $i$-th element of $\vetil$ from~\eqref{eqn:W_E_H_construction}, where $i \in [1, T-L]$. We have
    \begin{align}
     |\E[[\vetil]_i \bgl \vu_{1:T}]| &= |\vu_{L+i}^\T \E[\ve_{L+i} \bgl \vu_{1:T}]|, \nn \\
     &\leq \tn{\vu_{L+i}} \tn{\vC \vA^L \E[\vx_{i} \bgl \vu_{1:T}]}, \nn \\
     &\leq \beta \tn{\vC \vA^L \sum_{k=0}^{i-1}\vA^{k}\vB\vu_{i-k-1}}, \nn \\
     &\leq \beta \norm{\vC} \sum_{k=0}^{i-1}\norm{\vA^{L+k}}\|\vB\| \tn{\vu_{i-k-1}}, \nn\\
     &\leq \beta^2 \norm{\vB} \norm{\vC} \phi(\vA,\rho)\rho^L \sum_{k=0}^{i-1}\rho^k, \nn\\
     &\leq \beta^2 \norm{\vB} \norm{\vC} \phi(\vA,\rho) \frac{\rho^L}{1 - \rho}.\label{eqn:etil_ith_bound}
    \end{align}
    Combining this with \eqref{eqn:cond_mean_v2}, we obtain the following upper bound on the Frobenius norm of conditionally expected estimation error,
    \begin{align}
        \tf{\E[\vGhat - \vG \bgl \vu_{1:T}]} =  \tn{\E[\Delta\vG \bgl \vu_{1:T}]} &\leq \frac{\mysqrt{\sum_{i=1}^{T-L}|\E[[\vetil]_i \bgl \vu_{1:T}]|^2}}{\sqrt{\lambda_{\min}(\vUtil^\T \vUtil)}}, \nn \\
        &\leq \frac{\beta^2 \norm{\vB} \norm{\vC} \phi(\vA,\rho) \frac{\rho^L}{1 - \rho} \sqrt{T-L}}{\sqrt{\lambda_{\min}(\vUtil^\T \vUtil)}}. \label{eqn:cond_mean_final}
    \end{align}
    Similarly, we can also upper bound the weighted norm of the conditionally expected estimation error as follows,
    \begin{align}
	\norm{\E[\Delta\vG \bgl \vu_{1:T}]}_{\vUtil^T\vUtil} &= \sqrt{\E[\Delta\vG \bgl \vu_{1:T}]^\T \vUtil^T\vUtil\E[\Delta\vG \bgl \vu_{1:T}]}, \nn \\
	& = \bigg[\E[\vetil\bgl \vu_{1:T}]^\T \vUtil (\vUtil^\T \vUtil)^{-1}\vUtil^\T\E[\vetil\bgl \vu_{1:T}]\bigg]^{1/2}, \nn \\
    &\eqsym{a} \bigg[ \E[\vetil \bgl \vu_{1:T}]^\T \vQtil \tilde{\vSigma} \vRtil^\T (\vRtil \tilde{\vSigma} \vQtil^\T \vQtil \tilde{\vSigma} \vRtil^\T )^{-1} \vRtil \tilde{\vSigma} \vQtil^\T \E[\vetil \bgl \vu_{1:T}]\bigg]^{1/2}, \nn \\
    &= \bigg[ \E[\vetil \bgl \vu_{1:T}]^\T \vQtil  \vQtil^\T \E[\vetil \bgl \vu_{1:T}]\bigg]^{1/2}, \nn \\
    & \leqsym{b} \tn{\E[\vetil \bgl \vu_{1:T}]}, \nn \\
    & \leqsym{c} \beta^2 \norm{\vB} \norm{\vC} \phi(\vA,\rho) \frac{\rho^L}{1 - \rho} \sqrt{T-L},\label{eqn:mean_error_ellipsoid}
\end{align}
where we get (a) from the singular value decomposition of $\vUtil$, (b) from the observation that $\vQtil\vQtil^\T$ is an orthogonal projection operator onto the column-space of $\vQtil$, hence, $\norm{\vQtil\vQtil^\T} \leq 1$, and (c) from~\eqref{eqn:cond_mean_final} above. This completes the proof.
\end{proof}

\subsection{Verifying Assumption~\ref{assump:input}$(b)$}\label{sec:moment_bound_verification}
Suppose $\{\vu_t\}_{t=0}^T \distas \Ncal(0, \Iden_{p})$. Then, for any $\vv \in \Scal^{p^2L-1}$, we have
\begin{align}
	\E[(\vv^\T\vutil_t)^4] &= \E[((\ubb_t^\T \otimes \vu_{t+1}^\T) \vv)^4] = \E[(\vu_{t+1}^\T \vV \ubb_t)^4] = \E\big[\big(\sum_{k=0}^{L-1} \vu_{t+1}^\T \vV_k \vu_{t-k}\big)^4\big], \nn \\
	&= \sum_{k=0}^{L-1} \E[(\vu_{t+1}^\T \vV_k \vu_{t-k})^4] + 3 \sum_{k=0}^{L-1}\sum_{\substack{k' = 0 \\ k' \neq k}}^{L-1}\E[(\vu_{t+1}^\T \vV_k \vu_{t-k})^2(\vu_{t+1}^\T \vV_{k'} \vu_{t-k'})^2],  \label{eqn:moment_bound_split}
\end{align}
where we obtain the first line by setting
\begin{align}
	\vV :=  \mat(\vv) = 
	\begin{bmatrix}
		\vV_0 & \vV_1 & \cdots &\vV_{L-1}  
	\end{bmatrix} \in \R^{p \times pL}
	\quad \text{such that each} \quad \vV_k \in \R^{p \times p}.
\end{align}
In the following, we will upper bound each expectation in \eqref{eqn:moment_bound_split} separately. Observing that
\begin{align}
	\vu_{t+1}^\T \vV_k \vu_{t-k} \bgl \vu_{t+1} \sim \Ncal(0, \tn{\vV_k^\T \vu_{t+1}}^2),
\end{align}
we have
\begin{align}
	\E[(\vu_{t+1}^\T \vV_k \vu_{t-k})^4] = \E \big[\E[(\vu_{t+1}^\T \vV_k \vu_{t-k})^4 \bgl \vu_{t+1}]\big] &= 3\E[\tn{\vV_k^\T \vu_{t+1}}^4], \nn \\
	&= 3\E[ (\vu_{t+1}^\T \vV_k\vV_k^\T \vu_{t+1})^2], \nn \\
	& \eqsym{a} \tr(\vV_k\vV_k^\T)^2 + 2 \tr ((\vV_k\vV_k^\T)^2), \nn \\
	& \leqsym{b}  \tf{\vV_k}^4 + 2 \tr (\vV_k\vV_k^\T)^2, \nn \\
	& \leq 3 \tf{\vV_k}^4, \label{eqn:moment_firt_term}
\end{align}
where we obtain (a) from \cite{don1979expectation}, and (b) from the fact that for any positive semi-definite matrix $\vX$, we have $\tr(\vX^2) \leq \tr(\vX)^2$. Similarly, we also have
\begin{align}
	\E[(\vu_{t+1}^\T \vV_k \vu_{t-k})^2(\vu_{t+1}^\T \vV_{k'} \vu_{t-k'})^2] &= \E\big[\E[(\vu_{t+1}^\T \vV_k \vu_{t-k})^2(\vu_{t+1}^\T \vV_{k'} \vu_{t-k'})^2 \bgl \vu_{t+1}]\big], \nn \\
	&= \E\big[\E[(\vu_{t+1}^\T \vV_k \vu_{t-k})^2 \bgl \vu_{t+1}] \E[(\vu_{t+1}^\T \vV_{k'} \vu_{t-k'})^2 \bgl \vu_{t+1}]\big], \nn \\
	& = \E[\tn{\vV_k^\T \vu_{t+1}}^2\tn{\vV_{k'}^\T \vu_{t+1}}^2], \nn \\
	& = \E[(\vu_{t+1}^\T \vV_k\vV_k^\T \vu_{t+1})(\vu_{t+1}^\T \vV_{k'}\vV_{k'}^\T \vu_{t+1})], \nn \\
	& \eqsym{a} \tr(\vV_k\vV_k^\T) \tr(\vV_{k'}\vV_{k'}^\T) + 2 \tr (\vV_k\vV_k^\T \vV_{k'}\vV_{k'}^\T), \nn \\
	&\leqsym{b} \tf{\vV_k}^2\tf{\vV_{k'}}^2	+ 2 \sqrt{\tr((\vV_k\vV_k^\T)^2) \tr((\vV_{k'}\vV_{k'}^\T)^2)}, \nn \\
	& \leqsym{c} 3\tf{\vV_k}^2\tf{\vV_{k'}}^2, \label{eqn:moment_second_term}
\end{align}
where we obtain (a) from \cite{don1979expectation}, (b) from Cauchy-Schwarz inequality applied to the trace of the product of two positive semi-definite matrices, and (c) from the fact that for any positive semi-definite matrix $\vX$, we have $\tr(\vX^2) \leq \tr(\vX)^2$. Finally, combining \eqref{eqn:moment_firt_term} and \eqref{eqn:moment_second_term} into \eqref{eqn:moment_bound_split}, we obtain out final result,
\begin{align}
	\E[(\vv^\T\vutil_t)^4] &= 3 \sum_{k=0}^{L-1} \tf{\vV_k}^4 + 9 \sum_{k=0}^{L-1}\sum_{\substack{k' = 0 \\ k' \neq k}}^{L-1} \tf{\vV_k}^2\tf{\vV_{k'}}^2, \nn \\
	& \leq 9 \big(\sum_{k=0}^{L-1} \tf{\vV_k}^2 \big)^2 \leq 9.
\end{align}
Hence, Assumption\ref{assump:input}($\vb$) holds for Gaussian inputs with $m_4 = 9$.

\end{document}

%% file: notation.tex
\usepackage{hyperref,graphicx,amsmath,amsfonts,amssymb,bm,url,breakurl,epsfig,epsf,color,fullpage,MnSymbol,mathbbol,fmtcount,semtrans,cite,caption,subcaption,multirow,comment}
\usepackage{stackengine}[2013-10-15]
\usepackage{scalerel}
\newcommand\mysqrt[2][0pt]{\stretchrel{\sqrt{}}{\addstackgap%
		[#1]{$\displaystyle\overline{#2}$}}}

\usepackage{amssymb,adjustbox}
\usepackage[utf8]{inputenc} 
\usepackage[T1]{fontenc}    
\usepackage{url}            
\usepackage{booktabs}       
\usepackage{amsfonts}       
\usepackage{nicefrac}       
\usepackage{microtype}      
\usepackage{balance}
\makeatletter
\providecommand*{\boxast}{%
	\mathbin{
		\mathpalette\@boxit{*}%
	}%
}
\newcommand*{\@boxit}[2]{%
	\sbox0{$\m@th#1\Box$}%
	\ifx#1\displaystyle \ht0=\dimexpr\ht0+.05ex\relax \fi
	\ifx#1\textstyle \ht0=\dimexpr\ht0+.05ex\relax \fi
	\ifx#1\scriptstyle \ht0=\dimexpr\ht0+.04ex\relax \fi
	\ifx#1\scriptscriptstyle \ht0=\dimexpr\ht0+.065ex\relax \fi
	\sbox2{$#1\vcenter{}$}
	\rlap{%
		\hbox to \wd0{%
			\hfill
			\raisebox{%
				\dimexpr.5\dimexpr\ht0+\dp0\relax-\ht2\relax
			}{$\m@th#1#2$}%
			\hfill
		}%
	}%
	\Box
}
\makeatother

\makeatletter
\def\BState{\State\hskip-\ALG@thistlm}
\makeatother

\usepackage{mathtools}

\usepackage{titlesec}

\usepackage{tikz}
\usepackage{pgfplots}
\usetikzlibrary{pgfplots.groupplots}

\titleformat{\paragraph}
{\normalfont\normalsize\bfseries}{\theparagraph}{1em}{}
\titlespacing*{\paragraph}
{0pt}{3.25ex plus 1ex minus .2ex}{1.5ex plus .2ex}

\usepackage{movie15}

\usepackage{cite}
\usepackage{caption}
\usepackage[bottom,hang,flushmargin]{footmisc} 

\newcommand{\tsn}[1]{{\left\vert\kern-0.25ex\left\vert\kern-0.25ex\left\vert #1 
		\right\vert\kern-0.25ex\right\vert\kern-0.25ex\right\vert}}

\definecolor{darkred}{RGB}{150,0,0}
\definecolor{darkgreen}{RGB}{0,150,0}
\definecolor{darkblue}{RGB}{0,0,150}
\hypersetup{colorlinks=true, linkcolor=darkred, citecolor=darkgreen, urlcolor=darkgreen}

\usepackage{upgreek}
\usepackage[ruled, vlined, linesnumbered]{algorithm2e}

\usepackage{xstring}

\usepackage{multirow}
\usepackage{calc}

\usepackage{bbm}

\newtheorem{theorem}{Theorem}[section]

\newtheorem{assumption}{Assumption}

\newtheorem{lemma}[theorem]{Lemma}

\newtheorem{proposition}[theorem]{Proposition}

\newtheorem{remark}[subsection]{Remark}

\newtheorem{assumptionalt}{Assumption}[assumption]

\newenvironment{assumptionp}[1]{
  \renewcommand\theassumptionalt{#1}
  \assumptionalt
}{\endassumptionalt}

\newcommand{\bv}[1]{{\boldsymbol{#1}}}		
\newcommand{\bvgrk}[1]{{\boldsymbol{#1}}}	

\newcommand{\yhat}{\hat{y}}

\newcommand{\va}{\bv{a}}
\newcommand{\vb}{\bv{b}}
\newcommand{\vc}{\bv{c}}
\newcommand{\vd}{\bv{d}}
\newcommand{\ve}{\bv{e}}

\newcommand{\vu}{\bv{u}}
\newcommand{\vv}{\bv{v}}
\newcommand{\vw}{\bv{w}}
\newcommand{\vx}{\bv{x}}
\newcommand{\vy}{\bv{y}}
\newcommand{\vz}{\bv{z}}

\newcommand{\ub}{\bv{u}}
\newcommand{\wb}{\bv{w}}
\newcommand{\xb}{\bv{x}}

\newcommand{\vetil}{\tilde{\ve}}

\newcommand{\vutil}{\tilde{\vu}}

\newcommand{\Dcal}{\mathcal{D}}
\newcommand{\Ecal}{\mathcal{E}}

\newcommand{\Ncal}{\mathcal{N}}
\newcommand{\Ocal}{\mathcal{O}}

\newcommand{\Rcal}{\mathcal{R}}
\newcommand{\Scal}{\mathcal{S}}

\newcommand{\tn}[1]{\|{#1}\|_{\ell_2}}
\newcommand{\vA}{\bv{A}}
\newcommand{\vB}{\bv{B}}
\newcommand{\vC}{\bv{C}}

\newcommand{\vF}{\bv{F}}
\newcommand{\vG}{\bv{G}}
\newcommand{\vH}{\bv{H}}

\newcommand{\vL}{\bv{L}}
\newcommand{\vM}{\bv{M}}

\newcommand{\vO}{\bv{O}}

\newcommand{\vQ}{\bv{Q}}

\newcommand{\vT}{\bv{T}}
\newcommand{\vU}{\bv{U}}
\newcommand{\vV}{\bv{V}}

\newcommand{\vX}{\bv{X}}
\newcommand{\vY}{\bv{Y}}

\newcommand{\Ab}{\bv{A}}

\newcommand{\Bb}{\bv{B}}
\newcommand{\Cb}{\bv{C}}

\newcommand{\vAhat}{\hat{\bv{A}}}
\newcommand{\vBhat}{\hat{\bv{B}}}
\newcommand{\vChat}{\hat{\bv{C}}}

\newcommand{\vGhat}{\hat{\bv{G}}}
\newcommand{\vHhat}{\hat{\bv{H}}}

\newcommand{\vLhat}{\hat{\bv{L}}}

\newcommand{\vOhat}{\hat{\bv{O}}}

\newcommand{\vQhat}{\hat{\bv{Q}}}

\newcommand{\vUhat}{\hat{\bv{U}}}
\newcommand{\vVhat}{\hat{\bv{V}}}

\newcommand{\vQtil}{\tilde{\bv{Q}}}
\newcommand{\vRtil}{\tilde{\bv{R}}}

\newcommand{\vUtil}{\tilde{\bv{U}}}
\newcommand{\vVtil}{\tilde{\bv{V}}}
\newcommand{\vWtil}{\tilde{\bv{W}}}

\newcommand{\vGamma}{\bvgrk{\Gamma}}
\newcommand{\vdelta}{\bvgrk{\updelta}}

\newcommand{\vzeta}{\bvgrk{\upzeta}}

\newcommand{\vSigma}{\bvgrk{\Sigma}}

\newcommand{\nn}{\nonumber}

\newcommand{\mtx}[1]{\bm{#1}}
\newcommand{\bSi}{{\boldsymbol{{\Sigma}}}}
\newcommand{\E}{\operatorname{\mathbb{E}}}
\newcommand{\PP}{\operatorname{\mathbb{P}}}

\newcommand{\tf}[1]{\|{#1}\|_{F}}

\newcommand{\Iden}{{\mtx{I}}}

\renewcommand{\P}{\operatorname{\mathbb{P}}}
\newcommand{\leqsym}[1]{\stackrel{\text{(#1)}}{\leq}}

\newcommand{\eqsym}[1]{\stackrel{\text{(#1)}}{=}}

\newcommand{\preceqsym}[1]{\stackrel{\text{(#1)}}{\preceq}}

\newcommand{\norm}[1]{\|{#1} \|}

\newcommand{\R}{\mathbb{R}}				

\newcommand{\dm}[2]
{
	\IfStrEq{#2}{1}{\R^{#1}}{\R^{#1 \x #2}}
}

\newcommand{\tr}{\textup{\textbf{tr}}} 			
\newcommand{\vek}{\textup{\textbf{vec}}}			
\newcommand{\mat}{\textup{\textbf{mtx}}}			
\newcommand{\distas}{\overset{\text{i.i.d.}}{\sim}}

\makeatletter
\renewcommand*{\T}{{\mathpalette\@transpose{}}}
\newcommand*{\@transpose}[2]{\raisebox{\depth}{$\m@th#1\intercal$}}
\makeatother

\newcommand*{\x}{\mathsf{x}\mskip1mu} 	

\newcommand{\splitatcommas}[1]{%
	\begingroup
	\begingroup\lccode`~=`, \lowercase{\endgroup
		\edef~{\mathchar\the\mathcode`, \penalty0 \noexpand\hspace{0pt plus .1em}}%
	}\mathcode`,="8000 #1%
	\endgroup
}

\usepackage{bbm}
\newcommand{\Item}[1]{%
	\ifx\relax#1\relax  \item \else \item[#1] \fi
	\abovedisplayskip=0pt\abovedisplayshortskip=0pt~\vspace*{-\baselineskip}
} 

\usepackage{booktabs}
\usepackage{multirow}

\usepackage{mathtools}
\makeatletter
\newcommand{\raisemath}[1]{\mathpalette{\raisem@th{#1}}}
\newcommand{\raisem@th}[3]{\raisebox{#1}{$#2#3$}}
\makeatother

\newtheorem{innercustomgeneric}{\customgenericname}
\providecommand{\customgenericname}{}
\newcommand{\newcustomtheorem}[2]{%
	\newenvironment{#1}[1]
	{%
		\renewcommand\customgenericname{#2}%
		\renewcommand\theinnercustomgeneric{##1}%
		\innercustomgeneric
	}
	{\endinnercustomgeneric}
}
\newcustomtheorem{customtheorem}{Theorem}
\newcustomtheorem{customlemma}{Lemma}
\newcustomtheorem{customproposition}{Proposition}

\newcommand{\vertiii}[1]{{\vert\kern-0.25ex\vert\kern-0.25ex\vert #1 \vert\kern-0.25ex\vert\kern-0.25ex\vert}}

\newcommand{\beq}{\begin{equation}}
	
	\newcommand{\eeq}{\end{equation}}

\newcommand{\Sc}{\mathcal{S}}

\newcommand{\ubb}{\bar{\vu}}
\newcommand{\wbb}{\bar{\vw}}

\newcommand{\bgl}{{~\big |~}}

\definecolor{emmanuel}{RGB}{255,127,0}

\renewcommand{\P}{\operatorname{\mathbb{P}}}

\numberwithin{equation}{section} 

\def \endprf{\hfill {\vrule height6pt width6pt depth0pt}\medskip}
\newenvironment{proof}{\noindent {\bf Proof} }{\endprf\par}

\newcommand{\vAbar}{\bar{\bv{A}}}
\newcommand{\vBbar}{\bar{\bv{B}}}
\newcommand{\vCbar}{\bar{\bv{C}}}
\newcommand{\vObar}{\bar{\bv{O}}}
\newcommand{\vQbar}{\bar{\bv{Q}}}

\newcommand{\EE}{\mathbb{E}}

%% file: Bilinear_Obs_SYSID.bbl
\begin{thebibliography}{10}
\providecommand{\url}[1]{#1}
\csname url@samestyle\endcsname
\providecommand{\newblock}{\relax}
\providecommand{\bibinfo}[2]{#2}
\providecommand{\BIBentrySTDinterwordspacing}{\spaceskip=0pt\relax}
\providecommand{\BIBentryALTinterwordstretchfactor}{4}
\providecommand{\BIBentryALTinterwordspacing}{\spaceskip=\fontdimen2\font plus
\BIBentryALTinterwordstretchfactor\fontdimen3\font minus
  \fontdimen4\font\relax}
\providecommand{\BIBforeignlanguage}[2]{{%
\expandafter\ifx\csname l@#1\endcsname\relax
\typeout{** WARNING: IEEEtran.bst: No hyphenation pattern has been}%
\typeout{** loaded for the language `#1'. Using the pattern for}%
\typeout{** the default language instead.}%
\else
\language=\csname l@#1\endcsname
\fi
#2}}
\providecommand{\BIBdecl}{\relax}
\BIBdecl

\bibitem{donoho2006compressed}
D.~L. Donoho, ``Compressed sensing,'' \emph{IEEE Transactions on information
  theory}, vol.~52, no.~4, pp. 1289--1306, 2006.

\bibitem{tanaka2016semidefinite}
T.~Tanaka, K.-K.~K. Kim, P.~A. Parrilo, and S.~K. Mitter, ``Semidefinite
  programming approach to gaussian sequential rate-distortion trade-offs,''
  \emph{IEEE Transactions on Automatic Control}, vol.~62, no.~4, pp.
  1896--1910, 2016.

\bibitem{koren2021advances}
Y.~Koren, S.~Rendle, and R.~Bell, ``Advances in collaborative filtering,''
  \emph{Recommender systems handbook}, pp. 91--142, 2021.

\bibitem{jun2019bilinear}
K.-S. Jun, R.~Willett, S.~Wright, and R.~Nowak, ``Bilinear bandits with
  low-rank structure,'' in \emph{International Conference on Machine
  Learning}.\hskip 1em plus 0.5em minus 0.4em\relax PMLR, 2019, pp. 3163--3172.

\bibitem{baclawski2018observer}
K.~Baclawski, ``The observer effect,'' in \emph{2018 IEEE Conference on
  Cognitive and Computational Aspects of Situation Management (CogSIMA)}.\hskip
  1em plus 0.5em minus 0.4em\relax IEEE, 2018, pp. 83--89.

\bibitem{dean2022preference}
S.~Dean and J.~Morgenstern, ``Preference dynamics under personalized
  recommendations,'' in \emph{Proceedings of the 23rd ACM Conference on
  Economics and Computation}, 2022, pp. 795--816.

\bibitem{khosravi2023bandits}
K.~Khosravi, R.~P. Leme, C.~Podimata, and A.~Tsorvantzis, ``Bandits with
  deterministically evolving states,'' \emph{arXiv preprint arXiv:2307.11655},
  2023.

\bibitem{leqi2021rebounding}
L.~Leqi, F.~Kilinc~Karzan, Z.~Lipton, and A.~Montgomery, ``Rebounding bandits
  for modeling satiation effects,'' \emph{Advances in Neural Information
  Processing Systems}, vol.~34, pp. 4003--4014, 2021.

\bibitem{faradonbeh2018finite}
M.~K.~S. Faradonbeh, A.~Tewari, and G.~Michailidis, ``Finite time
  identification in unstable linear systems,'' \emph{Automatica}, vol.~96, pp.
  342--353, 2018.

\bibitem{dean2018regret}
S.~Dean, H.~Mania, N.~Matni, B.~Recht, and S.~Tu, ``Regret bounds for robust
  adaptive control of the linear quadratic regulator,'' in \emph{Advances in
  Neural Information Processing Systems}, 2018, pp. 4188--4197.

\bibitem{simchowitz2018learning}
M.~Simchowitz, H.~Mania, S.~Tu, M.~I. Jordan, and B.~Recht, ``Learning without
  mixing: Towards a sharp analysis of linear system identification,'' in
  \emph{Conference On Learning Theory}.\hskip 1em plus 0.5em minus 0.4em\relax
  PMLR, 2018, pp. 439--473.

\bibitem{dean2019sample}
S.~Dean, H.~Mania, N.~Matni, B.~Recht, and S.~Tu, ``On the sample complexity of
  the linear quadratic regulator,'' \emph{FOCM}, pp. 1--47, 2019.

\bibitem{fattahi2019learning}
S.~Fattahi, N.~Matni, and S.~Sojoudi, ``Learning sparse dynamical systems from
  a single sample trajectory,'' in \emph{2019 IEEE 58th Conference on Decision
  and Control (CDC)}.\hskip 1em plus 0.5em minus 0.4em\relax IEEE, 2019, pp.
  2682--2689.

\bibitem{sarkar2019finite}
T.~Sarkar, A.~Rakhlin, and M.~A. Dahleh, ``Finite time lti system
  identification,'' \emph{Journal of Machine Learning Research}, vol.~22, pp.
  1--61, 2021.

\bibitem{sarkar2019near}
T.~Sarkar and A.~Rakhlin, ``Near optimal finite time identification of
  arbitrary linear dynamical systems,'' in \emph{ICML}.\hskip 1em plus 0.5em
  minus 0.4em\relax PMLR, 2019, pp. 5610--5618.

\bibitem{lale2020logarithmic}
S.~Lale, K.~Azizzadenesheli, B.~Hassibi, and A.~Anandkumar, ``Logarithmic
  regret bound in partially observable linear dynamical systems,''
  \emph{Advances in Neural Information Processing Systems}, vol.~33, pp.
  20\,876--20\,888, 2020.

\bibitem{jedra2020finite}
Y.~Jedra and A.~Proutiere, ``Finite-time identification of stable linear
  systems optimality of the least-squares estimator,'' in \emph{2020 59th IEEE
  Conference on Decision and Control (CDC)}.\hskip 1em plus 0.5em minus
  0.4em\relax IEEE, 2020, pp. 996--1001.

\bibitem{wagenmaker2020active}
A.~Wagenmaker and K.~Jamieson, ``Active learning for identification of linear
  dynamical systems,'' in \emph{Conference on Learning Theory}.\hskip 1em plus
  0.5em minus 0.4em\relax PMLR, 2020, pp. 3487--3582.

\bibitem{tu2017non}
S.~Tu, R.~Boczar, A.~Packard, and B.~Recht, ``Non-asymptotic analysis of robust
  control from coarse-grained identification,'' \emph{arXiv preprint
  arXiv:1707.04791}, 2017.

\bibitem{hazan2017learning}
E.~Hazan, K.~Singh, and C.~Zhang, ``Learning linear dynamical systems via
  spectral filtering,'' \emph{Advances in Neural Information Processing
  Systems}, vol.~30, 2017.

\bibitem{hardt2018gradient}
M.~Hardt, T.~Ma, and B.~Recht, ``Gradient descent learns linear dynamical
  systems,'' \emph{The Journal of Machine Learning Research}, vol.~19, no.~1,
  pp. 1025--1068, 2018.

\bibitem{oymak2018non}
S.~Oymak and N.~Ozay, ``Non-asymptotic identification of lti systems from a
  single trajectory,'' \emph{American Control Conference}, 2019.

\bibitem{tsiamis2019finite}
A.~Tsiamis and G.~J. Pappas, ``Finite sample analysis of stochastic system
  identification,'' in \emph{2019 IEEE 58th Conference on Decision and Control
  (CDC)}.\hskip 1em plus 0.5em minus 0.4em\relax IEEE, 2019, pp. 3648--3654.

\bibitem{simchowitz2019learning}
M.~Simchowitz, R.~Boczar, and B.~Recht, ``Learning linear dynamical systems
  with semi-parametric least squares,'' in \emph{Conference on Learning
  Theory}.\hskip 1em plus 0.5em minus 0.4em\relax PMLR, 2019, pp. 2714--2802.

\bibitem{sun2020finite}
Y.~Sun, S.~Oymak, and M.~Fazel, ``Finite sample system identification: Optimal
  rates and the role of regularization,'' in \emph{Learning for dynamics and
  control}.\hskip 1em plus 0.5em minus 0.4em\relax PMLR, 2020, pp. 16--25.

\bibitem{djehiche2022efficient}
B.~Djehiche and O.~Mazhar, ``Efficient learning of hidden state lti state space
  models of unknown order,'' \emph{arXiv preprint arXiv:2202.01625}, 2022.

\bibitem{lee2022improved}
H.~Lee, ``Improved rates for prediction and identification of partially
  observed linear dynamical systems,'' in \emph{International Conference on
  Algorithmic Learning Theory}.\hskip 1em plus 0.5em minus 0.4em\relax PMLR,
  2022, pp. 668--698.

\bibitem{mania2022time}
H.~Mania, A.~Jadbabaie, D.~Shah, and S.~Sra, ``Time varying regression with
  hidden linear dynamics,'' in \emph{Learning for Dynamics and Control
  Conference}.\hskip 1em plus 0.5em minus 0.4em\relax PMLR, 2022, pp. 858--869.

\bibitem{bakshi2023new}
A.~Bakshi, A.~Liu, A.~Moitra, and M.~Yau, ``A new approach to learning linear
  dynamical systems,'' in \emph{Proceedings of the 55th Annual ACM Symposium on
  Theory of Computing}, 2023, pp. 335--348.

\bibitem{sarkar2019data}
T.~Sarkar, A.~Rakhlin, and M.~Dahleh, ``Nonparametric system identification of
  stochastic switched linear systems,'' in \emph{2019 IEEE 58th Conference on
  Decision and Control (CDC)}, 2019, pp. 3623--3628.

\bibitem{sattar2021identification}
Y.~Sattar, Z.~Du, D.~A. Tarzanagh, L.~Balzano, N.~Ozay, and S.~Oymak,
  ``Identification and adaptive control of markov jump systems: Sample
  complexity and regret bounds,'' \emph{arXiv preprint arXiv:2111.07018}, 2021.

\bibitem{du2022data}
Z.~Du, Y.~Sattar, D.~A. Tarzanagh, L.~Balzano, N.~Ozay, and S.~Oymak,
  ``Data-driven control of markov jump systems: Sample complexity and regret
  bounds,'' in \emph{2022 American Control Conference (ACC)}.\hskip 1em plus
  0.5em minus 0.4em\relax IEEE, 2022, pp. 4901--4908.

\bibitem{sayedana2024strong}
B.~Sayedana, M.~Afshari, P.~E. Caines, and A.~Mahajan, ``Strong consistency and
  rate of convergence of switched least squares system identification for
  autonomous markov jump linear systems,'' \emph{IEEE Transactions on Automatic
  Control}, 2024.

\bibitem{oymak2019stochastic}
S.~Oymak, ``Stochastic gradient descent learns state equations with nonlinear
  activations,'' in \emph{Conference on Learning Theory}, 2019, pp. 2551--2579.

\bibitem{bahmani2019convex}
S.~Bahmani and J.~Romberg, ``Convex programming for estimation in nonlinear
  recurrent models,'' \emph{arXiv preprint arXiv:1908.09915}, 2019.

\bibitem{mhammedi2020learning}
Z.~Mhammedi, D.~J. Foster, M.~Simchowitz, D.~Misra, W.~Sun, A.~Krishnamurthy,
  A.~Rakhlin, and J.~Langford, ``Learning the linear quadratic regulator from
  nonlinear observations,'' \emph{Advances in Neural Information Processing
  Systems}, vol.~33, pp. 14\,532--14\,543, 2020.

\bibitem{sattar2020non}
Y.~Sattar and S.~Oymak, ``Non-asymptotic and accurate learning of nonlinear
  dynamical systems,'' \emph{Journal of Machine Learning Research}, vol.~23,
  no. 140, pp. 1--49, 2022.

\bibitem{jain2021near}
P.~Jain, S.~S. Kowshik, D.~Nagaraj, and P.~Netrapalli, ``Near-optimal offline
  and streaming algorithms for learning non-linear dynamical systems,''
  \emph{Advances in Neural Information Processing Systems}, vol.~34, 2021.

\bibitem{sattar2022finite}
Y.~Sattar, S.~Oymak, and N.~Ozay, ``Finite sample identification of bilinear
  dynamical systems,'' in \emph{2022 IEEE 61st Conference on Decision and
  Control (CDC)}.\hskip 1em plus 0.5em minus 0.4em\relax IEEE, 2022, pp.
  6705--6711.

\bibitem{mania2022active}
H.~Mania, M.~I. Jordan, and B.~Recht, ``Active learning for nonlinear system
  identification with guarantees,'' \emph{Journal of Machine Learning
  Research}, vol.~23, no.~32, pp. 1--30, 2022.

\bibitem{taylor2021towards}
A.~J. Taylor, V.~D. Dorobantu, S.~Dean, B.~Recht, Y.~Yue, and A.~D. Ames,
  ``Towards robust data-driven control synthesis for nonlinear systems with
  actuation uncertainty,'' in \emph{2021 60th IEEE Conference on Decision and
  Control (CDC)}.\hskip 1em plus 0.5em minus 0.4em\relax IEEE, 2021, pp.
  6469--6476.

\bibitem{ziemann2022single}
I.~M. Ziemann, H.~Sandberg, and N.~Matni, ``Single trajectory nonparametric
  learning of nonlinear dynamics,'' in \emph{Conference on Learning
  Theory}.\hskip 1em plus 0.5em minus 0.4em\relax PMLR, 2022, pp. 3333--3364.

\bibitem{kazemian2024random}
K.~Kazemian, Y.~Sattar, and S.~Dean, ``Random features approximation for
  control-affine systems,'' \emph{arXiv preprint arXiv:2406.06514}, 2024.

\bibitem{oymak2021revisiting}
S.~Oymak and N.~Ozay, ``Revisiting ho--kalman-based system identification:
  Robustness and finite-sample analysis,'' \emph{IEEE Transactions on Automatic
  Control}, vol.~67, no.~4, pp. 1914--1928, 2021.

\bibitem{ho1966effective}
B.~Ho and R.~E. K{\'a}lm{\'a}n, ``Effective construction of linear
  state-variable models from input/output functions,''
  \emph{at-Automatisierungstechnik}, vol.~14, no. 1-12, pp. 545--548, 1966.

\bibitem{dean2020robust}
S.~Dean, N.~Matni, B.~Recht, and V.~Ye, ``Robust guarantees for
  perception-based control,'' in \emph{Learning for Dynamics and
  Control}.\hskip 1em plus 0.5em minus 0.4em\relax PMLR, 2020, pp. 350--360.

\bibitem{dean2021certainty}
S.~Dean and B.~Recht, ``Certainty equivalent perception-based control,'' in
  \emph{Learning for Dynamics and Control}.\hskip 1em plus 0.5em minus
  0.4em\relax PMLR, 2021, pp. 399--411.

\bibitem{koltchinskii2015bounding}
V.~Koltchinskii and S.~Mendelson, ``Bounding the smallest singular value of a
  random matrix without concentration,'' \emph{International Mathematics
  Research Notices}, vol. 2015, no.~23, pp. 12\,991--13\,008, 2015.

\bibitem{oliveira2016lower}
R.~I. Oliveira, ``The lower tail of random quadratic forms with applications to
  ordinary least squares,'' \emph{Probability Theory and Related Fields}, vol.
  166, pp. 1175--1194, 2016.

\bibitem{krahmer2014suprema}
F.~Krahmer, S.~Mendelson, and H.~Rauhut, ``Suprema of chaos processes and the
  restricted isometry property,'' \emph{Communications on Pure and Applied
  Mathematics}, vol.~67, no.~11, pp. 1877--1904, 2014.

\bibitem{sarkar2021finite}
T.~Sarkar, A.~Rakhlin, and M.~A. Dahleh, ``Finite time lti system
  identification,'' \emph{Journal of Machine Learning Research}, vol.~22,
  no.~26, pp. 1--61, 2021.

\bibitem{djehiche2019finite}
B.~Djehiche, O.~Mazhar, and C.~R. Rojas, ``Finite impulse response models: A
  non-asymptotic analysis of the least squares estimator,'' \emph{arXiv
  preprint arXiv:1911.12794}, 2019.

\bibitem{nakkiran2020optimal}
P.~Nakkiran, P.~Venkat, S.~Kakade, and T.~Ma, ``Optimal regularization can
  mitigate double descent,'' \emph{arXiv preprint arXiv:2003.01897}, 2020.

\bibitem{ziemann2023tutorial}
I.~Ziemann, A.~Tsiamis, B.~Lee, Y.~Jedra, N.~Matni, and G.~J. Pappas, ``A
  tutorial on the non-asymptotic theory of system identification,'' in
  \emph{2023 62nd IEEE Conference on Decision and Control (CDC)}.\hskip 1em
  plus 0.5em minus 0.4em\relax IEEE, 2023, pp. 8921--8939.

\bibitem{jedra2022finite}
Y.~Jedra and A.~Proutiere, ``Finite-time identification of linear systems:
  Fundamental limits and optimal algorithms,'' \emph{IEEE Transactions on
  Automatic Control}, vol.~68, no.~5, pp. 2805--2820, 2022.

\bibitem{vershynin2010introduction}
R.~Vershynin, ``Introduction to the non-asymptotic analysis of random
  matrices,'' \emph{arXiv preprint arXiv:1011.3027}, 2010.

\bibitem{tu2015low}
S.~Tu, R.~Boczar, M.~Simchowitz, M.~Soltanolkotabi, and B.~Recht, ``Low-rank
  solutions of linear matrix equations via procrustes flow,'' \emph{arXiv
  preprint arXiv:1507.03566}, 2015.

\bibitem{hazan2018spectral}
E.~Hazan, H.~Lee, K.~Singh, C.~Zhang, and Y.~Zhang, ``Spectral filtering for
  general linear dynamical systems,'' \emph{Advances in Neural Information
  Processing Systems}, vol.~31, 2018.

\bibitem{vincent2023prediction}
T.~L. Vincent, S.~Yan, and E.~Bitar, ``Prediction for dynamic systems using
  impulse response representation,'' \emph{IEEE Transactions on Automatic
  Control}, vol.~69, no.~2, pp. 1296--1302, 2023.

\bibitem{tzoumas2016sensor}
V.~Tzoumas, A.~Jadbabaie, and G.~J. Pappas, ``Sensor placement for optimal
  kalman filtering: Fundamental limits, submodularity, and algorithms,'' in
  \emph{2016 American control conference (ACC)}.\hskip 1em plus 0.5em minus
  0.4em\relax IEEE, 2016, pp. 191--196.

\bibitem{don1979expectation}
F.~Don, ``The expectation of products of quadratic forms in normal variables,''
  \emph{Statistica Neerlandica}, vol.~33, no.~2, pp. 73--79, 1979.

\end{thebibliography}
